\documentclass[]{ant_tech_report}

\usepackage{hyperref}
\usepackage{cleveref}
\usepackage{verbatim}

\usepackage{wrapfig}
\usepackage{graphicx}
\usepackage{floatrow}
\usepackage{subcaption}
\usepackage{listings}
\usepackage{algorithm}
\usepackage{subcaption} % 
\usepackage[toc,page,header]{appendix}
\usepackage{minitoc}

\title{Perceptual Flow Network for Visually Grounded Reasoning}

\author{%
\parbox{\textwidth}{\centering
Yangfu Li$^{*,1,5,\sharp}$, 
Yuning Gong$^{*,2,6,\sharp}$, 
Hongjian Zhan$^{1,\S}$, 
Teng Li$^{3,5,\sharp}$, 
Yuanhuiyi Lyu$^{3,5,\sharp}$\\[2mm]
Tianyi Chen$^{4}$, 
Qi Liu$^{1}$, 
Ziyuan Huang$^{5}$, 
Zhihang Zhong$^{4,\S}$, 
Dandan Zheng$^{5,\S}$, 
Yue Lu$^{1}$
}}

\affiliation{%
\parbox{\textwidth}{\centering\small
$^1$ECNU,
$^2$SCU,
$^3$HKUST,
$^4$SJTU,
$^5$Ant Group,
$^6$Shanghai AI Laboratory
}}

\contribution[*]{Equal contribution}
\contribution[\S]{Corresponding Author}
\contribution[\sharp]{Work done during internship}

\abstract{
Despite the success of Large-Vision Language Models (LVLMs), general optimization objectives (\eg, standard MLE) fail to constrain visual trajectories, leading to language bias and hallucination. To mitigate this, current methods introduce geometric priors from visual experts as additional supervision. However, we observe that such supervision is typically suboptimal: \emph{it is biased toward geometric precision and offers limited reasoning utility}. To bridge this gap, we propose Perceptual Flow Network (PFlowNet), which eschews rigid alignment with the expert priors and achieves interpretable yet more effective visual reasoning. Specifically, PFlowNet decouples perception from reasoning to establish a self-conditioned generation process. Based on this, it integrates multi-dimensional rewards with vicinal geometric shaping via variational reinforcement learning, thereby facilitating reasoning-oriented perceptual behaviors while preserving visual reliability. PFlowNet delivers a provable performance guarantee and competitive empirical results, particularly setting new SOTA records on V* Bench (90.6\%) and MME-RealWorld-lite (67.0\%).
}
\date{\today}
\checkdata[E-mail]{yfli\_cee@stu.ecnu.edu.cn}
\checkdata[Status]{\emph{Accepted to the 43rd International Conference on Machine Learning (ICML 2026).}}

\begin{document}
\maketitle

\section{Introduction}

Large Vision-Language Models (LVLMs) extend pretrained Large Language Models (LLMs) by integrating sophisticated vision encoders~\cite{radford2021learning} and cross-modal alignment~\cite{liu2023llava}, achieving remarkable performance across diverse visual tasks~\cite{bai2511qwen3,bai2025qwen25vl,liu2024llavanext,lyu2026struvis}. However, LVLMs still face challenges with interpretability and hallucination, particularly in complex scenarios, \eg, fine-grained visual understanding. 

To enhance reliability, recent advances~\cite{liu2025look,wang2025traceable,wang2025vgr,sarch2025grounded,liu2025visual} distill geometric priors from visual experts, \eg, GroundingDINO~\cite{liu2024grounding}, into LVLMs via \emph{Reinforcement Learning with Verifiable Reward} (RLVR). By directly maximizing geometric consistency between LVLM predictions and expert priors, these approaches effectively anchor intermediate reasoning processes in visual evidence.
Despite this progress, a critical question remains: 
\begin{insightline}
% \small
\setlength{\parskip}{0.6em}
\setlength{\parindent}{0pt}
The visual experts are initially designed for object detection; thus, are the geometric priors derived from these experts truly optimal for visual reasoning?
\end{insightline}

\textbf{Preliminary Study.} To investigate this, we conduct a probing study using Qwen2.5-VL~\cite{bai2025qwen25vl} family on V*~\cite{wu2024vstar}. This benchmark encompasses \emph{direct attribute} recognition and \emph{spatial relation} reasoning, backed by fine-grained expert annotations. We generate varying geometric priors by isotropically expanding the original annotations from their centers. By feeding the models directly with these evidence crops instead of full images, we measure the reasoning utility of different geometric priors. As illustrated in~\Cref{fig:1}, we observe a counterintuitive result: the most precise geometric prior, \ie, expert annotation, is \emph{not} the most helpful for reasoning. We attribute this to a fundamental mismatch between the design principles of visual experts and LVLMs. While these experts are optimized to localize evidence with strict geometric precision, such an approach may induce a \emph{tunnel vision} effect during reasoning. This effectively excludes context necessary for comprehensive understanding and degrades performance.

A natural intuition is to approximate the \emph{golden} evidence by applying heuristic transformations to expert priors, thereby constructing less biased geometric guidance for LVLMs. However, we find that the optimal geometric prior is highly \emph{instance-specific}, making such strategies intractable.

\begin{figure}[t]
	\centering
	\includegraphics[width=1\linewidth]{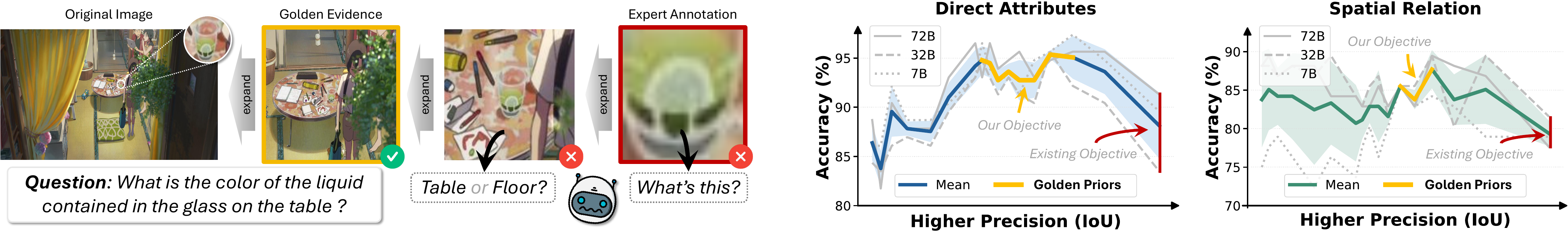}
	\caption{Impact of evidence geometric precision (IoU \emph{w.r.t.} the expert annotations) on reasoning performance (accuracy). The evidence with \emph{minimum} and \emph{maximum} precision is actually the full image and the expert annotation (\emph{outlined in red}), respectively.
    }
	\label{fig:1}
\end{figure}

Motivated by this challenge, we propose Perceptual Flow Network (\textbf{PFlowNet}). Instead of constraining visual rationales via rigid alignment with static geometric priors, PFlowNet employs a \emph{self-parameterized} variational distribution to approximate the posterior of idealized perceptual behaviors. By sampling from the optimized intrinsic distribution, PFlowNet \emph{self-conditions} its subsequent reasoning process, yielding grounded yet more accurate outputs. To realize this, PFlowNet features three key innovations: 
\begin{enumerate}[leftmargin=*,
                  itemsep=2pt,    
                  parsep=0pt,    
                  topsep=2pt,     
                  partopsep=0pt]  
                  
    \item[\ding{182}] \textbf{Perceptual Flow}, \ie, a structured trajectory formulation, designed to effectively characterize perceptual behaviors in LVLMs, facilitating efficient optimization via hierarchical variational objectives, \eg, Sub-Trajectory Balance (SubTB). 
    \item[\ding{183}] \textbf{Decoupled Framework} that separates optimizable perceptual behaviors from model's reasoning process, thereby enabling visually grounded reasoning via a \emph{self-conditioned} autoregressive generation.
    \item[\ding{184}] \textbf{Variational Reinforcement Fine-Tuning Strategy} that integrates a \emph{multi-dimensional reward function} with a \emph{vicinal geometric shaping} scheme to encourage visual-reliable yet reasoning-oriented perceptual behaviors.
\end{enumerate}

Building on these, we provide theoretical analysis that establishes a provable performance guarantee for PFlowNet, as detailed in~\Cref{thm:1,thm:2}. Moreover, comprehensive experimental results demonstrate its superiority across both general-purpose and fine-grained visual tasks from the empirical perspective. Importantly, it achieves substantial improvements of 13.1\%, 10.4\% and 21\% over the base model (\ie, Qwen3-VL 8B) on V* Bench, TreeBench, and MME-RealWorld-lite, respectively. Further analysis highlights its favorable performance-efficiency balance and effective test-time scaling properties.

\section{Background and Motivation}
\subsection{Problem Formulation}
Let $\mathcal{M}_\theta$ denote an LVLM parameterized by $\theta$, built upon a standard transformer architecture~\cite{vaswani2017attention}. Given a multimodal input $X$ (\eg, images and instructions), $\mathcal{M}_\theta$ defines an autoregressive conditional distribution:
\[
p_\theta(Y \mid X) = \prod_{t=1}^{T} p_\theta(y_t \mid X, y_{<t}),    
\]
where $Y = (y_1,y_2\dots,y_T)$ represents the output token sequence conditioned on $X$. Conventionally, $\mathcal{M}_\theta$ is optimized via Maximum-Likelihood Estimation (MLE):
\[
\max_{\theta} \;
\mathbb{E}_{(X,Y)\sim P_{\rm data}}
\big[ \log p_\theta(Y \mid X) \big].
\]
Despite the remarkable efficacy of this paradigm, it remains challenging to mitigate hallucination in $\mathcal{M}_\theta$~\cite{liu2024survey,gunjal2024detecting,chen2024multi}, particularly in visual-centric applications (\eg, fine-grained visual search).  
To formalize this, we consider the visual reasoning trajectory (\eg, the sequence of RoIs) as a latent variable $Z$. In this view, the fundamental cause of hallucination stems from an ill-posed posterior $P(Z \mid X, Y)$ that may assign probability mass to \emph{invalid} trajectories $Z$. Inspired by the success of RLVR in LLMs~\cite{guo2025deepseekr1}, recent works explore incorporating geometric priors as verifiable rewards to constrain $Z$ for Visually Grounded Reasoning (VGR). 
\begin{theorybox}
\begin{definition}[Visually Grounded Reasoning]
\label{def:grounded_reasoning}
Consider an input-output pair $(X,Y)$ and a \emph{golden} visual trajectory $G$ that mediates the inference process $X\xrightarrow{G}Y$. We define $\mathcal{S}_{\rm V}$ as the support of all \emph{valid} visual trajectory $Z$, which is the $\sigma$-neighborhood of $G$ under a deviation metric $d(\cdot,\cdot)$:
\(
\mathcal{S}_{\rm V} \coloneqq \big\{ Z \mid d(Z, G) \le \sigma \big\}.
\)
\emph{Target posterior} $P_V(Z\mid X,Y)\coloneqq P(Z \mid X, Y, Z\in\mathcal{S}_V)$ is given by assigning its probability mass exclusively to this support, and visually grounded reasoning is formulated as
\[
\textstyle
\max_{\theta} \;
\mathbb{E}_{(X,Y)\sim P_{\rm data}}
\Big[\, \log \int_{\mathcal{S}_{\rm V}} p_\theta(Y, Z \mid X)\, dZ\, \Big],
\]
which encourages $\mathcal{M}_\theta$ to both yield the correct answer $Y$ and anchor its latent visual rationales $Z$ to $G$.
\end{definition}
\end{theorybox}

\begin{figure*}[t]
	\centering
	\includegraphics[width=0.95\linewidth]{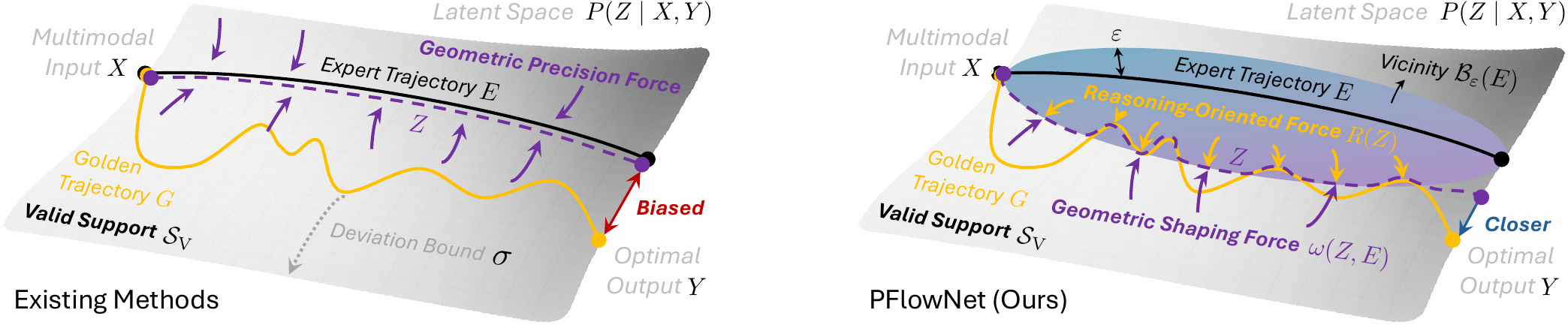}
	\caption{Illustration of feasible regions ($\mathcal{S}_{\rm v}$) and optimization objectives for visually grounded reasoning. Existing methods constrain LVLMs to imitate expert trajectories by maximizing their geometric consistency, whereas PFlowNet integrates a reasoning-oriented reward with vicinal geometric shaping to achieve more sufficient yet controlled exploration, leading to reliable and high-efficacy reasoning.}
	\label{fig:3}
    % \vspace{-4mm}
\end{figure*}

\vspace{-1mm}
\subsection{Revisit VGR as Reasoning over Perceptual Flow}
\vspace{-0.5mm}
The golden trajectory $G$ is generally intractable; thus, previous works typically adopt well-trained visual experts, \eg, GroundingDINO~\cite{liu2024grounding} to synthesize a proxy for $G$. However, these experts are initially optimized for grounding rather than downstream reasoning. As a result, the synthetic trajectory is biased toward high geometric precision rather than reasoning utility, leading to suboptimal performance of the policy $p_\theta(Y|X)$, as revealed in~\Cref{fig:3}. 
To address this misalignment, we apply a \emph{self-parameterized} variational distribution $p_\theta(Z \mid X)$ to approximate the target posterior $P_V(Z\,|\, X,Y)$, achieving VGR via a latent-variable mixture:
\[
\underbrace{p_\theta(Y, Z \mid X)}_{\textrm{VGR}}\ =\ p_\theta(Z \mid X)\ \, p_\theta(Y \mid X, Z)\ =\  \underbrace{P_V(Z\,|\, X,Y)}_{\textrm{Approximated}}\ \underbrace{p_\theta(Y \mid X, Z)}_{\textrm{Grounded Reasoning}},
\]
which is the key insight of the proposed PFlowNet. To more precisely characterize the behaviors of LVLMs and thus facilitate the optimization of $p_\theta(Z \mid X)$, we further introduce the concept of \emph{Perceptual Flow}. % in the following.
\begin{theorybox}
\begin{definition}[Perceptual Flow]
\label{def:1}
Given an input $X$, we define \textbf{Perceptual Flow} $Z = (z_0\rightarrow z_1 \dots  z_K)$ as a structured latent trajectory that explicates the visual thoughts. It comprises two distinct states:
\begin{enumerate}[leftmargin=*, itemsep=0pt, parsep=2pt, topsep=2pt, partopsep=0pt]
    \item[$\diamondsuit$] \textbf{Planning State} ($z_0$): A language sequence enclosed by special tokens $\langle \texttt{\small analyze} \rangle$ and $\langle/ \texttt{\small analyze} \rangle$. This state decomposes the query within $X$ and identifies relevant visual candidates for subsequent exploration.
    \item[$\diamondsuit$] \textbf{Perceptual States} ($z_{\ge1}$): A chain of grounded observations enclosed by $\langle \texttt{\small localize} \rangle$ and $\langle/ \texttt{\small localize} \rangle$. Each state $z_k = \langle r_k, c_k \rangle$ consists of a Region of Interest (RoI) $r_k \in \mathbb{N}^4$ (represented in \emph{relative} coordinates, \eg, from 0 to 1000) and a corresponding descriptive caption $c_k$.
\end{enumerate}
\end{definition}
\end{theorybox}

Leveraging this design, we incorporate Sub-Trajectory Balance (Sub-TB)~\cite{madan2023learning}, a hierarchical variational objective. Unlike PPO-like RL paradigm, this formulation provides dense intermediate supervision, thereby facilitating diverse perceptual behaviors. Formally, given a perceptual flow $Z\!\sim\!p_\theta(Z\,|\,X)$, let $z_{i:j} \subseteq Z$ be any sub-trajectory indexed by $0\!\le\!i\!\le\!j\!\le K$, the Sub-TB objective derived by a divergence metric $D$ is defined as:
\begin{equation}
 \min_\theta\ \sum_{i,j}  D\big(\mathcal{F}(z_i)\, \mathcal{T}_F(z_{i:j})\, \|\, (\mathcal{F}(z_j)\, \mathcal{T}_B(z_{j:i})\big)
\label{equ:6}
\end{equation}
where $\mathcal{T}_F(z_{i:j}) = \prod_{k=i+1}^j p_\theta(z_k \mid z_{k-1})$, $\mathcal{T}_B(z_{j:i}) = \prod_{k=i+1}^j p_\theta(z_{k-1} \mid z_k)$ denotes the forward, backward transitions over the flow $Z$, and $\mathcal{F}(z)$ is the total probability mass of all flows passing through the state $z$.

\begin{figure*}[t]
	\centering
	\includegraphics[width=\linewidth]{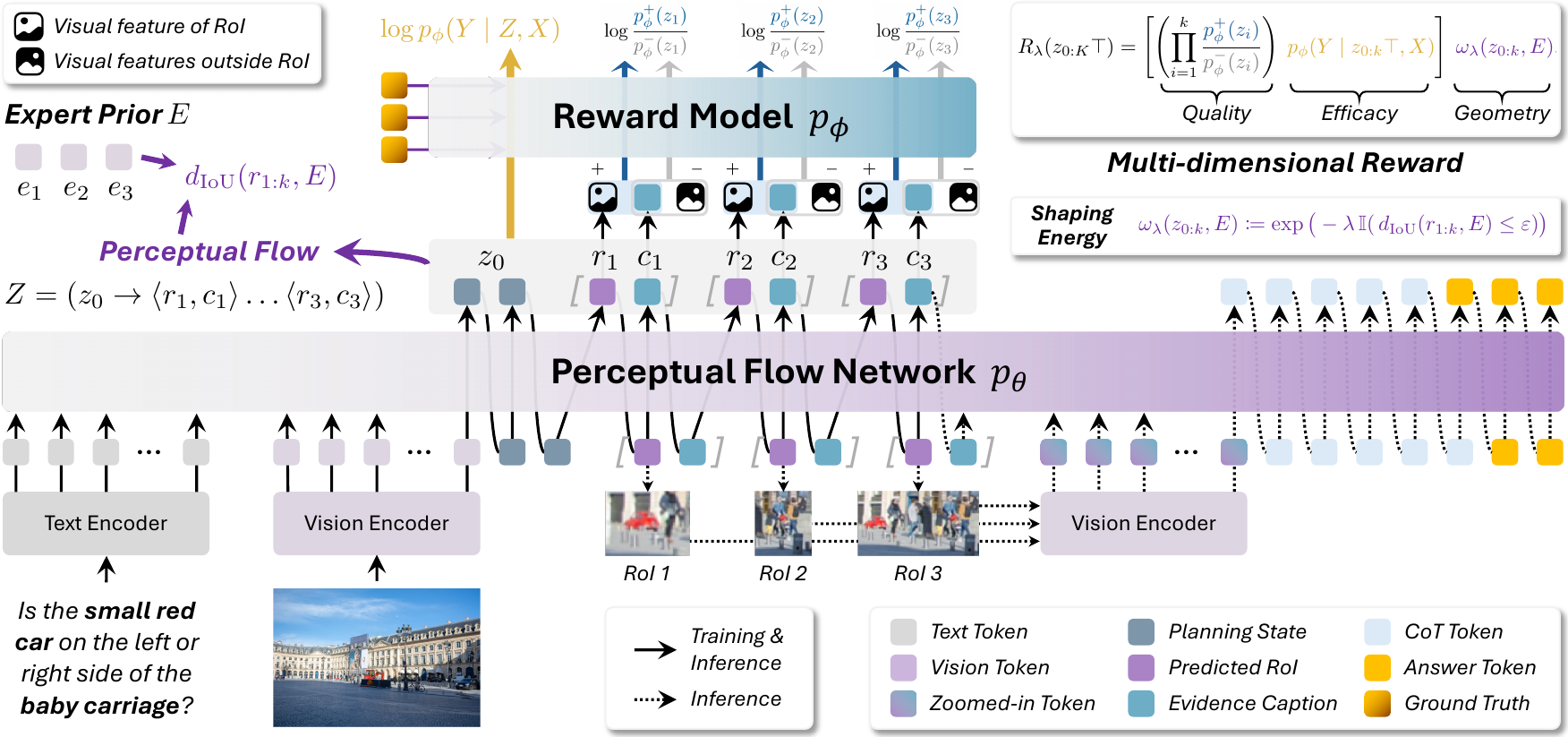}
	\caption{Overview of PFlowNet that consists of two \emph{decoupled} stages: flow generation and flow-guided reasoning. We leverage a frozen reward model with the multi-dimensional reward to guide PFlowNet toward reasoning-oriented yet visually reliable perceptual flows. During reasoning, PFlowNet integrates the textual flow with corresponding visual features to derive interpretable and accurate answers.
    }
	\label{fig:4}
\end{figure*}
\section{Perceptual Flow Network}
The overall architecture of PFlowNet is shown in~\Cref{fig:4}. Formally, let $X\coloneqq \langle I,T \rangle$ denote a multimodal input consisting of an image $I$ and an instruction $T$. PFlowNet first samples the perceptual flow $Z$ from its \emph{intrinsic} distribution and then yields the grounded output $Y$ via a \emph{self-conditioned} generation. The joint distribution is factorized as:
\[
p_\theta(Y,\,Z\mid X)=p_\theta(Z\mid X)\ p_\theta(Y\mid Z,\, \langle X,I_{\rm RoI}  \rangle),
\]
where $I_{\rm RoI}$ denotes the region of interest from the image $I$ conditioned on the perceptual flow $Z$. 
To effectively optimize the parameterized variational distribution $p_\theta(Z \mid X)$, we employ a progressive training paradigm. First, guided by the insights in~\Cref{fig:1}, we design a tailored data pipeline to synthesize fine-grained trajectories, explicitly aimed at preliminarily mitigating the inductive bias inherent in visual experts. Based on this, we bootstrap the model's capability to generate perceptual flows via Supervised Fine-Tuning (SFT). Furthermore, we propose a variational Reinforcement Fine-Tuning (RFT) strategy that integrates a carefully-designed reward and a vicinal geometric shaping to ensure a better approximation of the target posterior. This design liberates the model from the constraints of expert geometric priors, enabling it to extensively explore genuinely effective perceptual behaviors while maintaining visual reliability.

\newpage
\subsection{Training Data Curation \& Cold Start}
\label{sec:data}

\begingroup
\setlength{\columnsep}{0.05\textwidth}
\setlength{\intextsep}{0pt}            

\begin{wrapfigure}{r}{0.45\textwidth}
    \vspace{-1.5\baselineskip}
    \centering
    \includegraphics[width=\linewidth]{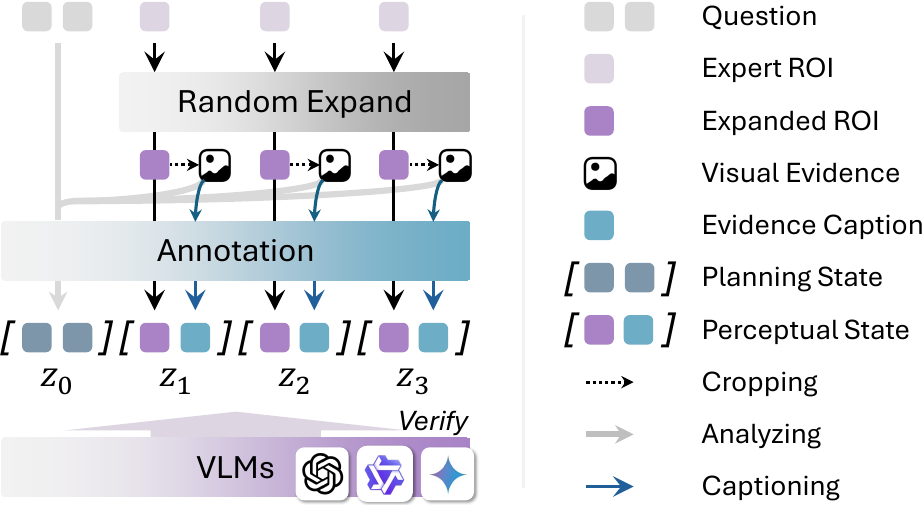}
    \caption{Data pipeline for perceptual flow synthesis.}
    \label{fig:5}
    \vspace{-0.0\baselineskip}
\end{wrapfigure}

\textbf{Data Collection.} We curate high-quality data for cold start and subsequent RFT based on two principles:

\vspace{-1mm}
\par%\smallskip
\noindent
\hangindent=1.6em\hangafter=1
\makebox[1.4em][l]{$\diamondsuit$}\textbf{Diverse Tasks.}
We consider a broad spectrum of visual tasks, spanning both fine-grained understanding and general-purpose scenarios, which ensures the model develops generalizable perceptual behaviors. 

\vspace{-1mm}
\par%\smallskip
\noindent
\hangindent=1.6em\hangafter=1
\makebox[1.4em][l]{$\diamondsuit$}\textbf{Diverse RoIs.}
To prevent overfitting to specific spatial patterns, we perform cross-expert annotation for each sample and preserve the samples whose RoIs have sufficiently broad and diverse spatial coverage.

\par
\endgroup
% \vspace{0.5em}

\begin{figure*}[b]
	\centering
	\includegraphics[width=0.9\linewidth]{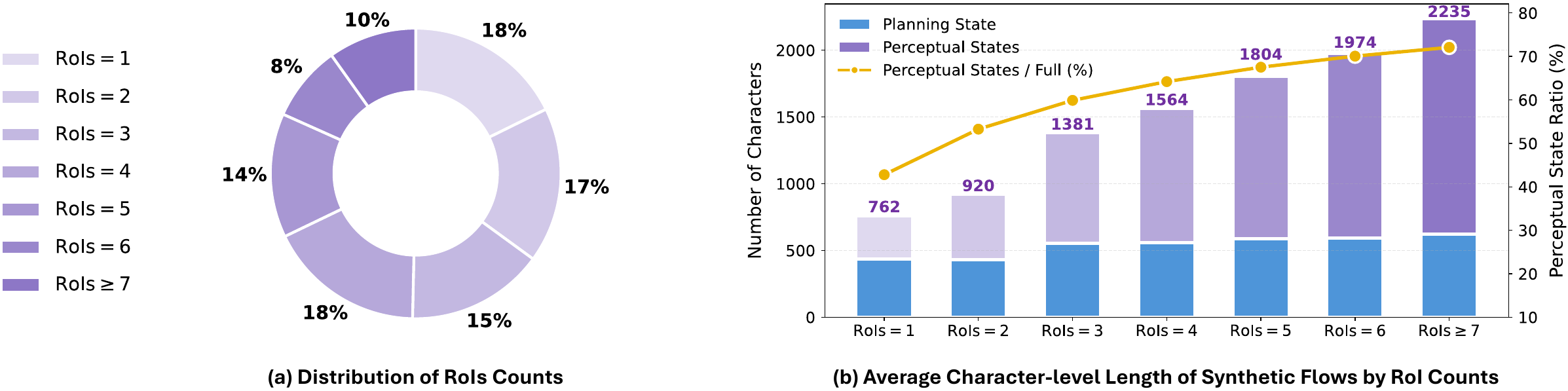}
	\caption{Statistics of the Cold-Start Dataset. Notably, as the number of RoIs increases, the average character length of the \emph{Planning State} remains largely stable, whereas that of the \emph{Perceptual States} grows substantially.
    }
	\label{fig:11}
    \vspace{-1mm}
\end{figure*}

\begin{table}[t]
\centering
\small
\caption{
Training data construction via verifier-based filtering and difficulty-aware splitting.
Here, $Z_{\rm s}$ denotes the synthetic perceptual flow, and $k_{\rm pass}$ denotes the \emph{minimum} sampling budget required for the verifier to produce a correct answer, with $k_{\rm pass}>n$ indicating failure within $n$ decoded responses. In the data tuple, $E$ denotes the original expert RoIs before random expansion, and $Y$ denotes the \emph{accepted response} generated by the verifier conditioned on $Z_{\rm s}$.}

\label{tab:difficulty_control}
\setlength{\tabcolsep}{17pt}
\begin{tabular}{ll | ll}
\toprule
\textbf{Verification} \emph{w/o} $Z_{\rm s}$
& \textbf{Verification} \emph{w} $Z_{\rm s}$ 
& \textbf{Decision}
& \textbf{Data Tuple} \\
\midrule
$k_{\rm pass}=1$ 
& -- 
& \emph{Rejected} as trivial 
& -- \\
-- 
& $k_{\rm pass}>1$ 
& \emph{Rejected} as unverified flow
& -- \\
\midrule
$2 \le k_{\rm pass} \le 16$ 
& $k_{\rm pass}=1$  
& \emph{Accepted} to the RFT dataset 
& $(X,Y,E)$ \\
$k_{\rm pass}>16$ 
& $k_{\rm pass}=1$  
& \emph{Accepted} to the cold-start dataset 
& $(X,Z_{\rm s})$ \\

\bottomrule
\end{tabular}
\end{table}

\vspace{-1mm}
\textbf{Flow Synthesis.}
We construct training datasets by eliciting step-by-step trajectories from teacher models, \eg, Gemini3flash~\cite{gemini-3-flash} and GPT-4o~\cite{gpt4o}. As shown in~\Cref{fig:5}, for each sample equipped with expert RoIs, we first randomly expand each RoI to mitigate the inductive bias introduced by visual experts. The teacher is then prompted to (i) identify the critical visual content conditioned on both the question and the RoIs, which serves as the \emph{Planning State} $z_0$, and (ii) generate detailed captions for each piece of visual evidence. Each expanded RoI, together with its corresponding caption, is treated as a \emph{Perceptual State} $z_{\ge 1}$. The synthetic perceptual flow $Z_{\rm s}$ is formed by composing the planning state with all subsequent perceptual states.

\vspace{-1mm}
\textbf{Verification \& Difficulty Control.}
After synthesizing candidate flows for all collected samples, we perform verifier-based filtering under two settings: 
(i) direct answering without the synthetic flow, \ie, \emph{w/o} $Z_{\rm s}$; and (ii) answering conditioned on $Z_{\rm s}$ and the corresponding zoomed-in evidence, \ie, \emph{w/} $Z_{\rm s}$.
As summarized in~\Cref{tab:difficulty_control}, we first drop trivial samples and samples with unreliable flows, and then assign the remaining samples to either the cold-start set or the RFT set according to the performance gain induced by the synthetic grounding behaviors. Finally, the detailed statistics of the cold-start dataset are provided in~\Cref{fig:11}.

\vspace{-1mm}
\textbf{Cold Start.}
For each sample $(X,Z_{\rm s})$ from the cold-start set, we initialize the policy via \emph{supervised fine-tuning} by minimizing the cross-entropy loss between $p_\theta(Z\mid X)$ and the synthetic flow $Z_{\rm s}$. This teaches the policy to generate perceptual flows that benefit downstream reasoning.

\subsection{Variational Reinforcement Fine-tuning}
While our flow synthesis pipeline applies heuristic strategies to mitigate the inductive bias of visual experts, explicitly determining the golden flow for each sample remains inherently challenging. To address this, we propose variational RFT, which leverages a variational objective coupled with a \emph{tailored reward} function and a \emph{vicinal geometric shaping} scheme to ensure a better approximation of the target posterior $P_V$ by the policy $p_\theta(Z | X)$.
Specifically, given the RFT set $P_{\rm data}(X,Y,E)$ introduced in~\Cref{tab:difficulty_control}, let $R_\lambda(z_{0:i})\coloneqq \mathcal{F}(z_i)=\frac{R_\lambda(z_{0:i}\top)}{p_\theta(\top\mid z_{0:i})}$ be the reward of a trajectory ending at $z_i$, and $\top$ denotes the terminal state (\ie, $\langle \texttt{\small /localize} \rangle$ token), we derive the objective for variational RFT by reformulating~\Cref{equ:6} (see Appendix~\ref{app:A2}) as follows:
\begin{equation}
\mathcal{L}_{\rm vRFT}(\theta)=\mathbb{E}_{\substack{X,Y,E\sim P_{\rm data}\\ \{Z\}_{l=1}^L\sim p_{\theta}(\mathcal{Z}\mid X)}}
\left[\
\sum_{0\le i\le j\le |Z|}\right. 
\left.\left(\log\frac{R_\lambda(z_{0:i}\top)\prod_{k=i+1}^jp_\theta(z_k\mid z_{0:k-1})p_\theta(\top\mid z_{0:j})}{R_\lambda(z_{0:j}\top)\,p_\theta(\top\mid z_{0:i})}\!\right)^2\ \right].
\label{equ:8}
\end{equation}
Notably,~\Cref{equ:8} involves dense computations of rewards and transition probabilities for trajectories sharing the same sub-flow prefixes. For computational efficiency, we develop a parallel strategy to solve this objective, ensuring scalable optimization even for extensive perception chains (detailed in Appendix~\ref{app:B3}).

\textbf{Reward Design.}
To comprehensively characterize the perceptual behaviors, given any sub-flow $z_{0:k}\subseteq Z$, we design a multi-dimensional reward that jointly evaluates its \emph{quality} and \emph{reasoning efficacy}, which is defined as
\[
R(z_{0:k}\top)= \left(\prod_{i=1}^k \frac{p_\phi^+(z_i)}{{p_\phi^-(z_i)}}\right)\, p_\phi(Y\,|\, z_{0:k}\top, X)\ \Longleftrightarrow \
\log (R(z_{0:k}\top))= \sum_{i=1}^k\log \frac{p_\phi^+(z_i)}{{p_\phi^-(z_i)}}+\log p_\phi(Y\,|\, z_{0:k}\top, X),
\label{equ:9}
\tag{3}
\]
where $p_\phi$ is a frozen reward model sharing the same initialization as PFlowNet. The positive and negative visual-context likelihoods are defined as
\(
p_\phi^+(z_i)=p_\phi(c_i\mid I_{r_i}),
\
p_\phi^-(z_i)=p_\phi(c_i\mid I\setminus I_{r_i}),
% \label{equ:10}
% \tag{4}
\)
where $I_{r_i}=\mathrm{Crop}(r_i,I)$ denotes the zoomed-in visual evidence targeted by $r_i$, and $I\setminus I_{r_i}$ denotes the complementary region outside $r_i$. 

\textbf{Key Insights in Reward Design.}
For a sampled flow $z_{0:k}\subseteq Z$, the contrastive term $\prod p_\phi^+(z)/p_\phi^-(z)$ admits an interpretation as \emph{privileged-information distillation} for improving its quality in the reverse-KL sense. Formally, let $q_\theta^i(c)\coloneqq p_\theta(c_i=c\mid X,z_{<i},r_i)$ be the policy-induced caption distribution for the predicted $r_i$. Under the trajectory expectation in~\Cref{equ:8}, the expected contrastive reward over the sub-flow can be written as
\[
\mathbb{E}_{c_{1:k}\sim q_\theta^{1:k}}
\left[
\sum_{i=1}^k
\log \frac{p_\phi(c_i\mid I_{r_i})}
{p_\phi(c_i\mid I\setminus I_{r_i})}
\right] =
\sum_{i=1}^k
\left[
D_{\mathrm{KL}}\!\left(
q_\theta^i \,\|\, p_\phi(\cdot\mid I\setminus I_{r_i})
\right)
-
D_{\mathrm{KL}}\!\left(
q_\theta^i \,\|\, p_\phi(\cdot\mid I_{r_i})
\right)
\right].
\]
Therefore, maximizing the contrastive term $\prod p_\phi^+(z)/p_\phi^-(z)$ encourages each $q_\theta^i$ to be closer to the privileged teacher distribution conditioned on the zoomed-in evidence $I_{r_i}$, while moving it away from the noisy distribution conditioned on the less informative region $I\setminus I_{r_i}$. This facilitates visually grounded and semantically specific captions, while suppressing generic descriptions induced by language priors or reward hacking.

Furthermore, we adopt the \emph{information gain} provided by the sampled flow $z_{0:k}\subseteq Z$ for deriving the target response $Y$ to measure its \emph{reasoning efficacy}. Ideally, this \emph{information gain} can be characterized as:
\[
\log p_\phi(Y\,|\,z_{0:k}\top,X)-\log p_\phi(Y\,|\,X).
\]
For a fixed data $(X,Y)$ and reward model $p_\phi$, the term $\log p_\phi(Y\,|\,X)$ is constant \emph{w.r.t.} the sampled flow. 
Thus, maximizing $\log p_\phi(Y\,|\,z_{0:k}\top,X)$ favors perceptual flows with higher utility for inducing the target response $Y$.

\textbf{Vicinal Geometric Shaping.}
While the designed reward $R(Z)$ characterizes the \emph{utility} of a perceptual flow, it does not encode any geometric bias and may therefore encourage excessive exploration, yielding invalid trajectories outside the support $\mathcal{S}_{\rm V}$. Motivated by \emph{Vicinal Risk Minimization}~\cite{chapelle2000vicinal}, we introduce \emph{vicinal geometric shaping} that constrains variational inference to a vicinity around the expert prior. Distinct from existing methods that enforce strict alignment between the policy and visual prior, our scheme targets only samples outside the vicinity, balancing sufficient exploration with validity to discover high-efficacy perceptual behaviors.
We first define the \emph{directed} Chamfer IoU between two RoI sets $A$ and $B$:
\[
IoU_{A\rightarrow B}\!=\!\frac{1}{|A|}\sum\nolimits_{a\in A}\ \sup_{b\in B}{\rm IoU}(a,b), \ \, {\rm IoU}(a,b)\!=\!\frac{a\cap b}{a\cup b},
\]    
and the \emph{symmetrized} Chamfer-IoU distance:
\[
d_{\rm IoU}(A,B)=1-0.5*(IoU_{A\rightarrow B}+IoU_{B\rightarrow A}).
% \tag{5}
% \label{equ:iou}
\]
For any $(X,Y,E)\!\sim\!P_{\rm data}$, we define an $\varepsilon$-vicinity of the prior $E$ by a ball $\mathcal{B}_\varepsilon(E)\coloneqq \{z_{0:k} \mid d_{\rm IoU}(r_{1:k},E)\le \varepsilon \}$, and thus introduce an energy weight $\omega_\lambda(z_{0:k},E)\coloneqq \exp\!\big(-\lambda\,\mathbb{I}(z_{0:k}\notin \mathcal{B}_\varepsilon(E))\big)$, where $\varepsilon$ and $\lambda$ are a hyper-parameters. Notably, since $z_0$ is defined as the planning state without RoI, we define $\omega_\lambda(z_0,E)=1$. Finally, we shape the reward $R(z_{0:k}\top)$ via the geometric energy $\omega_\lambda(z_0,E)$, which is formulated as follow
\[
R_\lambda(z_{0:k}\top)\coloneqq R(z_{0:k}\top)\,\omega_\lambda(z_{0:k},E)=\left[\left(\prod_{i=1}^k \frac{p_\phi^+(z_i)}{{p_\phi^-(z_i)}}\right)\, p_\phi(Y\,|\, z_{0:k}\top, X)\right]\ \omega_\lambda(z_{0:k},E),
\label{equ:mmreward}
\tag{4}
\]
which penalizes excursions outside the vicinity and encourages $p_\theta$ to concentrate probability mass near $\mathcal{B}_\varepsilon(E)$.

\subsection{Theoretical Analysis}

In this section, we derive an \emph{idealized} performance bound for PFlowNet under strict assumptions (see Appendix~\ref{app:A1}) to characterize the effect of its key hyperparameters. By examining the limiting regimes of this bound, we show that the standard MLE and expert-guided RL arise as \emph{special cases} of PFlowNet, establishing a \emph{guaranteed improvement}.

Let $(X, Y)\!\sim\! P_{\rm data}$ be any data tuple. We denote the expert annotation as $E\!\sim\! P(\cdot \mid X,Y)$ and the \emph{golden} evidence as $G\!\sim\!P(\cdot \mid X,Y)$. Any perceptual flow is denoted by $Z\!\coloneqq\!(z_0, \langle R, C\rangle)$, parametrized by the planning state $z_0$, predicted RoI $R$, and captions $C$. To formalize the relationship between $Z, E, G$, we define the \emph{valid support} $\mathcal{S}_{\rm V}$ and the \emph{expert vicinity} $\mathcal{B}_{\varepsilon}$ based on $d_{\rm IoU}$ with $\sigma, \varepsilon \in[0,1]$:
\[
 \mathcal{S}_{\rm V}\!\coloneqq\!\{ Z\, |\, d_{\rm IoU}(R, G)\!\le\!\sigma \},\   \mathcal{B}_{\varepsilon}\!\coloneqq\!\{ Z\,|\, d_{\rm IoU}(R, E)\!\le\!\varepsilon \}.
\]    
Accordingly, denote $s_{\rm V}$ and $s_{\mathcal{B}}$ as the probability masses associated with the support and the vicinity:
\[
s_{\rm V} \coloneqq P(\mathcal{S}_{\rm V} \mid X,Y), \quad 
s_{\mathcal{B}} \coloneqq P(\mathcal{B}_{\varepsilon} \mid X,Y).
\]
Thereby, we model the learning objective using a $\lambda$-shaped posterior distribution $P_\lambda(Z \mid X,Y,E)$, which re-weights the prior $P(Z \mid X,Y)$ to concentrate density around the expert vicinity via a shaping function $\omega_\lambda$:
\[
 P_\lambda(Z \mid X,Y,E) \coloneqq P(Z \mid X,Y)\,\omega_\lambda(Z,E)\,/\,\mathcal{Z}_\lambda.
\]   
where the partition function $\mathcal{Z}_\lambda$ is given by
\(
\mathcal{Z}_\lambda = \int P(Z \mid X,Y)\,\omega_\lambda(Z,E)\,dZ = s_{\mathcal{B}} + e^{-\lambda}(1-s_{\mathcal{B}}).
\)  
Let $P_{\rm V}(Z\,|\, X,Y)\!\coloneqq\!P(Z\,|\, X,Y)/s_{\rm V}$ be the target posterior for \emph{idealized perceptual
behaviors}. We now establish the \emph{TV distance} bound between $p_\theta(Z\,|\, X)$ and this posterior $P_{\rm V}$.
\begin{theorybox}
\begin{theorem}[Total Variation Distance Bound]
\label{thm:1}
Under Assump.~\ref{assump:1},~\ref{assump:2}, we suppose valid support $\mathcal{S}_{\rm V}$ satisfies $d_{\rm eff}$-regularity, where $d_{\rm eff}$ is its effective dimension; thus, $\exists\kappa \ge 1$ such that $q \coloneqq s_{\mathcal{B}}/s_{\rm V} \ge \kappa(\varepsilon/\sigma)^{d_{\rm eff}}$. Suppose the model $p_\theta$ is expressive and let $\theta^\star$ be the global minimizer of $\mathcal{L}_{\rm vRFT}(\theta)$. The total variation distance between the policy $p_{\theta^\star}(Z\,|\, X)$ and the target posterior $P_{\rm V}(Z\,|\,X, Y)$ is bounded by:
\[
D_{\rm TV}(p_{\theta^\star}(\cdot\mid X),P_{\rm V}(\cdot\mid X,Y))\; \le\; \frac{1}{2\,\mathcal{Z}_\lambda} \cdot
\left(
q\,|s_{\rm V}-\mathcal{Z}_\lambda|
+(1-q)\,|e^{-\lambda}s_{\rm V}-\mathcal{Z}_\lambda|
+e^{-\lambda}(1-s_{\rm V})
\right).    
\]
\end{theorem}
\end{theorybox}
\begin{remark}[Limit Analysis \emph{w.r.t.} $\lambda$] As $\lambda\! \to\! 0$, the bound $D_{\rm TV}\! \to\! (1\!-\!s_{\rm V})$, dominated by the inherent sparsity of the valid support; this implies PFlowNet discards geometric constraints and \emph{degrades to standard MLE}. Conversely, as $\lambda\! \to\! \infty$, the bound $D_{\rm TV}\!\to\! (1-q)$, where $q$ quantifies the discrepancy between the expert and golden priors. Thereby, the performance of PFlowNet is bottlenecked by expert bias, \ie, \emph{degenerating to expert-guided RLVR}. 
\end{remark}
\begin{remark}[Limit Analysis \emph{w.r.t.} $\varepsilon$] As $\varepsilon \!\to\! 0$, the vicinity contracts to a singularity ($q \!\to\! 0$); this forces the shaping energy to act indiscriminately on all trajectories, rendering the reward signal uninformative and ultimately loosening the bound. Conversely, increasing $\varepsilon$ within the valid region ($\mathcal{B}_\varepsilon \subseteq \mathcal{S}_{\rm V}$) monotonically improves coverage ($q\uparrow$) and tightens the bound. However, if $\varepsilon$ exceeds the tolerance $\sigma$, the vicinity inevitably encompasses invalid regions, which \emph{dilutes} the geometric guidance and degrades performance.
\end{remark}

\newpage
\begin{theorybox}
\begin{theorem}[Guaranteed Improvement over Baselines]
\label{thm:2}
Let $D_{\rm TV}(\lambda,\varepsilon)$ be the TV bound in~\Cref{thm:1}. For any $\varepsilon$ satisfying $\mathcal{B}_\varepsilon\subseteq \mathcal{S}_{\rm V}$, there exists an intensity $\lambda^\star$ such that
\[
D_{\rm TV}(\lambda^\star,\varepsilon)
\;\le\;
\min\{\,1-s_{\rm V},\,1-q\,\}.
\]   
For fixed $\lambda=\lambda^\star$, the bound is strictly decreasing in $q$ ($\varepsilon\uparrow$).
\end{theorem}
\end{theorybox}
\begin{remark}
This confirms that with proper calibration of intensity $\lambda$ and radius $\varepsilon$, PFlowNet strictly tightens the \emph{idealized} TV bound of standard MLE and expert-guided RLVR.
\end{remark}
\begin{proof}
Refer to Appendix~\ref{app:A4} for the proofs.
\end{proof}

\begin{table*}[t!]
    \centering
    \caption{Comparison with competitive alternatives on TreeBench (\emph{left}) and MME-RealWorld-Lite (\emph{right}).}
    \label{tab:merged_results}
    \setlength{\tabcolsep}{3.5pt}
    \renewcommand{\arraystretch}{1.1}
    \small 
    \resizebox{\linewidth}{!}{
    \begin{tabular}{l | c cccccccccc | c ccccccccc}
    \toprule
    & & \multicolumn{5}{c}{\textbf{Perception}} & \multicolumn{5}{c|}{\textbf{Reasoning}} & & \multicolumn{5}{c}{\textbf{Perception}} & \multicolumn{4}{c}{\textbf{Reasoning}} \\
    \cmidrule(lr){3-7} \cmidrule(lr){8-12} \cmidrule(lr){14-18} \cmidrule(lr){19-22} 
    & \rotatebox{90}{Overall}  
    & \rotatebox{90}{Attributes} 
    & \rotatebox{90}{Material} 
    & \rotatebox{90}{Phy. State} 
    & \rotatebox{90}{Obj. Retr.} 
    & \rotatebox{90}{OCR} 
    & \rotatebox{90}{Per. Trans.} 
    & \rotatebox{90}{Ordering} 
    & \rotatebox{90}{Con. \& Oc.} 
    & \rotatebox{90}{Spa. Cont.} 
    & \rotatebox{90}{Comparison}
    & \rotatebox{90}{Overall}  
    & \rotatebox{90}{OCR}
    & \rotatebox{90}{Remote Sen.}
    & \rotatebox{90}{Diag. \& Tab.}
    & \rotatebox{90}{Monitoring} 
    & \rotatebox{90}{Auto. Driv.} 
    & \rotatebox{90}{OCR}
    & \rotatebox{90}{Diag. \& Tab.} 
    & \rotatebox{90}{Monitoring} 
    & \rotatebox{90}{Auto. Driv.} \\
    \midrule
    \rowcolor{gray!15}\multicolumn{22}{c}{\itshape General Large Vision-Language Models} \\
    LLaVA-OV-7B & 37.3 & 55.2 & 53.8 & 56.5 & 50.0 & 32.4 & 21.2 & 22.8 & 41.5 & 72.4 & 36.4 & 43.7 & 80.0 & 40.0 & 56.0 & 31.7 & 39.4 & 65.0 & 33.0 & 38.0 & 32.0 \\
    LLaVA-OV-72B & 40.5 & \cellcolor{blue!5}{62.1} & 53.8 & 65.2 & 62.3 & 36.8 & 12.9 & 28.1 & 53.7 & 65.5 & 47.7 & 48.7 & 79.2 & 50.7 & 67.0 & 37.9 & 40.0 & 76.0 & 41.0 & 38.7 & 39.3 \\
    InternVL3-8B & 38.8 & 51.7 & \cellcolor{blue!15}{69.2} & 56.5 & 56.3 & 33.7 & 21.2 & 24.6 & 39.0 & 72.4 & 43.2 & 47.9 & 83.6 & 49.3 & 75.0 & 34.5 & 36.9 & 70.0 & 44.0 & 40.0 & 37.0 \\
    InternVL3-38B & 42.0 & 51.7 & 61.5 & 52.2 & 68.8 & 51.5 & 12.9 & 33.3 & \cellcolor{blue!5}{56.1} & 65.5 & 38.6 & 51.0 & 85.6 & 56.0 & 71.0 & 42.6 & 40.0 & 77.0 & 45.0 & 47.3 & 35.0 \\
    InternVL3-78B & 46.4 & \cellcolor{blue!5}{62.1} & 61.5 & 52.2 & 68.8 & 52.9 & 16.5 & 33.3 & \cellcolor{blue!15}{61.0} & \cellcolor{blue!15}{86.2} & 45.5 & 52.3 & 87.6 & 54.7 & 77.0 & 42.6 & 36.6 & 76.0 & 56.0 & 46.0 & \cellcolor{blue!5}{40.3} \\
    Qwen2.5-VL-7B & 37.0 & 55.2 & 53.8 & 56.5 & 62.5 & 27.9 & 20.0 & 35.1 & 39.0 & 44.8 & 43.2 & 42.3 & 87.6 & 32.7 & 83.0 & 27.3 & 30.0 & 72.0 & 62.0 & 28.7 & 23.0 \\
    Qwen2.5-VL-32B & 42.5 & 51.7 & 53.8 & 69.6 & 62.5 & 54.4 & 16.5 & 33.3 & 46.3 & 62.1 & 38.6 & 45.6 & 87.2 & 40.7 & 83.0 & 29.5 & 40.7 & 74.0 & 60.0 & 27.3 & 29.5 \\
    Qwen2.5-VL-72B & 42.2 & \cellcolor{blue!15}{65.5} & \cellcolor{blue!15}{69.2} & 56.5 & 56.3 & 48.5 & 11.8 & 33.3 & 51.2 & 72.4 & 38.6 & 43.7 & 90.8 & 34.0 & 87.0 & 27.9 & 30.6 & 74.0 & 61.0 & 26.7 & 25.5 \\
    Qwen3-VL-4B & 42.2 & 48.3 & 61.5 & 65.2 & \cellcolor{blue!5}{81.3} & 35.3 & 18.8 & 31.6 & 46.3 & \cellcolor{blue!15}{86.2} & 43.2 & 47.1 & 90.8 & 44.7 & 87.0 & 34.8 & 32.6 & 72.0 & 64.0 & 43.4 & 24.3 \\
    Qwen3-VL-8B & 44.9 & \cellcolor{blue!15}{65.5} & 53.9 & 65.2 & 75.0 & \cellcolor{blue!5}{64.7} & 12.9 & 24.6 & 48.8 & 72.4 & 43.2 & 48.6 & \cellcolor{blue!5}{92.8} & \cellcolor{blue!5}{57.3} & 87.0 & 36.4 & 31.4 & 73.0 & 70.0 & 39.3 & 25.3 \\
    Qwen3-VL-32B & 45.2 & 60.3 & \cellcolor{blue!5}{63.4} & 58.1 & \cellcolor{blue!15}{83.6} & 30.3 & \cellcolor{blue!15}{24.2} & \cellcolor{blue!5}{39.7} & 47.7 & \cellcolor{blue!5}{85.2} & \cellcolor{blue!5}{51.4} & 52.0 & 91.6 & 47.3 & \cellcolor{blue!15}{96.0} & 36.1 & 42.9 & 76.0 & \cellcolor{blue!15}{77.0} & 42.7 & 30.0 \\
    \midrule
    \rowcolor{gray!15}\multicolumn{22}{c}{\itshape Visually Grounded Reasoning Models} \\
    Pixel-Reasoner & 39.0 & 58.6 & 61.5 & 65.2 & 50.0 & 48.5 & 14.1 & 31.6 & 39.0 & 44.8 & 40.9 & 49.7 & 89.6 & 52.0 & 86.0 & 38.9 & 30.9 & 71.0 & 72.0 & 46.0 & 32.5 \\
    DeepEyes & 37.5 & \cellcolor{blue!5}{62.1} & 53.8 & 65.2 & 68.8 & 51.5 & 11.8 & 24.6 & 36.6 & 51.7 & 47.7 & 53.2 & 90.0 & 52.7 & 89.0 & 43.3 & 33.4 & 76.0 & 69.0 & 44.0 & 35.0 \\
    DeepEyesV2 & 40.7 & \cellcolor{blue!15}{65.5} & \cellcolor{blue!15}{69.2} & 56.5 & 62.5 & 55.9 & 11.8 & 35.1 & 46.3 & 37.9 & 36.4 & 52.4 & 85.6 & 49.3 & 89.0 & 45.8 & 33.4 & 70.0 & \cellcolor{blue!5}{76.0} & 44.0 & 37.0 \\
    Thyme & 38.2 & 48.2 & 46.1 & 69.5 & 50.0 & 51.4 & 22.3 & 21.0 & 41.3 & 44.8 & 34.0 & 54.4 & 90.4 & 56.7 & 86.0 & 46.3 & 38.5 & \cellcolor{blue!5}{78.0} & 71.0 & 48.0 & 36.0 \\
    TreeVGR & \cellcolor{blue!5}{50.4} & \cellcolor{blue!15}{65.5} & 53.8 & \cellcolor{blue!15}{82.6} & 68.8 & 63.3 & \cellcolor{blue!5}{22.4} & 36.8 & \cellcolor{blue!15}{61.0} & 69.0 & 45.5 & \cellcolor{blue!5}{54.9} & 87.6 & 50.7 & 83.0 & \cellcolor{blue!5}{47.0} & \cellcolor{blue!5}{43.4} & 74.0 & 66.0 & \cellcolor{blue!5}{51.3} & 39.0 \\
    \midrule
    \textbf{PFlowNet} (Ours) & \cellcolor{blue!15}{55.3} & \cellcolor{blue!15}{65.5} & \cellcolor{blue!15}{69.2} & \cellcolor{blue!5}{80.2} & 75.0 & \cellcolor{blue!15}{77.9} & 20.0 & \cellcolor{blue!15}{40.4} & \cellcolor{blue!5}{56.1} & 82.8 & \cellcolor{blue!15}{56.8} & \cellcolor{blue!15}{67.0} & \cellcolor{blue!15}{95.6} & \cellcolor{blue!15}{69.3} & \cellcolor{blue!5}{90.0} & \cellcolor{blue!15}{53.6} & \cellcolor{blue!15}{58.2} & \cellcolor{blue!15}{83.0} & \cellcolor{blue!5}{76.0} & \cellcolor{blue!15}{70.0} & \cellcolor{blue!15}{53.5} \\
    $\Delta$ \emph{vs.} Base Model & \up{10.4} & -- & \up{15.3} & \up{15.0} & -- & \up{13.2} & \up{7.1} & \up{15.8} & \up{7.3} & \up{10.4} & \up{13.6} & \up{18.4} & \up{2.8} & \up{12.0} & \up{3.0} & \up{17.2} & \up{26.8} & \up{10.0} & \up{6.0} & \up{30.7} & \up{28.2} \\
    \bottomrule
    \end{tabular}
    }
\end{table*}

\section{Experiment}
We initialize PFlowNet from Qwen3-VL-8B and evaluate it against representative baselines spanning both general-purpose and fine-grained visual tasks. More implementation details and experimental setups are provided in Appendix~\ref{app:B} and~\ref{app:C}.

\subsection{Main Results}
\textbf{General-purpose Tasks.} 
As shown in \Cref{tab:merged_results}, PFlowNet exhibits robust capabilities in both perception and reasoning. It delivers substantial gains over the vanilla Qwen3-VL-8B across all scenarios, achieving overall improvements of 10.4\% on TreeBench and 18.4\% on MME-RealWorld-Lite. Notably, driven by our reasoning-oriented reward design, these gains are particularly pronounced on reasoning-heavy subsets. Furthermore, PFlowNet outperforms both grounded RLVR-based methods, \eg, TreeVGR~\cite{wang2025traceable}, Pixel Reasoner~\cite{su2025pixelreasoner}, and agentic frameworks, \eg, DeepEyes~\cite{zheng2025deepeyes}, Thyme~\cite{zhang2025thyme}. It yields the best average performance, surpassing the nearest competitors by 5.3\% and 12.6\% on TreeBench and MME-RealWorld-Lite, respectively; and sets SOTA records on 89\% ($17/19$) of sub-tasks, underscoring its generalization.

\textbf{Fine-grained Visual Understanding.} 
As presented in \Cref{tab:1}, PFlowNet achieves SOTA results across all benchmarks, outperforming both representative baselines and general LVLMs. Notably, although the Qwen3-VL series incorporates architectural improvements (\eg, DeepStack) to enhance fine-grained capabilities, PFlowNet still delivers clear gains of 13\%, 8\%/8.8\%, and 2.5\%–7\% on V, HR-Bench (4k/8k), and ScreenSpot, respectively. These improvements are primarily concentrated in reasoning-oriented subsets, such as spatial reasoning and cross-objective relationship recognition. This validates our key insight: PFlowNet yields high-utility perceptual results that enhance visual reasoning while ensuring reliability. Consequently, despite being built on Qwen3-VL-8B, PFlowNet matches the performance of the larger Qwen3-VL-32B on these challenging tasks, \ie, \textbf{90.6} \emph{vs.} 87.4 on V*, 80.4 \emph{vs.} 82.1 / \textbf{75.9} \emph{vs.} 74.8 on HR-Bench 4K / 8K, respectively.

\begin{table*}[t]
    \centering
    \vspace{-2mm}
    \caption{Performance comparison on fine-grained visual tasks: visual search, high-resolution VQA, and GUI grounding.}
    \begin{minipage}[t]{0.705\textwidth}\vspace{0pt}
        \centering
        \footnotesize
        \setlength{\tabcolsep}{5pt}
        \renewcommand{\arraystretch}{1.1}
        \resizebox{\linewidth}{!}{%
        \begin{tabular}{l | ccc ccc ccc}
        \toprule
        & \multicolumn{3}{c}{\textbf{V* Bench}} & \multicolumn{3}{c}{\textbf{HR-Bench 4K}} & \multicolumn{3}{c}{\textbf{HR-Bench 8K}} \\
        \cmidrule(lr){2-4}
        \cmidrule(lr){5-7}
        \cmidrule(lr){8-10}
        & \emph{Overall} & \emph{Attribute} & \emph{Spatial} & \emph{Overall} & \emph{Single} & \emph{Cross} & \emph{Overall} & \emph{Single} & \emph{Cross} \\
        \midrule
        \rowcolor{gray!15}\multicolumn{10}{c}{\itshape General Large Vision-Language Models} \\
        GPT-4o-1120~\cite{gpt4o} & 66.0 & -- & -- & 59.0 & 70.0 & 48.0 & 55.5 & 62.0 & 49.0 \\
        LLaVA-OV-72B~\cite{li2024llavaov} & 73.8 & 80.9 & 63.2 & 66.3 & 76.5 & 56.0 & 60.9 & 68.8 & 53.0 \\
        InternVL3-8B~\cite{zhu2025internvl3} & 72.3 & 73.0 & 71.1 & 70.8 & 79.3 & 62.3 & 62.0 & 64.3 & 59.8 \\
        InternVL3-38B & 77.5 & 77.4 & 77.6 & 76.3 & 83.5 & 69.0 & 67.0 & 71.3 & 62.8 \\
        Qwen2.5-VL-7B~\cite{bai2025qwen25vl} & 74.3 & 77.4 & 69.7 & 72.1 & 88.8 & 55.5 & 68.8 & 83.5 & 54.0 \\
        Qwen2.5-VL-32B & 85.9 & 83.5 & \cellcolor{blue!15}{89.5} & 74.8 & 89.3 & 60.3 & 71.6 & 86.5 & 56.8 \\
        Qwen2.5VL-72B & 84.8 & 90.8 & 80.9 & 79.4 & 88.8 & \cellcolor{blue!5}{70.0} & \cellcolor{blue!5}{76.3} & 84.3 & \cellcolor{blue!15}{68.3} \\
        Qwen3-VL-4B~\cite{bai2511qwen3} & 74.9 & 78.3 & 69.7 & 73.5 & 84.8 & 62.3 & 67.1 & 83.5 & 50.7 \\
        Qwen3-VL-8B           & 77.5 & 80.2 & 73.7 
                                                 & 72.4 & 88.5 & 56.3 
                                                 & 68.1 & 82.0 & 54.3 \\
        Qwen3-VL-32B       & 87.4 & 87.0 & \cellcolor{blue!5}{88.2} & \cellcolor{blue!15}{82.1} & \cellcolor{blue!15}{94.0} & \cellcolor{blue!15}{70.2} & 74.8 & \cellcolor{blue!15}{90.1} & 59.5 \\
        \midrule
        \rowcolor{gray!15}\multicolumn{10}{c}{\itshape Visually Grounded Reasoning Models} \\
        PixelReasoner~\cite{su2025pixelreasoner} & 80.6 & 83.5 & 76.3 & 72.9 & 86.0 & 60.3 & 66.9 & 80.0 & 54.3 \\
        DeepEyes~\cite{zheng2025deepeyes}      & \cellcolor{blue!5}{90.0} & \cellcolor{blue!15}{92.1} & 86.8 & 75.1 & 91.3 & 59.0 & 72.6 & 86.8 & 58.5 \\
        DeepEyesV2~\cite{hong2025deepeyesv2}   & 81.8 & 81.7    & 80.3      & 77.9    &  \cellcolor{blue!5}{92.8}      & 63.0    & 73.8 & 88.5    & 59.0 \\
        Thyme~\cite{zhang2025thyme}         & 82.2 & 83.5 & 80.3 & 77.0 & 91.0 & 63.0 & 72.0 & 86.5 & 57.5 \\
        TreeVGR~\cite{wang2025traceable}       & 87.4 & 89.5 & 84.2 & 77.1  & 89.5 &  64.8 & 72.8 & 86.0 & 59.5 \\
        \midrule
        \textbf{PFlowNet} (Ours)                 & \cellcolor{blue!15}{90.6} & \cellcolor{blue!5}{91.4}  & \cellcolor{blue!15}{89.5} & \cellcolor{blue!5}{80.4}  & 91.2 & \cellcolor{blue!5}{69.5} & \cellcolor{blue!15}{76.9} & \cellcolor{blue!5}{89.0} & \cellcolor{blue!5}{64.8} \\
        $\Delta$ \emph{vs.} Base Model          & \up{13}  & \up{11}  & \up{16} 
                                                 & \up{8.0} & \up{2.7} & \up{13.2}  
                                                 & \up{8.8} & \up{7.0} & \up{10.5} \\
        \bottomrule
        \end{tabular}%
        }

        \label{tab:1}
    \end{minipage}
    \hfill
    \begin{minipage}[t]{0.2715\textwidth}\vspace{0pt}
        \centering
        \footnotesize
        \setlength{\tabcolsep}{5pt}
        \renewcommand{\arraystretch}{1.1}
        \resizebox{\linewidth}{!}{%
        \begin{tabular}{l | cc}
        \toprule
         &  \multicolumn{2}{c}{\textbf{ScreenSpot}} \\
        \cmidrule(lr){2-3}
         & \emph{v2} & \emph{Pro}\\
        \midrule
        \rowcolor{gray!15}\multicolumn{3}{c}{\itshape General LVLMs} \\
        GPT-4o-1120~\cite{gpt4o}      & 18.1 & 0.8  \\
        Claude Comp. Use~\cite{hu2024dawn} & -    & 17.1 \\
        OpenAI CUA~\cite{openai2025operator}       & 87.9 & 23.4 \\
        Qwen2-VL-7B     & -    & 1.6  \\
        Qwen2.5-VL-3B~\cite{bai2025qwen25vl}   & 68.4 & 23.9 \\
        Qwen2.5-VL-7B   & 73.6 & 29.0 \\
        Qwen2.5-VL-72B  & 87.1 & 43.6 \\
        Qwen3-VL-8B~\cite{bai2511qwen3}     & \cellcolor{blue!5}{92.7} & \cellcolor{blue!5}{54.6} \\
        Kimi-VL-16B-MoE~\cite{team2025kimi} & \cellcolor{blue!5}{92.8} & 34.5 \\
        SeedVL-1.5~\cite{guo2025seed1}      & \cellcolor{blue!15}{95.0} & \cellcolor{blue!15}{60.9} \\
        \midrule
        \rowcolor{gray!15}\multicolumn{3}{c}{\itshape GUI Grounding Models} \\
        SeeClick~\cite{cheng2024seeclick}      & 55.1    & 1.1  \\
        OS-Atlas-4B~\cite{wuatlas}   & 71.9    & 3.7  \\
        OS-Atlas-7B   & 84.1    & 18.9 \\
        UI-TARS-2B~\cite{qin2025ui}    & 84.7    & 27.7 \\
        ViGoRL-7B~\cite{sarch2025grounded}      & 86.5    & 31.1 \\
        \midrule
        \textbf{PFlowNet} (Ours)                  & \cellcolor{blue!15}{95.1} & \cellcolor{blue!15}{61.8} \\
        $\Delta$ \emph{vs.} Base Model           & \up{2.4} & \up{7.2} \\
        \bottomrule 
        \end{tabular}}  
        \label{tab:combined_web_eval}
    \end{minipage}
    \vspace{-2mm}
\end{table*}
\begin{figure*}[t]
	\centering
	\includegraphics[width=1\linewidth]{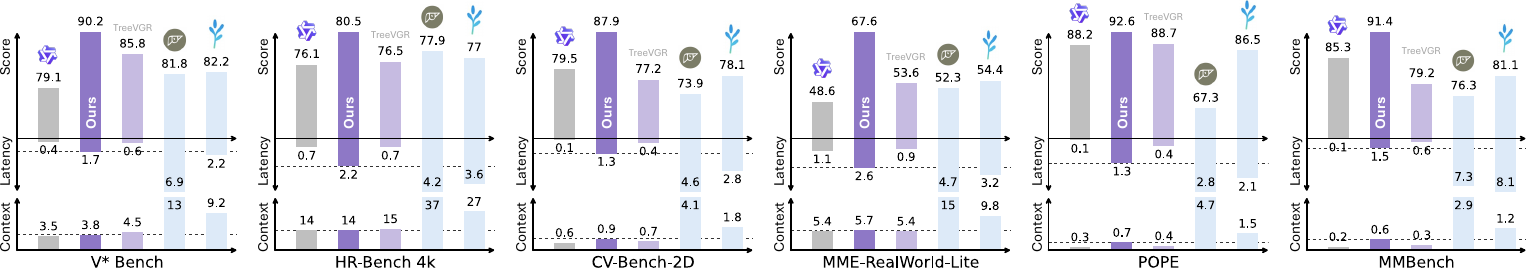}
	\caption{Performance-efficiency trade-offs of Qwen3-VL 8B, PFlowNet, TreeVGR, DeepEyesV2, and Thyme. All models are evaluated on an H200 via VLMEvalKit (see App.~\ref{app:C}). Latency (s) denotes the average inference time per sample, and context (k) is the averaged token-level length. Notably, a smaller occupied region \emph{below the shared axis} indicates lower computational (temp. \& spatial) cost.
    }
	\label{fig:7}
\end{figure*}

\subsection{In-depth Analysis}

\textbf{Performance-Efficiency Trade-off.} As shown in \Cref{fig:7}, PFlowNet exhibits an excellent balance between performance and efficiency. Compared to the agentic frameworks, PFlowNet substitutes complex tool or code executions with carefully designed structured perceptual flows to efficiently encode visual thoughts. This results in significantly shorter context lengths and reduced inference latency without compromising performance. In contrast to TreeVGR, PFlowNet decouples the process into flow generation and flow-guided visual reasoning, which incurs affordable computational costs to substantially improve perceptual quality and utility, thereby significantly boosting visual reasoning performance.

\textbf{Test-Time Scaling.} Theoretically, PFlowNet's variational objective~\eqref{equ:8} ensures a diverse rationale distribution, whereas grounded RLVR often implicitly optimizes a highly sharp distribution due to rigid alignment with sparse expert trajectories. To validate this, we conduct an empirical analysis shown in~\Cref{fig:8} and Appx.~\ref{app:E1}.

\begingroup
\setlength{\columnsep}{0.05\textwidth}  
\setlength{\intextsep}{0pt}              
\setlength{\abovecaptionskip}{2pt}
\setlength{\belowcaptionskip}{-6pt}

\begin{wrapfigure}{r}{0.55\textwidth}
    \vspace{-1.2\baselineskip}
    \centering
    \includegraphics[width=\linewidth]{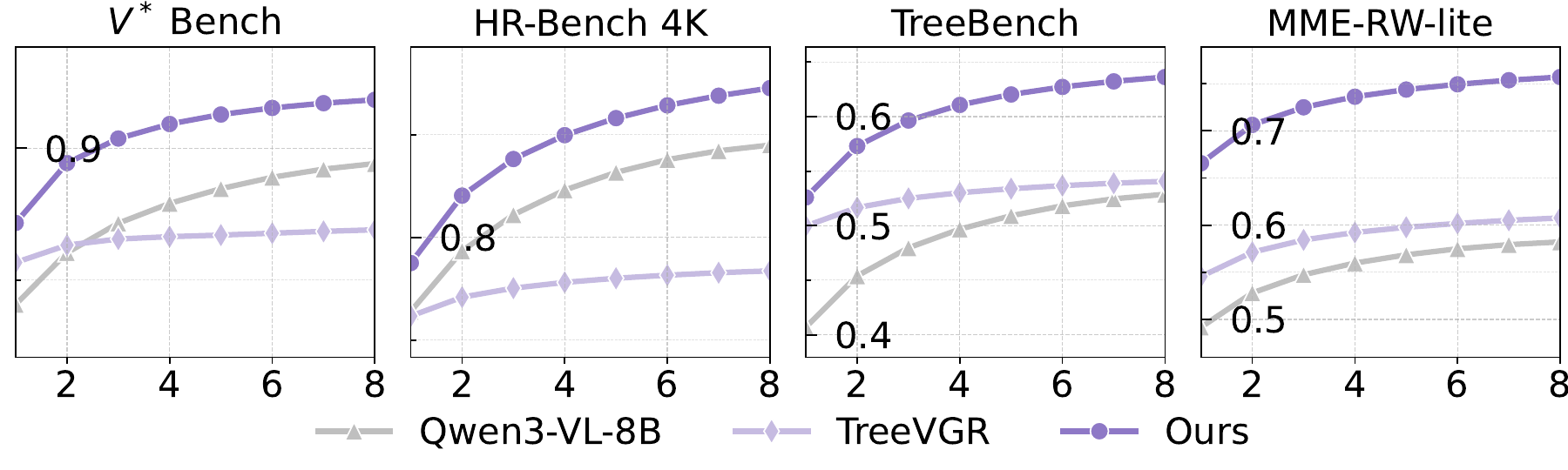}
    \caption{Pass@$k$ curves (\ie, $k\in[1,8]$) of different methods across both fine-grained and general-purpose benchmarks.}
    \label{fig:8}
    \vspace{-0.8\baselineskip}
\end{wrapfigure}

While TreeVGR achieves high Pass@1 accuracy, it yields negligible performance gains as the computational budget ($k$) scales up, particularly in challenging scenarios (\eg, V* Bench, TreeBench). This phenomenon aligns with recent findings by~\cite{yue2025does}. Notably, by incorporating variational inference with tailored reward design and geometric shaping, PFlowNet achieves superior Pass@1 results while demonstrating robust test-time scaling capabilities.
\par
\endgroup

\begin{figure*}[t]
	\centering
	\includegraphics[width=1\linewidth]{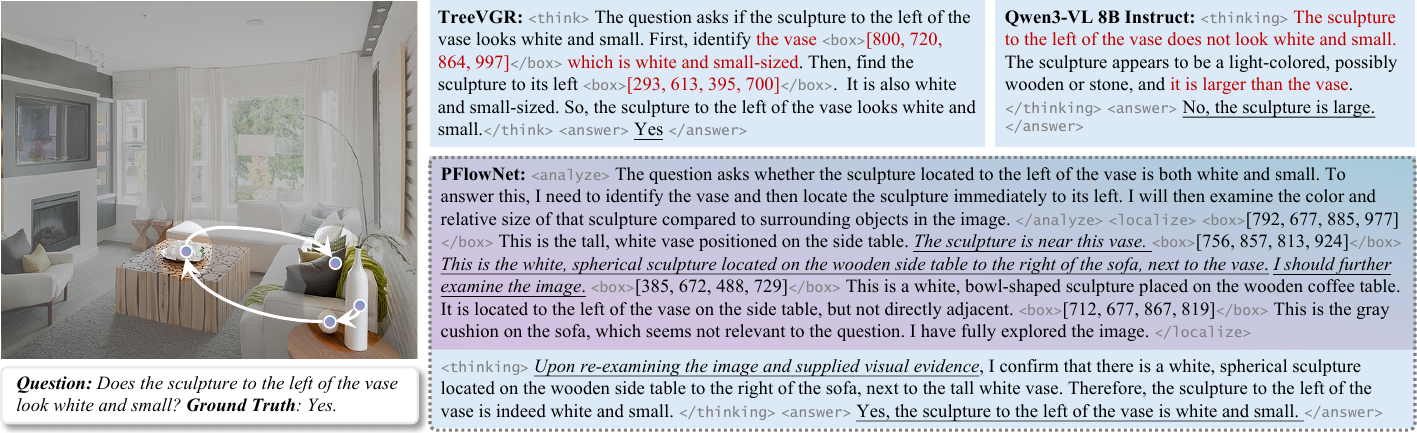}
	\caption{Qualitative comparison of visual reasoning across different methods. PFlowNet enables precise yet comprehensive exploration of visual evidence and effectively anchors the reasoning process to perceptual outcomes, producing the most reliable and accurate answers.
    }
	\label{fig:9}
    % \vspace{-2mm}
\end{figure*}

\textbf{Case Study.}~\Cref{fig:9} qualitatively highlights PFlowNet's superior reliability. Different from TreeVGR, where coupled perception-reasoning often yields geometrically precise yet semantically misaligned boxes due to sparse reward signals, PFlowNet utilizes dense contrastive rewards to enforce strict visual-textual dependency ($c_i$ on $r_i$), ensuring faithful interpretability. Interestingly, we observe that PFlowNet often prioritizes precise localization and then expands its visual scope. We attribute this to the shaping energy $\omega_\lambda(z_{0:k},E)$ derived from the sequence-level metric $d_{\rm IoU}$. Specifically, when $k$ is small, the scarcity of participating RoIs compels the model to maximize the precision of each individual proposal; however, this constraint naturally relaxes as the sequence elongates, facilitating comprehensive reasoning.

\textbf{Character-level Output Length.} We further investigate the effect of RFT on the length of model outputs, as shown in~\Cref{fig:12}. 
Given the same number of RoIs, the generated flow lengths remain highly consistent across benchmarks, indicating that flow length is mainly governed by the amount of required visual evidence rather than benchmark-specific difficulty. Moreover, model-generated flows are generally shorter than the synthetic flows. This is likely due to two factors: (i) the planning state lacks of direct supervision during RFT, making it difficult to match the detailed grounding plans produced by teacher models; and (ii) the contrastive term $\prod p_\phi^+(z)/p_\phi^-(z)$ in the reward encourages grounded captions to be more concise and discriminative.

\begin{figure*}[b]
	\centering
	\includegraphics[width=1\linewidth]{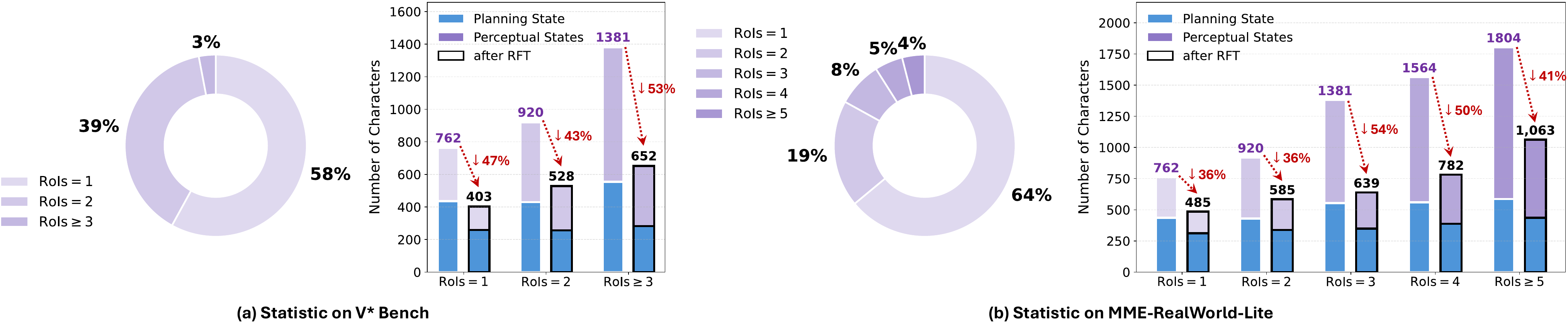}
	\caption{Statistics of RoIs distributions and character-level output length across different type of benchmarks. 
    }
	\label{fig:12}
    \vspace{-1mm}
\end{figure*}

\newpage
\subsection{Ablation Studies}

\begingroup
\setlength{\columnsep}{0.05\textwidth}  
\setlength{\intextsep}{0pt}           
\setlength{\abovecaptionskip}{2pt}
\setlength{\belowcaptionskip}{-6pt}

\begin{wraptable}{r}{0.5\textwidth}
    % \vspace{-2.\baselineskip}
    \centering
    \footnotesize
    \setlength{\tabcolsep}{3pt}
    \renewcommand{\arraystretch}{1.1}
    \caption{Ablation on the training recipe, reward design and reasoning pipeline, where $\pm$ and $+$ denote $P^+/P^-$ and $P^+$.}
    \label{tab:2}
    \resizebox{\linewidth}{!}{
    \begin{tabular}{ll | cc | cc | cccc}
    \toprule
    & &\multicolumn{2}{c|}{Inference} &\multicolumn{2}{c|}{Rewards} & \multicolumn{2}{c}{Treebench} & V* & MME-RW \\
    \cmidrule(lr){3-4}
    \cmidrule(lr){5-6}
    \cmidrule(lr){7-8}
    \cmidrule(lr){9-9}
    \cmidrule(lr){10-10}
    & & $Z$ & $I_{\rm RoI}$ & $R_{\text{cl}}$ & $R_{\text{gain}}$
    & Acc & mIoU & Acc & Acc \\
    \midrule
    (1)& Base Model     & - & - &  -   & -
    & 44.9  &   -   & 77.5   & 46.0 \\
    (2)& + SFT           & \textcolor{gray!50}{\checkmark} & \textcolor{gray!50}{\checkmark}  &  -  &   -
    & 48.3  & 44.2  & 83.7   & 54.2  \\
    \midrule
    (3)& \multirow{4.3}{*}{+ RFT} & \textcolor{gray!50}{\checkmark} & \textcolor{gray!50}{\checkmark} & $\pm$ & -
    & 51.5  & 43.7  & 85.3   & 59.5 \\
    
    (4)& & \textcolor{gray!50}{\checkmark} & \textcolor{gray!50}{\checkmark} & - & \checkmark
    & 52.8  & 40.5  & 87.4   & 62.8 \\
    
    (5)& & \textcolor{gray!50}{\checkmark} & \textcolor{gray!50}{\checkmark} & $\pm$ & \checkmark
    & 55.3  & 38.2  & 90.6   & 67.0 \\
    
    (6)& & \textcolor{gray!50}{\checkmark} & \textcolor{gray!50}{\checkmark} &  $+$
    & \checkmark & 52.2    &  36.4  &  88.1  & 65.5   \\
    \midrule
    (7) &PFlowNet & \checkmark & -  & \textcolor{gray!50}{$\pm$}
    & \textcolor{gray!50}{\checkmark}
    & 54.5    &  \textcolor{gray!50}{38.2}  &  89.4  & 66.4   \\
    (8) &PFlowNet & - & \checkmark & \textcolor{gray!50}{$\pm$}
    & \textcolor{gray!50}{\checkmark}
    & 49.2 & \textcolor{gray!50}{38.2} & 83.8 & 52.1  \\
    \bottomrule
    \end{tabular}%
    }
    \vspace{-0.8\baselineskip}
\end{wraptable}

\textbf{Framework \& Reward Design.}
\Cref{tab:2} validates the effectiveness of SFT and the synthetic flows, while the proposed RFT strategy yields further substantial gains (1, 2, 6). Furthermore, the quality and efficacy rewards exhibit a collaborative effect during the RFT, while the contrastive formulation also plays a crucial role (3 -- 6). Regarding the macro design, we examine the impact of input information during the reasoning stage (7, 8). Interestingly, incorporating external fine-grained visual features yields only marginal improvements; in contrast, removing the perceptual flow leads to severe performance degradation. This result indicates that the perceptual flow functions as more than a localization tool; it serves as a critical explicit semantic anchor. By translating visual thoughts into a structured textual prefix, the flow effectively conditions the LVLM's autoregressive generation, thereby bridging the semantic gap and guiding the reasoning trajectory more directly than raw visual features.

\par
\endgroup

\begingroup
\setlength{\columnsep}{0.05\textwidth}  
\setlength{\intextsep}{0pt}               
\setlength{\abovecaptionskip}{2pt}
\setlength{\belowcaptionskip}{-6pt}

\begin{wrapfigure}{r}{0.50\textwidth}
    \centering
    \includegraphics[width=\linewidth]{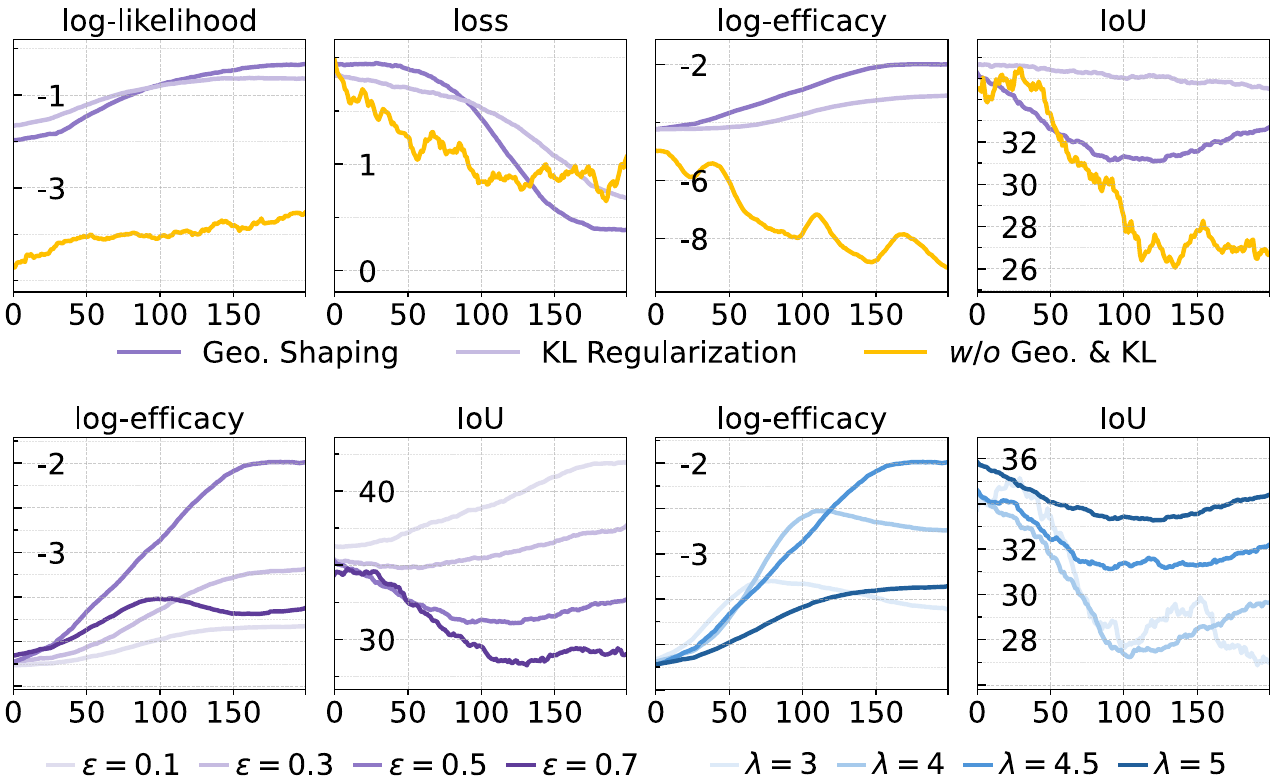}
    \caption{Ablation study on the geometric shaping scheme.}
    \label{fig:10}
    \vspace{-0.5\baselineskip}
\end{wrapfigure}

\textbf{Geometric Shaping.}
As illustrated in~\Cref{fig:10}, the geometric shaping stabilizes training and permits broader exploration than typical KL regularization, effectively mitigating SFT inductive bias. The initial drop and subsequent resurgence in $d_{\rm IoU}$ reflect a healthy transition from early exploration to late exploitation, yielding reliable yet high-efficacy perceptual behaviors. Further ablations of vicinal radius $\varepsilon$ and shaping intensity $\lambda$ show trends consistent with our analysis (\Cref{thm:1}): excessive radius $\varepsilon$ (\eg, 0.7) may lead to over-exploration into invalid supports, while a too small radius $\varepsilon$ enforces expert bias, preventing reasoning-oriented exploration. Likewise, small intensity $\lambda$ causes substantial training instability, whereas large $\lambda$ restricts exploration. Based on the ablation results, we set $\lambda=4.5,\ \varepsilon=0.5$ in this work.

\par
\endgroup

\begingroup
\setlength{\columnsep}{0.05\textwidth}   
\setlength{\intextsep}{0pt}               
\setlength{\abovecaptionskip}{2pt}
\begin{wraptable}{r}{0.5\textwidth}
    \centering 
    \footnotesize
    \caption{Cross-scale gains of PFlowNet across general-purpose visual tasks and fine-grained understanding.}
    \label{tab:b1}
    \setlength{\tabcolsep}{3pt}
    \renewcommand{\arraystretch}{1.1}
    \resizebox{\linewidth}{!}{
    \begin{tabular}{l | ccc | ccc c | cc}
    \toprule
    &\rotatebox{60}{V*} & \rotatebox{60}{HR 4k} & \rotatebox{60}{HR 8k} & \rotatebox{60}{TreeB} 
    & \rotatebox{60}{MME-rw} & \rotatebox{60}{CVB 2D} & \rotatebox{60}{MMB}
    & \rotatebox{60}{AI2D} &  \rotatebox{60}{ChartQA} \\
    \midrule
    Qwen3-VL 4B          & 74.8 & 73.5 & 67.1 & 42.2 
                         & 47.1 & 78.2 & 84.5
                         & 83.8 & 82.1
                         \\
    \midrule                     
    \emph{w} SFT       & 79.1 & 75.9 & 69.3 & 44.1
                         & 50.6 & 79.8 & 85.7
                         & 84.9 & 83.5
                         \\
    $\Delta$ \emph{vs.} Base Model             
                         & \upsft{4.3} & \upsft{2.4} & \upsft{2.2} & \upsft{1.9}
                         & \upsft{3.5} & \upsft{1.6} & \upsft{1.2} 
                         & \upsft{1.1} & \upsft{1.4}
                         
                         \\
    % \midrule
    \emph{w} RFT       & 83.5 & 77.9 & 70.9 & 46.8
                         & 55.3 & 82.0 & 87.2
                         & 90.0 & 84.5 
                         \\
    $\Delta$ \emph{vs.} Base Model             
                         & \up{8.7} & \up{4.4} & \up{3.8} & \up{4.6} 
                         & \up{8.2} & \up{3.8} & \up{2.7}
                         & \up{2.7}  & \up{2.3} \\
                         
    \midrule
    Qwen3-VL 32B         & 87.4 & 82.1 & 74.8 & 45.2   
                         & 52.0 & 81.5 & 87.7
                         & 89.0 & 83.1  \\
    \midrule
    \emph{w} SFT       & 89.1 & 83.5 & 76.3 & 46.7
                         & 54.3 & 82.8 & 88.4
                         & 89.9 & 84.2
                         \\
    $\Delta$ \emph{vs.} Base Model             
                         & \upsft{1.7} & \upsft{1.4} & \upsft{1.3}& \upsft{1.4} 
                         & \upsft{2.3} & \upsft{1.3} & \upsft{0.7}
                         & \upsft{0.9} & \upsft{1.0}
                         \\                      
    % \midrule
    \emph{w} RFT       & 91.6 & 85.5 & 77.5 & 49.2
                         & 58.8 & 85.2 & 89.8
                         & 91.5 & 85.7 
                         \\
    $\Delta$ \emph{vs.} Base Model             
                         & \up{4.2} & \up{3.1} & \up{2.8} & \up{4.0}  
                         & \up{6.8} & \up{3.7} & \up{2.1}
                        
                        & \up{2.5} & \up{2.6} \\
    \bottomrule
    \end{tabular}
    } 
    \vspace{-0.8\baselineskip}
\end{wraptable}

\textbf{Cross-Scale Evaluations.}
While our primary experiments utilize the Qwen3-VL 8B backbone to maintain parameter parity with baselines, we extend our evaluation to the Qwen3-VL 4B and 32B variants to verify the scalability of PFlowNet. We employ the consistent training recipe detailed in Appendix~\ref{app:B}, with specific adjustments for computational efficiency: \emph{the SFT and RFT phases are restricted to 1 and 2 epochs}, respectively, with an SFT global batch size of 128. 
As shown in~\Cref{tab:b1}, even under this computationally efficient training regime, backbones across different scales derive clear margins of improvement from our framework. Specifically, across general-purpose and fine-grained visual understanding tasks, our SFT stage yields average performance gains of 2.2\% and 1.3\% for Qwen3-VL 4B and Qwen3-VL 32B, respectively. Building upon this, our tailored RFT strategy further delivers substantial improvements exceeding 2.3\%, effectively demonstrating the scalability of our approach.

\par
\endgroup

\begin{figure*}[t]
	\centering
    \includegraphics[width=0.98\linewidth]{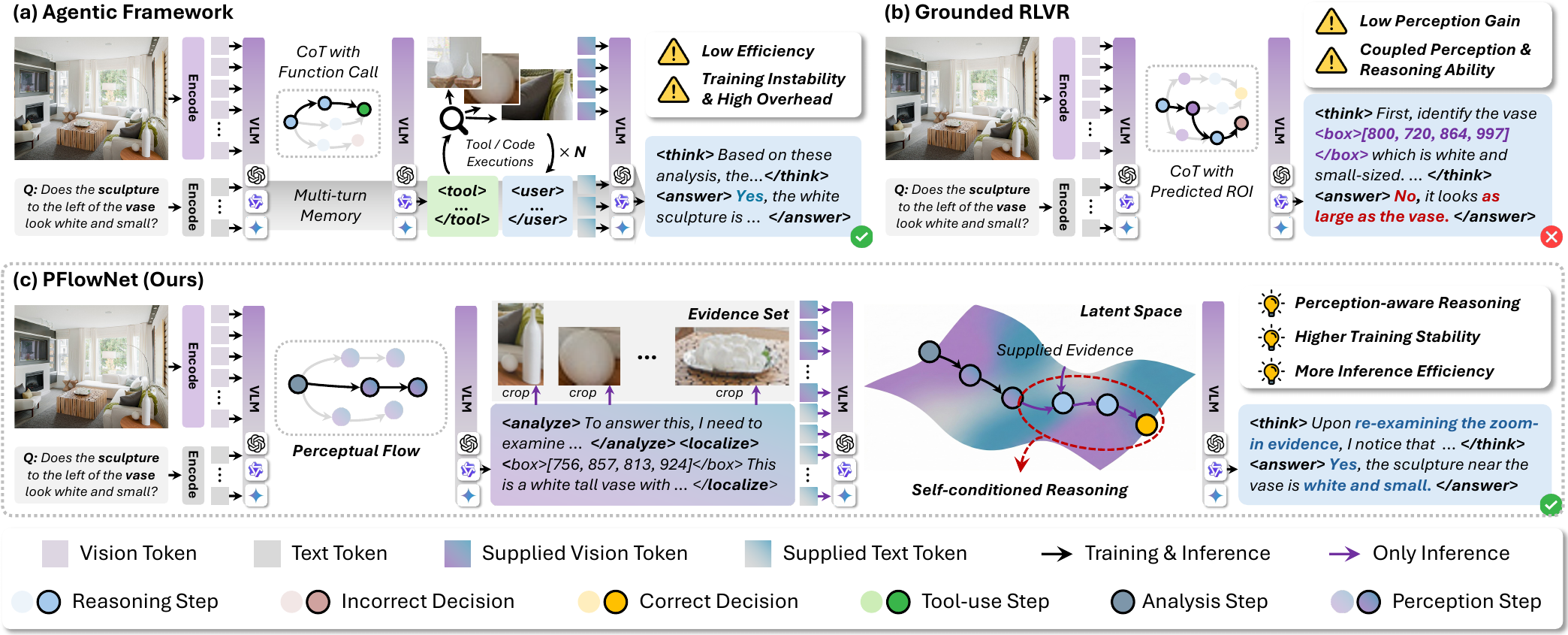}
	\caption{Framework comparisons between different paradigms, where (a) agentic frameworks rely on multi-turn tool executions for a "perceive-then-reason" process, (b) grounded RLVR integrates perception into reasoning in a single turn, and (c) the proposed PFlowNet decouples perception from reasoning via a two-stage perceptual flow, achieving robust yet efficient visually grounded reasoning.
    }
	\label{fig:2}
    \vspace{-2.2mm}
\end{figure*}

\vspace{-3mm}
\section{Related Work}
\vspace{-1mm}
\label{sec:2}

\textbf{Agentic Frameworks} equip LVLMs with dynamic image manipulation capabilities via multi-turn tool/code executions, thereby facilitating reliable visual reasoning, as shown in~\Cref{fig:2}(a). Specifically, several works explore the potential of enhancing visual reasoning by integrating external sandbox tools~\cite{li2026deepscan,hong2025deepeyesv2,su2025pixelreasoner,liu2505visual}. Thyme~\cite{zhang2025thyme} realizes ``\emph{thinking with images}'' by enabling the model to generate and execute code, while Visual Sketchpad~\cite{hu2024visual} empowers MLLMs with a sketching workspace to extend CoT with intermediate visual thoughts. Similarly, VaCoT~\cite{xu2025vacot} leverages visual tools to mitigate performance degradation on low-quality inputs. Despite these advances, CodeDance~\cite{song2025codedance} highlights a critical trade-off between tool utilization and intrinsic reasoning capabilities. Moreover, due to the entanglement of perception with complex invocations, these methods often suffer from excessive context and high latency. Different from these methods, the proposed PFlowNet utilizes \emph{structured text tokens} as the proxy of perceptual behaviors and efficiently achieves comparable high-quality visual reasoning based on the \emph{self-conditioned generation} with supplied fine-grained features (\Cref{fig:2}(c)).

\textbf{Grounded RLVR} maximizes the geometric consistency between intermediate visual rationales and external expert prior via reinforcement learning to regularize the reasoning process, as presented in~\Cref{fig:2}(b). Prior works~\cite{wang2025traceable,wang2025vgr,liu2025look,liu2025visual,shen2025vlm} represent perception via normalized \emph{bounding box coordinates} interleaved with stepwise reasoning, optimizing the policy by maximizing the Intersection over Union (IoU) with the ground truth. In contrast, ViGoRL~\cite{sarch2025grounded} and GUI-R1~\cite{luo2025gui} employ a point-based perceptual proxy, encouraging more precise perception through spatial distance constraints. Moreover, MIRG-RL~\cite{zheng2025mirg} explores Grounded RLVR in multi-image scenarios. However, our experiment results reveal the risk of rigidly enforcing alignment with biased and sparse expert priors. Furthermore, these methods often overlook the semantic coherence between visual rationales and the surrounding textual context, potentially compromising reasoning performance. These limitations are effectively mitigated by the multi-dimensional reward function in PFlowNet.

\vspace{-1mm}
\section{Conclusion} 
\vspace{-1mm}
This paper introduces PFlowNet, a novel framework based on structured perceptual flows that enables high-quality and interpretable visual reasoning. PFlowNet incorporates a carefully designed reinforcement fine-tuning strategy, comprising a tailored reward function with vicinal geometric shaping, which allows LVLMs to explore reasoning-oriented yet valid perceptual behaviors. Formal analysis establishes a provable performance guarantee for PFlowNet. Extensive experiments further demonstrate its superiority across both general-purpose and fine-grained tasks. Notably, the empirical analysis highlights PFlowNet’s excellent performance-efficiency balance and robust test-time scaling property.

\clearpage
\section*{Limitations and Future Work}
\label{app:B5}

The theoretical analysis (\Cref{thm:1,thm:2}) rests on idealized assumptions (\ref{assump:1},~\ref{assump:2}) and regularity conditions on valid support $S_V$ that may not strictly hold in practice; however, these bounds serve as a qualitative guide supported by our empirical validation. The proposed method involves two key hyperparameters ($\varepsilon$ and $\lambda$), whose optimal configurations may vary across base models and domains. Moreover, a comprehensive analysis of representative failure cases is provided in the \emph{Failure Case Analysis}~\ref{app:E2}.

Another inherent limitation of the proposed PFlowNet lies in its lack of \emph{adaptive perception}. Specifically, PFlowNet currently relies on a fixed structured reasoning format across questions of varying types and difficulty. While this design is beneficial for complex visual reasoning, it may be unnecessary for simple questions or certain STEM-oriented tasks, where the additional perceptual flow introduces reasoning overhead with limited marginal benefit due to salient visual evidence. It may also redistribute part of the model's capacity from direct problem solving to following the prescribed structure, which can lead to suboptimal results in some scenarios. 
Enabling adaptive visual reasoning, where the model dynamically adjusts its perceptual process according to the question difficulty and task context, is therefore an important direction for future work.

\bibliographystyle{plainnat}
\bibliography{main}

@String(AAAI = {AAAI})

@article{guo2025deepseekr1,
  title   = {DeepSeek-R1: Incentivizing reasoning capability in LLMs via reinforcement learning},
  author  = {Guo, Daya and Yang, Dejian and Zhang, Haowei and Song, Junxiao and Zhang, Ruoyu and Xu, Runxin and Zhu, Qihao and Ma, Shirong and Wang, Peiyi and Bi, Xiao and others},
  journal = {arXiv preprint arXiv:2501.12948},
  year    = {2025}
}

@article{chapelle2000vicinal,
  title={Vicinal risk minimization},
  author={Chapelle, Olivier and Weston, Jason and Bottou, L{\'e}on and Vapnik, Vladimir},
  journal={Advances in neural information processing systems},
  volume={13},
  year={2000}
}

@article{zheng2025mirg,
  title={MIRG-RL: Multi-Image Reasoning and Grounding with Reinforcement Learning},
  author={Zheng, Lihao and Chen, Jiawei and Shen, Xintian and Ma, Hao and Wei, Tao},
  journal={arXiv preprint arXiv:2509.21788},
  year={2025}
}

@article{luo2025gui,
  title={Gui-r1: A generalist r1-style vision-language action model for gui agents},
  author={Luo, Run and Wang, Lu and He, Wanwei and Chen, Longze and Li, Jiaming and Xia, Xiaobo},
  journal={arXiv preprint arXiv:2504.10458},
  year={2025}
}

@article{shen2025vlm,
  title={Vlm-r1: A stable and generalizable r1-style large vision-language model},
  author={Shen, Haozhan and Liu, Peng and Li, Jingcheng and Fang, Chunxin and Ma, Yibo and Liao, Jiajia and Shen, Qiaoli and Zhang, Zilun and Zhao, Kangjia and Zhang, Qianqian and others},
  journal={arXiv preprint arXiv:2504.07615},
  year={2025}
}

@article{liu2025visual,
  title={Visual-rft: Visual reinforcement fine-tuning},
  author={Liu, Ziyu and Sun, Zeyi and Zang, Yuhang and Dong, Xiaoyi and Cao, Yuhang and Duan, Haodong and Lin, Dahua and Wang, Jiaqi},
  journal={arXiv preprint arXiv:2503.01785},
  year={2025}
}

@article{hu2024visual,
  title={Visual sketchpad: Sketching as a visual chain of thought for multimodal language models},
  author={Hu, Yushi and Shi, Weijia and Fu, Xingyu and Roth, Dan and Ostendorf, Mari and Zettlemoyer, Luke and Smith, Noah A and Krishna, Ranjay},
  journal={Advances in Neural Information Processing Systems},
  volume={37},
  pages={139348--139379},
  year={2024}
}

@article{liu2505visual,
  title={Visual agentic reinforcement fine-tuning},
  author={Liu, Ziyu and Zang, Yuhang and Zou, Yushan and Liang, Zijian and Dong, Xiaoyi and Cao, Yuhang and Duan, Haodong and Lin, Dahua and Wang, Jiaqi},
  journal={URL https://arxiv. org/abs/2505.14246},
  year={2025}
}

@inproceedings{madan2023learning,
  title={Learning gflownets from partial episodes for improved convergence and stability},
  author={Madan, Kanika and Rector-Brooks, Jarrid and Korablyov, Maksym and Bengio, Emmanuel and Jain, Moksh and Nica, Andrei Cristian and Bosc, Tom and Bengio, Yoshua and Malkin, Nikolay},
  booktitle={International Conference on Machine Learning},
  pages={23467--23483},
  year={2023},
  organization={PMLR}
}

@article{vaswani2017attention,
  title={Attention is all you need},
  author={Vaswani, Ashish and Shazeer, Noam and Parmar, Niki and Uszkoreit, Jakob and Jones, Llion and Gomez, Aidan N and Kaiser, {\L}ukasz and Polosukhin, Illia},
  journal={Advances in neural information processing systems},
  volume={30},
  year={2017}
}

@article{song2025codedance,
  title={CodeDance: A Dynamic Tool-integrated MLLM for Executable Visual Reasoning},
  author={Song, Qi and Li, Honglin and Yu, Yingchen and Zhou, Haoyi and Yang, Lin and Bai, Song and She, Qi and Huang, Zilong and Zhao, Yunqing},
  journal={arXiv preprint arXiv:2512.17312},
  year={2025}
}

@article{liu2025look,
  title={Look as You Think: Unifying Reasoning and Visual Evidence Attribution for Verifiable Document RAG via Reinforcement Learning},
  author={Liu, Shuochen and Luo, Pengfei and Zhang, Chao and Chen, Yuhao and Zhang, Haotian and Liu, Qi and Kou, Xin and Xu, Tong and Chen, Enhong},
  journal={arXiv preprint arXiv:2511.12003},
  year={2025}
}

@article{xu2025vacot,
  title={VACoT: Rethinking Visual Data Augmentation with VLMs},
  author={Xu, Zhengzhuo and Sun, Chong and Du, SiNan and Li, Chen and Lyu, Jing and Yuan, Chun},
  journal={arXiv preprint arXiv:2512.02361},
  year={2025}
}

@article{bai2511qwen3,
  title={Qwen3-vl technical report, 2025},
  author={Bai, Shuai and Cai, Yuxuan and Chen, Ruizhe and Chen, Keqin and Chen, Xionghui and Cheng, Zesen and Deng, Lianghao and Ding, Wei and Gao, Chang and Ge, Chunjiang and others},
  journal={URL https://arxiv. org/abs/2511.21631},
  year={2025}
}

@inproceedings{radford2021learning,
  title={Learning transferable visual models from natural language supervision},
  author={Radford, Alec and Kim, Jong Wook and Hallacy, Chris and Ramesh, Aditya and Goh, Gabriel and Agarwal, Sandhini and Sastry, Girish and Askell, Amanda and Mishkin, Pamela and Clark, Jack and others},
  booktitle={International conference on machine learning},
  pages={8748--8763},
  year={2021},
  organization={PmLR}
}

@article{liu2024survey,
  title={A survey on hallucination in large vision-language models},
  author={Liu, Hanchao and Xue, Wenyuan and Chen, Yifei and Chen, Dapeng and Zhao, Xiutian and Wang, Ke and Hou, Liping and Li, Rongjun and Peng, Wei},
  journal={arXiv preprint arXiv:2402.00253},
  year={2024}
}

@inproceedings{gunjal2024detecting,
  title={Detecting and preventing hallucinations in large vision language models},
  author={Gunjal, Anisha and Yin, Jihan and Bas, Erhan},
  booktitle={Proceedings of the AAAI Conference on Artificial Intelligence},
  volume={38},
  pages={18135--18143},
  year={2024}
}

@article{chen2024multi,
  title={Multi-object hallucination in vision language models},
  author={Chen, Xuweiyi and Ma, Ziqiao and Zhang, Xuejun and Xu, Sihan and Qian, Shengyi and Yang, Jianing and Fouhey, David and Chai, Joyce},
  journal={Advances in Neural Information Processing Systems},
  volume={37},
  pages={44393--44418},
  year={2024}
}

@inproceedings{liu2024grounding,
  title={Grounding dino: Marrying dino with grounded pre-training for open-set object detection},
  author={Liu, Shilong and Zeng, Zhaoyang and Ren, Tianhe and Li, Feng and Zhang, Hao and Yang, Jie and Jiang, Qing and Li, Chunyuan and Yang, Jianwei and Su, Hang and others},
  booktitle={European conference on computer vision},
  pages={38--55},
  year={2024},
  organization={Springer}
}

@article{liu2023llava,
  title={Visual instruction tuning},
  author={Liu, Haotian and Li, Chunyuan and Wu, Qingyang and Lee, Yong Jae},
  journal={Advances in neural information processing systems},
  volume={36},
  pages={34892--34916},
  year={2023}
}

@misc{liu2024llavanext,
  title        = {LLaVA-NeXT: Improved reasoning, OCR, and world knowledge},
  howpublished = {\url{https://llava-vl.github.io/blog/2024-01-30-llava-next/}},
  author       = {Liu, Haotian and Li, Chunyuan and Li, Yuheng and Li, Bo and Zhang, Yuanhan and Shen, Sheng and Lee, Yong Jae},
  year         = {2024}
}

@article{bai2025qwen25vl,
  title   = {Qwen2.5-VL Technical Report},
  author  = {Bai, Shuai and Chen, Keqin and Liu, Xuejing and Wang, Jialin and Ge, Wenbin and Song, Sibo and Dang, Kai and Wang, Peng and Wang, Shijie and Tang, Jun and others},
  journal = {arXiv preprint arXiv:2502.13923},
  year    = {2025}
}

@article{li2024llavaov,
  title   = {LLaVA-OneVision: Easy Visual Task Transfer},
  author  = {Li, Bo and Zhang, Yuanhan and Guo, Dong and Zhang, Renrui and Li, Feng and Zhang, Hao and Zhang, Kaichen and Zhang, Peiyuan and Li, Yanwei and Liu, Ziwei and others},
  journal = {arXiv preprint arXiv:2408.03326},
  year    = {2024}
}

@article{zhu2025internvl3,
  title   = {InternVL3: Exploring advanced training and test-time recipes for open-source multimodal models},
  author  = {Zhu, Jinguo and Wang, Weiyun and Chen, Zhe and Liu, Zhaoyang and Ye, Shenglong and Gu, Lixin and Tian, Hao and Duan, Yuchen and Su, Weijie and Shao, Jie and others},
  journal = {arXiv preprint arXiv:2504.10479},
  year    = {2025}
}

@inproceedings{wu2024vstar,
  title     = {V*: Guided Visual Search as a Core Mechanism in Multimodal LLMs},
  author    = {Wu, Penghao and Xie, Saining},
  booktitle = {Proceedings of the Computer Vision and Pattern Recognition Conference},
  pages     = {13084--13094},
  year      = {2024}
}

@inproceedings{wang2025hrbench,
  title={Divide, conquer and combine: A training-free framework for high-resolution image perception in multimodal large language models},
  author={Wang, Wenbin and Ding, Liang and Zeng, Minyan and Zhou, Xiabin and Shen, Li and Luo, Yong and Yu, Wei and Tao, Dacheng},
  booktitle={Proceedings of the AAAI Conference on Artificial Intelligence},
  volume={39},
  pages={7907--7915},
  year={2025}
}

@article{sarch2025grounded,
  title   = {Grounded Reinforcement Learning for Visual Reasoning},
  author  = {Sarch, Gabriel and Saha, Snigdha and Khandelwal, Naitik and Jain, Ayush and Tarr, Michael J and Kumar, Aviral and Fragkiadaki, Katerina},
  journal = {arXiv preprint arXiv:2505.23678},
  year    = {2025}
}

@article{su2025pixelreasoner,
  title   = {Pixel Reasoner: Incentivizing Pixel-Space Reasoning with Curiosity-Driven Reinforcement Learning},
  author  = {Su, Alex and Wang, Haozhe and Ren, Weimin and Lin, Fangzhen and Chen, Wenhu},
  journal = {arXiv preprint arXiv:2505.15966},
  year    = {2025}
}

@article{zheng2025deepeyes,
  title   = {DeepEyes: Incentivizing ``Thinking with Images'' via Reinforcement Learning},
  author  = {Zheng, Ziwei and Yang, Michael and Hong, Jack and Zhao, Chenxiao and Xu, Guohai and Yang, Le and Shen, Chao and Yu, Xing},
  journal = {arXiv preprint arXiv:2505.14362},
  year    = {2025}
}

@article{hong2025deepeyesv2,
  title={DeepEyesV2: Toward Agentic Multimodal Model},
  author={Hong, Jack and Zhao, Chenxiao and Zhu, ChengLin and Lu, Weiheng and Xu, Guohai and Yu, Xing},
  journal={arXiv preprint arXiv:2511.05271},
  year={2025}
}

@article{yu2025zoom,
  title   = {Zoom-Refine: Boosting High-Resolution Multimodal Understanding via Localized Zoom and Self-Refinement},
  author  = {Yu, Xuan and Guan, Dayan and Yang, Michael Ying and Gu, Yanfeng},
  journal = {arXiv preprint arXiv:2506.01663},
  year    = {2025}
}

@inproceedings{li2025dyfo,
  title={Dyfo: A training-free dynamic focus visual search for enhancing lmms in fine-grained visual understanding},
  author={Li, Geng and Xu, Jinglin and Zhao, Yunzhen and Peng, Yuxin},
  booktitle={Proceedings of the Computer Vision and Pattern Recognition Conference},
  pages={9098--9108},
  year={2025}
}

@misc{gpt4o,
  howpublished = {\url{https://openai.com/index/gpt-4o-system-card/}},
  author       = {OpenAI},
  title        = {OpenAI-GPT-4o},
  year         = {2024}
}

@misc{o3,
  howpublished = {\url{https://openai.com/index/introducing-o3-and-o4-mini/}},
  author       = {OpenAI},
  title        = {OpenAI-o3},
  year         = {2025}
}

@misc{gemini-3-pro,
  howpublished = {\url{https://deepmind.google/models/gemini/pro/}},
  author       = {DeepMind},
  title        = {Gemini-3-Pro},
  year         = {2025}
}

@misc{gemini-3-flash,
  howpublished = {\url{https://deepmind.google/models/gemini/flash/}},
  author       = {DeepMind},
  title        = {Gemini-3-Flash},
  year         = {2025}
}

@article{zhang2025thyme,
  title   = {Thyme: Think Beyond Images},
  author  = {Zhang, Yi{-}Fan and Lu, Xingyu and Yin, Shukang and Fu, Chaoyou and Chen, Wei and Hu, Xiao and Wen, Bin and Jiang, Kaiyu and Liu, Changyi and Zhang, Tianke and others},
  journal = {arXiv preprint arXiv:2508.11630},
  year    = {2025}
}

@article{wang2025traceable,
  title   = {Traceable Evidence Enhanced Visual Grounded Reasoning: Evaluation and Methodology},
  author  = {Wang, Haochen and Li, Xiangtai and Huang, Zilong and Wang, Anran and Wang, Jiacong and Zhang, Tao and Zheng, Jiani and Bai, Sule and Kang, Zijian and Feng, Jiashi and others},
  journal = {arXiv preprint arXiv:2507.07999},
  year    = {2025}
}

@article{wang2025vgr,
  title   = {VGR: Visual Grounded Reasoning},
  author  = {Wang, Jiacong and Kang, Zijian and Wang, Haochen and Jiang, Haiyong and Li, Jiawen and Wu, Bohong and Wang, Ya and Ran, Jiao and Liang, Xiao and Feng, Chao and others},
  journal = {arXiv preprint arXiv:2506.11991},
  year    = {2025}
}

@article{hu2024dawn,
  title={The dawn of gui agent: A preliminary case study with claude 3.5 computer use},
  author={Hu, Siyuan and Ouyang, Mingyu and Gao, Difei and Shou, Mike Zheng},
  journal={arXiv preprint arXiv:2411.10323},
  year={2024}
}

@misc{openai2025operator,
  title={Operator: A Computer-Using Agent},
  author={OpenAI},
  year={2025},
  howpublished={\url{https://openai.com/index/operator-system-card/}},
  note={System Card and Technical Report}
}

@article{team2025kimi,
  title={Kimi-vl technical report},
  author={Team, Kimi and Du, Angang and Yin, Bohong and Xing, Bowei and Qu, Bowen and Wang, Bowen and Chen, Cheng and Zhang, Chenlin and Du, Chenzhuang and Wei, Chu and others},
  journal={arXiv preprint arXiv:2504.07491},
  year={2025}
}

@article{guo2025seed1,
  title={Seed1. 5-vl technical report},
  author={Guo, Dong and Wu, Faming and Zhu, Feida and Leng, Fuxing and Shi, Guang and Chen, Haobin and Fan, Haoqi and Wang, Jian and Jiang, Jianyu and Wang, Jiawei and others},
  journal={arXiv preprint arXiv:2505.07062},
  year={2025}
}

@inproceedings{cheng2024seeclick,
  title={Seeclick: Harnessing gui grounding for advanced visual gui agents},
  author={Cheng, Kanzhi and Sun, Qiushi and Chu, Yougang and Xu, Fangzhi and YanTao, Li and Zhang, Jianbing and Wu, Zhiyong},
  booktitle={Proceedings of the 62nd Annual Meeting of the Association for Computational Linguistics (Volume 1: Long Papers)},
  pages={9313--9332},
  year={2024}
}

@inproceedings{wuatlas,
  title={OS-ATLAS: Foundation Action Model for Generalist GUI Agents},
  author={Wu, Zhiyong and Wu, Zhenyu and Xu, Fangzhi and Wang, Yian and Sun, Qiushi and Jia, Chengyou and Cheng, Kanzhi and Ding, Zichen and Chen, Liheng and Liang, Paul Pu and others},
  booktitle={The Thirteenth International Conference on Learning Representations},
  year={2024}
}

@article{gou2024navigating,
  title={Navigating the digital world as humans do: Universal visual grounding for gui agents},
  author={Gou, Boyu and Wang, Ruohan and Zheng, Boyuan and Xie, Yanan and Chang, Cheng and Shu, Yiheng and Sun, Huan and Su, Yu},
  journal={arXiv preprint arXiv:2410.05243},
  year={2024}
}

@article{qin2025ui,
  title={Ui-tars: Pioneering automated gui interaction with native agents},
  author={Qin, Yujia and Ye, Yining and Fang, Junjie and Wang, Haoming and Liang, Shihao and Tian, Shizuo and Zhang, Junda and Li, Jiahao and Li, Yunxin and Huang, Shijue and others},
  journal={arXiv preprint arXiv:2501.12326},
  year={2025}
}

@article{zhang2024mme,
  title={Mme-realworld: Could your multimodal llm challenge high-resolution real-world scenarios that are difficult for humans?},
  author={Zhang, Yi-Fan and Zhang, Huanyu and Tian, Haochen and Fu, Chaoyou and Zhang, Shuangqing and Wu, Junfei and Li, Feng and Wang, Kun and Wen, Qingsong and Zhang, Zhang and others},
  journal={arXiv preprint arXiv:2408.13257},
  year={2024}
}

@article{li2024multimodal,
  title={Multimodal arxiv: A dataset for improving scientific comprehension of large vision-language models},
  author={Li, Lei and Wang, Yuqi and Xu, Runxin and Wang, Peiyi and Feng, Xiachong and Kong, Lingpeng and Liu, Qi},
  journal={arXiv preprint arXiv:2403.00231},
  year={2024}
}

@article{jiang2025vlm,
  title={VLM-R$^3$: Region Recognition, Reasoning, and Refinement for Enhanced Multimodal Chain-of-Thought},
  author={Jiang, Chaoya and Heng, Yongrui and Ye, Wei and Yang, Han and Xu, Haiyang and Yan, Ming and Zhang, Ji and Huang, Fei and Zhang, Shikun},
  journal={arXiv preprint arXiv:2505.16192},
  year={2025}
}

@article{wang2025sota,
  title={Sota with less: Mcts-guided sample selection for data-efficient visual reasoning self-improvement},
  author={Wang, Xiyao and Yang, Zhengyuan and Feng, Chao and Lu, Hongjin and Li, Linjie and Lin, Chung-Ching and Lin, Kevin and Huang, Furong and Wang, Lijuan},
  journal={arXiv preprint arXiv:2504.07934},
  year={2025}
}

@article{zheng2024llamafactory,
  title={Llamafactory: Unified efficient fine-tuning of 100+ language models},
  author={Zheng, Yaowei and Zhang, Richong and Zhang, Junhao and Ye, Yanhan and Luo, Zheyan and Feng, Zhangchi and Ma, Yongqiang},
  journal={arXiv preprint arXiv:2403.13372},
  year={2024}
}

@article{loshchilov2017decoupled,
  title={Decoupled weight decay regularization},
  author={Loshchilov, Ilya and Hutter, Frank},
  journal={arXiv preprint arXiv:1711.05101},
  year={2017}
}

@misc{vonwerra2022trl,
  author = {Leandro von Werra and Younes Belkada and Lewis Tunstall and Edward Beeching and Tristan Thrush and Nathan Lambert and Shengyi Huang and Kashif Rasul and Quentin Gallouédec},
  title = {TRL: Transformer Reinforcement Learning},
  year = {2020},
  publisher = {GitHub},
  journal = {GitHub repository},
  howpublished = {\url{https://github.com/huggingface/trl}}
}

@article{yue2025does,
  title={Does reinforcement learning really incentivize reasoning capacity in llms beyond the base model?},
  author={Yue, Yang and Chen, Zhiqi and Lu, Rui and Zhao, Andrew and Wang, Zhaokai and Song, Shiji and Huang, Gao},
  journal={arXiv preprint arXiv:2504.13837},
  year={2025}
}

@inproceedings{liu2024mmbench,
  title={Mmbench: Is your multi-modal model an all-around player?},
  author={Liu, Yuan and Duan, Haodong and Zhang, Yuanhan and Li, Bo and Zhang, Songyang and Zhao, Wangbo and Yuan, Yike and Wang, Jiaqi and He, Conghui and Liu, Ziwei and others},
  booktitle={European conference on computer vision},
  pages={216--233},
  year={2024},
  organization={Springer}
}

@article{li2023evaluating,
  title={Evaluating object hallucination in large vision-language models},
  author={Li, Yifan and Du, Yifan and Zhou, Kun and Wang, Jinpeng and Zhao, Wayne Xin and Wen, Ji-Rong},
  journal={arXiv preprint arXiv:2305.10355},
  year={2023}
}

@inproceedings{guan2024hallusionbench,
  title={Hallusionbench: an advanced diagnostic suite for entangled language hallucination and visual illusion in large vision-language models},
  author={Guan, Tianrui and Liu, Fuxiao and Wu, Xiyang and Xian, Ruiqi and Li, Zongxia and Liu, Xiaoyu and Wang, Xijun and Chen, Lichang and Huang, Furong and Yacoob, Yaser and others},
  booktitle={Proceedings of the IEEE/CVF Conference on Computer Vision and Pattern Recognition},
  pages={14375--14385},
  year={2024}
}

@inproceedings{kembhavi2016diagram,
  title={A diagram is worth a dozen images},
  author={Kembhavi, Aniruddha and Salvato, Mike and Kolve, Eric and Seo, Minjoon and Hajishirzi, Hannaneh and Farhadi, Ali},
  booktitle={European conference on computer vision},
  pages={235--251},
  year={2016},
  organization={Springer}
}

@inproceedings{masry2022chartqa,
  title={Chartqa: A benchmark for question answering about charts with visual and logical reasoning},
  author={Masry, Ahmed and Do, Xuan Long and Tan, Jia Qing and Joty, Shafiq and Hoque, Enamul},
  booktitle={Findings of the association for computational linguistics: ACL 2022},
  pages={2263--2279},
  year={2022}
}

@article{wang2024measuring,
  title={Measuring multimodal mathematical reasoning with math-vision dataset},
  author={Wang, Ke and Pan, Junting and Shi, Weikang and Lu, Zimu and Ren, Houxing and Zhou, Aojun and Zhan, Mingjie and Li, Hongsheng},
  journal={Advances in Neural Information Processing Systems},
  volume={37},
  pages={95095--95169},
  year={2024}
}

@article{tong2024cambrian,
  title={Cambrian-1: A fully open, vision-centric exploration of multimodal llms},
  author={Tong, Peter and Brown, Ellis and Wu, Penghao and Woo, Sanghyun and IYER, Adithya Jairam Vedagiri and Akula, Sai Charitha and Yang, Shusheng and Yang, Jihan and Middepogu, Manoj and Wang, Ziteng and others},
  journal={Advances in Neural Information Processing Systems},
  volume={37},
  pages={87310--87356},
  year={2024}
}

@inproceedings{li2025screenspot,
  title={Screenspot-pro: Gui grounding for professional high-resolution computer use},
  author={Li, Kaixin and Meng, Ziyang and Lin, Hongzhan and Luo, Ziyang and Tian, Yuchen and Ma, Jing and Huang, Zhiyong and Chua, Tat-Seng},
  booktitle={Proceedings of the 33rd ACM International Conference on Multimedia},
  pages={8778--8786},
  year={2025}
}

@inproceedings{kwon2023efficient,
  title={Efficient Memory Management for Large Language Model Serving with PagedAttention},
  author={Woosuk Kwon and Zhuohan Li and Siyuan Zhuang and Ying Sheng and Lianmin Zheng and Cody Hao Yu and Joseph E. Gonzalez and Hao Zhang and Ion Stoica},
  booktitle={Proceedings of the ACM SIGOPS 29th Symposium on Operating Systems Principles},
  year={2023}
}

@article{li2026deepscan,
  title={Deepscan: A training-free framework for visually grounded reasoning in large vision-language models},
  author={Li, Yangfu and Zhan, Hongjian and Chen, Jiawei and Gong, Yuning and Liu, Qi and Lu, Yue},
  journal={arXiv preprint arXiv:2603.03857},
  year={2026}
}

@article{lyu2026struvis,
  title={StruVis: Enhancing Reasoning-based Text-to-Image Generation via Thinking with Structured Vision},
  author={Lyu, Yuanhuiyi and Lei, Kaiyu and Weng, Ziqiao and Zheng, Xu and Jiang, Lutao and Li, Teng and Li, Yangfu and Huang, Ziyuan and Zhang, Linfeng and Hu, Xuming},
  journal={arXiv preprint arXiv:2603.06032},
  year={2026}
}

\clearpage
\appendix
% \clearpage
\section{Omitted Technical Details}
\label{app:A}

\textbf{Roadmap.} 
\emph{We organize the theoretical analysis as follows.
In \textbf{Appendix~\ref{app:A1}}, we formalize the probabilistic framework and preliminary definitions.
\textbf{Appendix~\ref{app:A2}} provides the rigorous derivation of our variational objective, stemming from the general Sub-Trajectory Balance principle.
Building on this, we introduce necessary regularity assumptions to facilitate tractable analysis and establish two auxiliary lemmas in \textbf{Appendix~\ref{app:A3}}: \textbf{Lemma~\ref{lem:B1}} demonstrates that the shaped reward is proportional to an exponentially tilted posterior $P_\lambda$, \ie, $R(Z)\propto P_\lambda$, while \textbf{Lemma~\ref{lem:B2}} proves that the global optimum of the policy recovers this tilted distribution, \ie, $p_{\theta^\star}(Z\mid X)\propto R(Z)\propto P_\lambda$.
Finally, in \textbf{Appendix~\ref{app:A4}}, we present the complete proofs of the main theorems.
By bridging the optimal policy $p_{\theta^\star}$ and the target valid posterior $P_{\rm V}$ via the tilted distribution $P_\lambda$, we derive the Total Variation (TV) bound in \textbf{Theorem~\ref{thm:1}}.
We conclude with an algebraic analysis of this bound in \textbf{Theorem~\ref{thm:2}}, confirming that PFlowNet provides strictly tighter guarantees compared to limiting baselines.}

\subsection{Preliminaries}
\label{app:A1}

\emph{Basic variables and flow notation.}
Let $(X,Y,E)\sim P_{\rm data}$ denote a sample from the RFT dataset, where
\(X\coloneqq\langle I,T\rangle\) consists of an image $I$ and an instruction $T$,
\(Y=(y_1,\dots,y_L)\) is the response sequence generated by the verifier when given the synthetic flow $Z_{\rm s}(X,E)$, and $E$ is a reference set of RoIs (\eg, expert evidence) used for vicinal geometry shaping. For theoretical
analysis, we assume an intractable target joint distribution
\[
P(X,Y,Z,\top),
\]
defined over inputs $X$, outputs $Y$, latent perceptual flows $Z$, and a terminal symbol $\top$ that marks the end of a flow. Following Definition~\ref{def:1}, a perceptual flow is written as a finite trajectory
\[
Z \;=\; (z_0 \rightarrow z_1 \rightarrow \cdots \rightarrow z_K),
\qquad
z_k \;=\; \langle r_k,c_k\rangle \quad (k\ge 1),
\]
where $z_0$ is the planning state, and each perceptual state $z_k$ consists of a RoI $r_k$ and its caption $c_k$. We write
\[
r_{1:K}\coloneqq (r_1,\dots,r_K),
\qquad
c_{1:K}\coloneqq (c_1,\dots,c_K).
\]
Throughout the analysis, we condition on a fixed flow length $K$ for each input
instance, while allowing the geometric precision and traversal order of the RoIs to
vary across latent realizations.
For mathematical convenience, RoIs are represented by normalized coordinates
\(
r_k\in[0,1]^4,
\)
\eg, \(r_k=(x_1,y_1,x_2,y_2)\) with
\(0\le x_1<x_2\le 1\) and \(0\le y_1<y_2\le 1\). All geometric distances below are
defined on this normalized coordinate system. For any \(k\in\{0,1,\dots,K\}\), define
the prefix or sub-flow
\[
z_{0:k}
\;\coloneqq\;
(z_0 \rightarrow z_1 \rightarrow \cdots \rightarrow z_k),
\]
and use \(z_{0:k}\top\) to denote its terminated prefix.

\emph{Model factorization.}
Given an input $X$, let
\[
\mathcal{R}(X)=\{R_1(X),\dots,R_M(X)\}
\]
denote the support of all admissible unordered RoI bags appearing in valid flows under
the target process \(P(Z\mid X)\). Each
\(
R_j(X)=\{r_i^j\}_{i=1}^K
\)
is a set of RoIs with cardinality $K$ and represents one admissible unordered
realization of the intrinsic visual evidence. An ordered tuple \(r_{1:K}\) is a
permutation of a bag \(R\), denoted by
\(
r_{1:K}\in\mathbb{S}_R,
\)
where \(\mathbb{S}_R\) is the permutation class induced by the elements of \(R\).

PFlowNet parameterizes a variational distribution \(p_\theta(Z\mid X)\) over flows
and a conditional generator \(p_\theta(Y\mid X,Z)\). We do not assume a specific
architecture; it suffices that \(p_\theta(Z\mid X)\) induces a forward transition
kernel over states. For notational brevity, we use the generic autoregressive
factorization
\[
p_\theta(Z\mid X)
\;=\;
p_\theta(z_0\mid X)
\prod_{k=1}^K
p_\theta(z_k\mid z_{0:k-1},X),
\label{equ:F1}
\tag{A1-1}
\]
and define forward/backward sub-trajectory probabilities as in Eq.~\eqref{equ:6}.

\emph{Structural hypotheses.} We impose three hypotheses used in the subsequent analysis. First, the planning state is assumed to be uniquely determined by the input $X$: there exists a deterministic mapping $g$ such that
\[
z_0 \;=\; g(X),
\qquad
P(z_0\mid \cdot,X)
\;=\;
\mathbf{1}\{z_0=g(X)\}.
\tag{A1-2}
\label{equ:C1}
\]
Second, for each RoI $r$, define a deterministic crop operator
\[
I_r \;=\; \mathrm{Crop}(r,I),
\qquad
I_{r_{j:k}}\coloneqq (I_{r_j},\dots,I_{r_k}).
\]
The captioning likelihood is assumed to depend on $I$ only through the cropped visual
content:
\[
P(c_{j:k}\mid I,r_{j:k})
\;=\;
P(c_{j:k}\mid I_{r_{j:k}}).
\tag{A1-3}
\label{equ:C2}
\]
We also use \(I\setminus I_r\) to denote the complement context outside the RoI $r$.
Third, since the contrastive caption term applies only to perceptual states
\(z_{k\ge 1}\), we set the boundary condition at the planning state as
\[
\frac{P^+(z_0)}{P^-(z_0)}
\;\coloneqq\;
1,
\qquad
\text{s.t.}\qquad
R(z_0\top)
=
P(Y\mid z_0\top,X).
\tag{A1-4}
\label{equ:C3}
\]

\emph{Vicinal support.} Assume there exists a conceptual golden RoI trajectory
\(
G_R=(g_1,\dots,g_K)
\)
that captures the intrinsic visual evidence for $X$. For a tolerance parameter
\(\sigma\in[0,1]\), define the support of valid flows as
\[
\mathcal{S}_{\rm V}
\;\coloneqq\;
\Big\{
Z:
d_{\rm IoU}
\big(
\{r_1,\dots,r_K\},
\{g_1,\dots,g_K\}
\big)
\le \sigma
\Big\},
\]
which is assumed to be nonempty. Its posterior mass is denoted by
\[
s_{\rm V}
\;\coloneqq\;
P(\mathcal{S}_{\rm V}\mid X,Y).
\tag{A1-5}
\label{equ:S1}
\]

Given a reference RoI set $E$ and radius \(\varepsilon\in[0,1]\), define the
\(\varepsilon\)-vicinity of $E$ by
\[
\mathcal{B}_\varepsilon(E)
\;\coloneqq\;
\Big\{
z_{0:k}:
d_{\rm IoU}
\big(
\{r_1,\dots,r_k\},
E
\big)
\le \varepsilon
\Big\}.
\]
Given a complete flow $Z\in\mathcal{B}_\varepsilon(E)$, its terminal RoI set \(\{r_1,\dots,r_K\}\) lies within the \(\varepsilon\)-vicinity of
$E$. We define
\[
s_{\mathcal{B}}
\;\coloneqq\;
P(\mathcal{B}_\varepsilon(E)\mid X,Y),
\qquad
q
\;\coloneqq\;
\frac{s_{\mathcal{B}}}{s_{\rm V}}.
\]
We assume \(\varepsilon\) is sufficiently small such that
\[
\mathcal{B}_\varepsilon(E)
\subseteq
\mathcal{S}_{\rm V},
\qquad
\text{hence}
\qquad
0\le s_{\mathcal{B}}\le s_{\rm V}\le 1.
\tag{A1-6}
\label{equ:C4}
\]

\emph{Reward-induced tilted posterior.} For any \(\lambda\ge 0\), the shaping weight $\omega_\lambda$ induces an exponentially tilted posterior
\[
P_\lambda(Z\mid X,Y,E)
\;\coloneqq\;
\frac{
P(Z\mid X,Y)\,
\omega_\lambda(Z,E)
}{
\mathcal{Z}_\lambda
},
\qquad
\mathcal{Z}_\lambda
\;\coloneqq\;
\int
P(Z\mid X,Y)\,
\omega_\lambda(Z,E)
\,dZ.
\tag{A1-7}
\label{equ:P1}
\]
Since \(\omega_\lambda(Z,E)=1\) on \(\mathcal{B}_\varepsilon(E)\) and
\(\omega_\lambda(Z,E)=e^{-\lambda}\) on its complement, the normalizer admits the
closed form
\[
\begin{aligned}
\mathcal{Z}_\lambda
&=
\int_{\mathcal{B}_\varepsilon(E)}
P(Z\mid X,Y)\,dZ
+
e^{-\lambda}
\int_{\mathcal{B}_\varepsilon(E)^c}
P(Z\mid X,Y)\,dZ
\\
&=
P(\mathcal{B}_\varepsilon(E)\mid X,Y)
+
e^{-\lambda}
P(\mathcal{B}_\varepsilon(E)^c\mid X,Y)
\\
&=
s_{\mathcal{B}}
+
e^{-\lambda}
(1-s_{\mathcal{B}}).
\end{aligned}
\]
Using \(s_{\mathcal{B}}=q\,s_{\rm V}\), we equivalently have
\[
\mathcal{Z}_\lambda
=
q\,s_{\rm V}
+
e^{-\lambda}
\bigl(1-q\,s_{\rm V}\bigr).
\tag{A1-8}
\label{equ:Z1}
\]

Finally, define the valid-support posterior, which serves as the \emph{target distribution} in total-variation distance:
\[
P_{\rm V}(Z\mid X,Y)
\;\coloneqq\;
P(Z\mid X,Y,Z\in\mathcal{S}_{\rm V})
\;=\;
\frac{
P(Z\mid X,Y)\,
\mathbf{1}\{Z\in\mathcal{S}_{\rm V}\}
}{
s_{\rm V}
}.
\tag{A1-9}
\label{equ:Z2}
\]
\clearpage
\subsection{Derivation of Variational Objective}
\label{app:A2}
In this section, we provide the detailed algebraic derivation connecting the general Sub-Trajectory Balance (SubTB) objective in~\Cref{equ:6} to our specific variational loss function in~\Cref{equ:8}.
We begin with the squared log-difference form of the SubTB objective for a sub-trajectory $z_{i:j}$ (where $0 \le i < j \le K$), which is formulated as
\[
\mathcal{L}(z_{i:j}) \;=\; \left( \log \frac{\mathcal{F}(z_i) \cdot \mathcal{T}_F(z_{i:j})}{\mathcal{F}(z_j) \cdot \mathcal{T}_B(z_{j:i})} \right)^2.
\label{equ:L1}
\tag{A2-1}
\]

Based on the structural constraints of Perceptual Flow (Definition~\ref{def:1}) and the variational distribution $p_\theta$, we instantiate the terms in~\Cref{equ:L1} as follows. Given the tree-structured autoregressive generation, the backward path from any state is deterministic, effectively collapsing the backward transition to unity, \ie, $\mathcal{T}_B(z_{j:i}) = 1$. Conversely, the forward transition $\mathcal{T}_F(z_{i:j})$ is governed by the product of conditional likelihoods parameterized by $\theta$, yielding $\mathcal{T}_F(z_{i:j}) = \prod_{k=i+1}^j p_\theta(z_k \mid z_{0:k-1})$. Furthermore, we formulate the scalar flow $\mathcal{F}(z_k)$ at state $z_k$ as the terminal reward normalized by the termination policy, formally defined as $\mathcal{F}(z_k)\coloneqq R_\lambda(z_{0:k}\top)/p_\theta(\top \mid z_{0:k})$.

Substituting the above definitions into the ratio term of~\Cref{equ:L1}, denoted as $\Delta_{i,j}$, we have
\[
\begin{aligned}
\Delta_{i,j} \;&=\; \log \left(\frac{\mathcal{F}(z_i) \cdot \mathcal{T}_F(z_{i:j})}{\mathcal{F}(z_j) \cdot 1}\right) \\
&=\; \log \left(\frac{\left( \frac{R_\lambda(z_{0:i}\top)}{p_\theta(\top \mid z_{0:i})} \right) \cdot \prod_{k=i+1}^j p_\theta(z_k \mid z_{0:k-1})}{\left( \frac{R_\lambda(z_{0:j}\top)}{p_\theta(\top \mid z_{0:j})} \right)}\right) \\
&=\; \log \left(\left( \frac{R_\lambda(z_{0:i}\top)}{p_\theta(\top \mid z_{0:i})} \right)\cdot \prod_{k=i+1}^j p_\theta(z_k \mid z_{0:k-1}) \cdot \left( \frac{p_\theta(\top \mid z_{0:j})}{R_\lambda(z_{0:j}\top)} \right)\right) \\
&=\; \log \left(\frac{R_\lambda(z_{0:i}\top) \cdot \prod_{k=i+1}^j p_\theta(z_k \mid z_{0:k-1}) \cdot p_\theta(\top \mid z_{0:j})}{R_\lambda(z_{0:j}\top) \cdot p_\theta(\top \mid z_{0:i})}\right).
\end{aligned}
\]

Then, given a trajectory $Z=(z_0,\dots,z_K)$, we derive the loss by summing the squared log-ratios over all valid sub-flow $z_{i:j}$
\[
\mathcal{L}_{\rm vRFT}(Z,\theta)\;=\sum_{0\le i\le j\le K} \Delta_{i,j}^2   \;= \sum_{0\le i\le j\le K} \left( \log \frac{R_\lambda(z_{0:i}\top) \prod_{k=i+1}^j p_\theta(z_k \mid z_{0:k-1})\,p_\theta(\top\mid z_{0:j})}{R_\lambda(z_{0:j}\top)\,p_\theta(\top\mid z_{0:i})} \right)^2.
\label{equ:L2}
\tag{A2-2}
\]
Finally, we extend this formulation to the empirical data setting $(X, Y, E) \sim P_{\text{data}}$, where $E = \{e_l\}_{l=1}^L \subset \mathbb{N}^4$ denotes the set of expert-annotated Regions of Interest (RoIs). Consequently, the objective in~\Cref{equ:8} is given by 
\[
\mathcal{L}_{\rm vRFT}(\theta)=\mathbb{E}_{\substack{X,Y,E\sim P_{\rm data}\\ \{Z\}_{l=1}^L\sim p_{\theta}(\mathcal{Z}\mid X)}} \left[\sum_{0\le i\le j\le |Z|} \left( \log \frac{R_\lambda(z_{0:i}\top) \prod_{k=i+1}^j p_\theta(z_k \mid z_{0:k-1})\,p_\theta(\top\mid z_{0:j})}{R_\lambda(z_{0:j}\top)\,p_\theta(\top\mid z_{0:i})} \right)^2\right].
\label{equ:L3}
\tag{A2-3}
\]
where $\{Z\}_{l=1}^L\sim\mathcal{Z}$ denotes the group of sampled perceptual flows, and $\theta^\star = \arg\min_{\theta} \mathcal{L}_{\rm vRFT}(\theta)$.

\clearpage
\subsection{Assumptions and Auxiliary Lemmas}
\label{app:A3}

\begin{assumption}[Uniform Prior]
\label{assump:1}
Motivated by the empirical results, we adopt a uniform prior of \emph{admissible} RoI space $\mathcal{R}(X)$. Formally, given $X$, for all $Z(z_0,R,C)\sim P(Z\mid X)$ where $R\in\mathcal{R}(X)$ and $|R|=K$, we have
\[
\textstyle
P(R\mid X)=\frac{1}{|\mathcal{R}(X)|},\quad
P(r_{1:K}\mid R,X)=\frac{1}{K\,!},
\]    
where $r_{1:K}\!\in\!\mathbb{S}_R$ and $\mathbb{S}_R$ is the symmetric group of $\forall r\in R$.
\end{assumption}
\begin{assumption}[Faithful Captioning with Non-Informative Prior]
\label{assump:2}
We assume $\forall c\in C$ is faithful when conditioned on the cropped evidence $I_{r}$, and become non-informative when the evidence is absent (\ie, $I\setminus I_{r}$). Formally, let $\mathcal{C}$ be a finite space containing \emph{all} candidate captions $c$, we have
\[
\textstyle
P(c_{j:k}\mid I_{r_{j:k}})=\prod_{i=j}^k P(c_i\mid I_{r_i}),\ P(c\mid I\setminus I_r)=\frac{1}{|\mathcal{C}|}.
\]    
\end{assumption}

\begin{lemma}[Reward Consistency]
\label{lem:B1}
Under Assumptions~\ref{assump:1} and~\ref{assump:2}, for all $(X,Y,E)\sim P_{\rm data}$, the shaped multi-dimensional reward $R_\lambda(Z)$ satisfy
\[
R_\lambda(Z)\coloneqq R_\lambda(z_{0:k}\top)\propto P_\lambda(Z\mid X,Y,E).
\]
\end{lemma}

\begin{proof}
Leveraging the definition of multi-dimensional reward~\Cref{equ:mmreward}, given any sub-flow $z_{i:j}$, we have
\[
R_\lambda(z_{0:k}\top)\;\coloneqq\;
\left(\prod_{i=1}^k \frac{P(c_i\mid I_{r_i})}{P(c_i\mid I\setminus I_{r_i})}\right)
\ P(Y\mid z_{0:k}\top,X)\ \omega_\lambda(z_{0:k},E).
% \label{equ:A1}
% \tag{A1}
\]
Invoking Assumption~\ref{assump:2}, for all $i\in\{1,\dots,K\}$, we have \( P(c_i\mid I\setminus I_{r_i})=1/|\mathcal{C}|\) where $\mathcal{C}$ is a finite space containing all candidate captions $c$. Plugging this into the above yields
\begin{equation}
\begin{aligned}
R_\lambda(Z)
&=
\Big(\prod_{i=1}^K P(c_i\mid I_{r_i})\Big)\ |\mathcal{C}|^{K}\ P(Y\mid Z,X)\ \omega_\lambda(Z,E),\\
&\propto
\Big(\prod_{i=1}^K P(c_i\mid I_{r_i})\Big)\ P(Y\mid Z,X)\ \omega_\lambda(Z,E),
\end{aligned}
\label{equ:A1}
\tag{A3-1}
\end{equation}
where the proportionality absorbs the factor $|\mathcal{C}|^{K}$, which only depends on $X$.
By the factorization of Assump.~\ref{assump:2}, we have
\[
\prod_{i=1}^K P(c_i\mid I_{r_i})
\;=\;
P(c_{1:K}\mid I_{r_{1:K}}).
\]
Subsequently, \Cref{equ:A1} becomes
\[
\boxed{
R_\lambda(Z)\ \propto\ P(c_{1:K}\mid I_{r_{1:K}})\ P(Y\mid Z,X)\ \omega_\lambda(Z,E).
}
\label{equ:A2}
\tag{A3-2}
\]

For the standard posterior, by Bayes’ rule, we have
\[
P(Z\mid X,Y) \;=\; \frac{P(Y,Z\mid X)}{P(Y\mid X)},
\]
and $P(Y\mid X)$ is a normalizing constant independent of $Z$.
Thus,
\[
\begin{aligned}
P(Z\mid X,Y)\ &\propto\ P(Y,Z\mid X)\ = P(Z\mid X)\ P(Y\mid Z,X)\ \equiv \ P(R, z_0, r_{1:K}, c_{1:K}\mid X)\ P(Y\mid Z,X),\\    
&\propto\ \underbrace{P(R\mid X)\ P(z_0\mid R,X)\ P(r_{1:K}\mid z_0,R,X)\ P(c_{1:K}\mid r_{1:K}, z_0,R,X)}_{P(R, z_0, r_{1:K}, c_{1:K}\mid X)}\ 
P(Y\mid Z,X). \\
\end{aligned}
\label{equ:A3}
\tag{A3-3}
\]
Based on the structural constraint that the traversal order is a uniform permutation over the support set (Assumption~\ref{assump:1}), the specific sequence $r_{1:K}$ is conditionally independent of the deterministic planning state $z_0$ given $X$, yielding $P(r_{1:K}\mid z_0,R,X)=P(r_{1:K}\mid R,X)$. 
Leveraging the deterministic crop operator ${\rm Crop}:\langle I,r\rangle\mapsto I_r$, we have
\[
\begin{aligned}
P(c_{1:K}\mid r_{1:K}, z_0,R,X)&=P(c_{1:K}\mid r_{1:K}, R,I,T), \\
&=P(c_{1:K}\mid I_{r_{1:K}}, z_0,R,T), \\
&=P(c_{1:K}\mid \cdot, I_{r_{1:K}}).    
\end{aligned}
\]
By Assumption~\ref{assump:2} and~\Cref{equ:C2}, the caption $c$ is only conditioned on the attended region $I_{r}$; thus, we have
\[
P(c_{1:K}\mid r_{1:K}, z_0,R,X)=P(c_{1:K}\mid I_{r_{1:K}}).\\
\]
Therefore, substituting the above into the~\Cref{equ:A3}, the standard posterior $P(Z\mid X,Y)$ satisfies
\[
P(Z\mid X,Y)\ \propto\ \underbrace{P(R\mid X)}_{\rm Support\ Prior}\ \underbrace{P(z_0\mid \cdot,X)}_{\rm Deterministic\ State}\ \underbrace{P(r_{1:K}\mid R,X)}_{\rm Ordering\ Prior}\ \underbrace{P(c_{1:K}\mid I_{r_{1:K}})}_{\rm Caption\ Sequence}\ 
\underbrace{P(Y\mid Z,X)}_{\rm Perceptual\ Efficacy}.
\]
By leveraging Assumption~\ref{assump:1} (\ie, uniform support and ordering prior) and \Cref{equ:C1}, we have
\[
\begin{aligned}
P(Z\mid X,Y)\ &\propto\ \frac{1}{|\mathcal{R}(X)|K\,!}\ P(c_{1:K}\mid I_{r_{1:K}})\ P(Y\mid Z,X),\\
&\propto\ P(c_{1:K}\mid I_{r_{1:K}})\ P(Y\mid Z,X).
\end{aligned}
\tag{A3-4}
\label{equ:A4}
\]
Based on~\Cref{equ:P1,equ:A4}, we have
\[
\boxed{
\begin{aligned}
P_\lambda(Z\mid X,Y,E)
\;&\propto\;
P(Z\mid X,Y)\ \omega_\lambda(Z,E),\\
&\propto\; P(c_{1:K}\mid I_{r_{1:K}})\ P(Y\mid Z,X)\ \omega_\lambda(Z,E).
\end{aligned}
}
\tag{A3-5}
\label{equ:A5}
\]
Consequently, combining~\Cref{equ:A2} with~\Cref{equ:A5} gives
\[
R_\lambda(Z)\ \propto\ P_\lambda(Z\mid X,Y,E).
\]
This completes the proof. 
\end{proof}

\begin{lemma}[Posterior Matching Induced by the Variational Objective]
\label{lem:B2}
\emph{Under Assumption \ref{assump:1} and~\ref{assump:2}, suppose the policy $p_\theta$ is expressive and $\theta^\star$ globally minimizes $\mathcal{L}_{\rm vRFT}(\theta)$, for every $(X,Y,E)\sim P_{\rm data}$ we have
\[
p_{\theta^\star}(Z\mid X)\;\propto\; R_\lambda(Z)
\;\propto\; P_\lambda(Z\mid X,Y,E).
\]
}
\end{lemma}

\begin{proof}
Based on~\Cref{equ:F1}, the forward trajectory probability induced by policy $p_\theta$ is defined as
\[
p_\theta(Z,\top\mid X)
\;\coloneqq\;
\Big(\prod_{k=1}^{K} p_\theta(z_k\mid z_{0:k-1})\Big)\, p_\theta(\top\mid z_{0:K}).
\label{equ:B1}
\tag{A3-6}
\]
Based on~\Cref{equ:L2}, when policy $p_\theta$ is expressive enough and the optimization reaches a solution with $\mathcal{L}_{\rm VRFT}(Z,\theta)=0$ for all valid $(i,j)$, we have each squared term $\Delta_{i,j}$ equals $0$. Formally, when for every $0\le i\le j\le K$, the following holds
\[
\log\frac{R_\lambda(z_{0:i}\top)\prod_{k=i+1}^j p_{\theta^\star}(z_k\mid z_{0:k-1})\,p_{\theta^\star}(\top\mid z_{0:j})}{R_\lambda(z_{0:j}\top)\,p_{\theta^\star}(\top\mid z_{0:i})}
\;=\;0.
\]
Exponentiating the above yields the exact balance constraints
\[
R_\lambda(z_{0:i}\top)\Big(\prod_{k=i+1}^j p_{\theta^\star}(z_k\mid z_{0:k-1})\Big)\,p_{\theta^\star}(\top\mid z_{0:j})
\;=\;
R_\lambda(z_{0:j}\top)\,p_{\theta^\star}(\top\mid z_{0:i}).
\]
Taking $(i,j)=(0,K)$ in the previous equation gives
\[
R_\lambda(z_{0}\top)\underbrace{\Big(\prod_{k=1}^K p_{\theta^\star}(z_k\mid z_{0:k-1})\Big)\,p_{\theta^\star}(\top\mid z_{0:K})}_{p_{\theta^\star}(Z,\top\mid X)}
\;=\;
R_\lambda(z_{0:K}\top)\,p_{\theta^\star}(\top\mid z_{0}).
\label{equ:B2}
\tag{A3-7}
\]
The product term in~\eqref{equ:B2} is precisely the forward trajectory probability $p_{\theta^\star}(Z,\top\mid X)$ in~\eqref{equ:B1},
hence
\[
p_{\theta^\star}(Z,\top\mid X)
\;=\;
\frac{p_{\theta^\star}(\top\mid z_0)}{R_\lambda(z_0\top)}\; R_\lambda(z_{0:K}\top).
\label{equ:B3}
\tag{A3-8}
\]
Based on~\Cref{equ:C1,equ:C3}, $z_0$ is uniquely determined by $X$; thus, $R_\lambda(z_0\top)$, \ie, $P(Y\mid z_0\top ,X)$, is a normalization choice at the boundary. Hence, the prefactor $\frac{p_{\theta^\star}(\top\mid z_0)}{R_\lambda(z_0\top)}$ does not depend on the particular trajectory realization $Z$. Therefore, ~\Cref{equ:B3} implies the proportionality
\[
p_{\theta^\star}(Z\mid X)\ \coloneqq\ p_{\theta^\star}(Z,\top\mid X)\ \propto\ R_\lambda(Z).
% \label{equ:B4}
% \tag{B4}
\]
\ie, minimizing the variational loss $\mathcal{L}_{\rm VRFT}(Z,\theta)$ yields an optimum policy $p_{\theta^\star}(Z\mid X)$ that is proportional to the reward $R_\lambda(Z)$.
Thereby, based on the conclusion of \textbf{Lemma~\ref{lem:B1}}, \ie, $R_\lambda(z)\propto P_\lambda(Z\mid X,Y,E)$, we have
\[
\boxed{
p_{\theta^\star}(Z\mid X)\ \propto P_\lambda(Z\mid X,Y,E).
}
\label{equ:B4}
\tag{A3-9}
\]
This completes the proof. 
\end{proof}

% \clearpage
\subsection{Proofs}
\label{app:A4}
\textbf{Restatement of~\Cref{thm:1}} [Variation Distance Bound]\textbf{.}
\emph{Under Assumption~\ref{assump:1},~\ref{assump:2}, we suppose the valid support $\mathcal{S}_{\rm V}$ satisfies $d_{\rm eff}$-regularity, where $d_{\rm eff}$ is its effective dimension; thus, $\exists\kappa \ge 1$ such that $q \coloneqq s_{\mathcal{B}}/s_{\rm V} \ge \kappa(\varepsilon/\sigma)^{d_{\rm eff}}$. Suppose the model $p_\theta$ is expressive and let $\theta^\star$ be the global minimizer of $\mathcal{L}_{\rm vRFT}(\theta)$. The total variation distance between the policy $p_{\theta^\star}(Z\,|\, X)$ and the target posterior $P_{\rm V}(Z\,|\,X, Y)$ is bounded by:}
\[
D_{\rm TV}(p_{\theta^\star}(\cdot\mid X),P_{\rm V}(\cdot\mid X,Y))\; \le\; \frac{1}{2\,\mathcal{Z}_\lambda} \cdot \left(
q\,|s_{\rm V}-\mathcal{Z}_\lambda|
+(1-q)\,|e^{-\lambda}s_{\rm V}-\mathcal{Z}_\lambda|
+e^{-\lambda}(1-s_{\rm V})
\right).    
\]
\begin{remark}[Limit Analysis \emph{w.r.t.} $\lambda$]
We analyze the asymptotic behavior of the bound by evaluating the limits of $\lambda$ 
\begin{itemize}[leftmargin=*, nosep]
    \item[$\diamondsuit$] \emph{MLE Regime} ($\lambda \!\to\! 0$): Since $e^{-\lambda} \!\to\! 1$, the partition function $\mathcal{Z}_\lambda \!\to\! s_{\mathcal{B}} + (1-s_{\mathcal{B}}) \!=\! 1$. The bound simplifies to $\frac{1}{2}\big(q(1-s_{\rm V}) + (1-q)(1-s_{\rm V}) + (1-s_{\rm V})\big) = 1 - s_{\rm V}$. Dominated by inherent data variance ($s_{\rm V}$), PFlowNet discards geometric constraints and \emph{degrades to standard MLE}.
    
    \item[$\diamondsuit$] \emph{RLVR Regime} ($\lambda \!\to\! \infty$): Since $e^{-\lambda} \!\to\! 0$, we have $\mathcal{Z}_\lambda \!\to\! s_{\mathcal{B}}$. The bound becomes $\frac{1}{2s_{\mathcal{B}}}\big(q(s_{\rm V}-s_{\mathcal{B}}) + (1-q)s_{\mathcal{B}} + 0\big)$. Using $s_{\mathcal{B}} \!=\! q s_{\rm V}$, the numerator simplifies to $q s_{\rm V}(1-q) + s_{\rm V}q(1-q) = 2s_{\mathcal{B}}(1-q)$, yielding a final bound of $1-q$. Here, performance is bottlenecked by the expert bias ($q$), \emph{degenerating to expert-guided RLVR}.
\end{itemize}
\end{remark}
\begin{remark}[Limit Analysis \emph{w.r.t.} $\varepsilon$]
The vicinity radius $\varepsilon$ affects the bound through the coverage ratio $q$. As $\varepsilon \!\to\! 0$, the vicinity contracts to a singularity, implying $q \!\to\! 0$ and $s_{\mathcal{B}} \!\to\! 0$. Consequently, $\mathcal{Z}_\lambda \!\to\! e^{-\lambda}$. Substituting these into \Cref{thm:1}, the bound approaches $\frac{1}{2e^{-\lambda}}\big(0 + |e^{-\lambda}s_{\rm V}-e^{-\lambda}| + e^{-\lambda}(1-s_{\rm V})\big) = 1-s_{\rm V}$. This algebraic equivalence to the MLE bound confirms that as the reward signal becomes uninformative, the geometric guidance vanishes. Conversely, increasing $\varepsilon$ (where $\mathcal{B}_\varepsilon \subseteq \mathcal{S}_{\rm V}$) increases $q$, monotonically tightening the bound towards $(1-q)$. However, if $\varepsilon > \sigma$, the vicinity encompasses invalid regions, diluting the guidance and degrading performance.
\end{remark}

\begin{proof}
Since $\theta^\star$ acts as a global minimizer of the objective $\mathcal{L}_{\rm vRFT}$, which is formulated as an expectation over $(X,Y,E)\sim P_{\rm data}$~\eqref{equ:L3}, the optimality condition derived in \textbf{Lemma~\ref{lem:B2}} holds for $P_{\rm data}$-almost every data tuple. Thereby, we have:
\[
p_{\theta^\star}(Z\mid X) \;\propto\; P_\lambda(Z\mid X,Y,E),
\]
Since both sides are valid probability distributions normalized over the space of $Z$, this implies strict point-wise equality
\[
p_{\theta^\star}(Z\mid X) \;=\; P_\lambda(Z\mid X,Y,E).
\label{equ:Thm.1}
\tag{A4-1}
\]
\emph{Remark.} While a realizable parametric model $p_\theta(Z\mid X)$ cannot analytically depend on $Y,E$, this equality characterizes the ideal behavior of the policy at the global optimum of the variational objective for a given training instance. 

Then, recall the definition of total variation distance $D_{\rm TV}$:
\[
D_{\rm TV}(P,Q)
\;=\;\sup_{\mathcal{A}} |P(\mathcal{A})-Q(\mathcal{A})|
\;=\;\frac{1}{2}\int \bigl|p(z)-q(z)\bigr|\,dz,
\label{equ:Thm.2}
\tag{A4-2}
\]
where $p,q$ are densities w.r.t.\ a common base measure. Based on~\Cref{equ:Thm.1}, we will bound
\(
D_{\rm TV}(P_\lambda(\cdot\mid X,Y,E),P_{\rm V}(\cdot\mid X,Y)).
\)
Since $\mathcal{B}_\varepsilon(E)\subseteq\mathcal{S}_{\rm V}$~\eqref{equ:C4}, we partition the flow space $\Omega$, \ie, the support of $P(Z\mid X,Y)$, into three \emph{disjoint} measurable regions:
\[
\Omega
=\underbrace{\mathcal{B}_\varepsilon(E)}_{\rm vicinal}\ \sqcup\ \underbrace{\bigl(\mathcal{S}_{\rm V}\setminus \mathcal{B}_\varepsilon(E)\bigr)}_{\rm valid}\ \sqcup\ \underbrace{\mathcal{S}_{\rm V}^c}_{\rm invalid}.
\label{equ:Thm.3}
\tag{A4-3}
\]
Based on the definition of the energy weight $\omega_\lambda$, we have
\[
\omega_\lambda(z_{0:k},E)
\;\coloneqq\;
\exp\!\Big(-\lambda\cdot \mathbb{I}\{z_{0:k}\notin \mathcal{B}_\varepsilon(E)\}\Big)\quad \textrm{and}\quad \omega_\lambda(z_{0},E)\equiv1.
% \label{equ:R2}
% \tag{B2-8}
\]
Thereby, leveraging~\Cref{equ:P1}, for any $Z$:

$\diamondsuit$ If $Z\in \mathcal{B}_\varepsilon(E)$, then $\omega_\lambda(Z,E)=1$, hence
\[
P_\lambda(Z\mid X,Y,E)=\frac{P(Z\mid X,Y)}{\mathcal{Z}_\lambda}.
\label{equ:Thm.4}
\tag{A4-4}
\]
$\diamondsuit$ If $Z\notin \mathcal{B}_\varepsilon(E)$, then $\omega_\lambda(Z,E)=e^{-\lambda}$, hence
\[
P_\lambda(Z\mid X,Y,E)=\frac{e^{-\lambda}P(Z\mid X,Y)}{\mathcal{Z}_\lambda}.
\label{equ:Thm.5}
\tag{A4-5}
\]
Leveraging~\Cref{equ:Z2}, for the \emph{target} posterior:

$\diamondsuit$ If $Z\in\mathcal{S}_{\rm V}$, then
\[
P_{\rm V}(Z\mid X,Y)=\frac{P(Z\mid X,Y)}{s_{\rm V}}.
\label{equ:Thm.6}
\tag{A4-6}
\]
$\diamondsuit$ If $Z\notin\mathcal{S}_{\rm V}$, then
\[
P_{\rm V}(Z\mid X,Y)=0.
\label{equ:Thm.7}
\tag{A4-7}
\]

Based on~\Cref{equ:Thm.2,equ:Thm.3}, we have
\[
\begin{aligned}
2\,D_{\rm TV}(P_\lambda(Z\mid X,Y,E),P_{\rm V}(Z\mid X,Y))
&=\int_{\mathcal{B}_\varepsilon(E)} \Bigl|P_\lambda(Z\mid X,Y,E)-P_{\rm V}(Z\mid X,Y)\Bigr|\,dZ\\
&+\int_{\mathcal{S}_{\rm V}\setminus\mathcal{B}_\varepsilon(E)} \Bigl|P_\lambda(Z\mid X,Y,E)-P_{\rm V}(Z\mid X,Y)\Bigr|\,dZ\\
&+\int_{\mathcal{S}_{\rm V}^c} \Bigl|P_\lambda(Z\mid X,Y,E)-P_{\rm V}(Z\mid X,Y)\Bigr|\,dZ.
\end{aligned}
\label{equ:Thm.8}
\tag{A4-8}
\]
We then evaluate the three integrals separately.

For $Z\in\mathcal{B}_\varepsilon(E)\subseteq\mathcal{S}_{\rm V}$, substituting~\Cref{equ:Thm.4,equ:Thm.6} into the first term in~\Cref{equ:Thm.8}, we have
\[
\begin{aligned}
\int_{\mathcal{B}_\varepsilon(E)}
\Bigl|\tfrac{P(Z\mid X,Y)}{\mathcal{Z}_\lambda}-\tfrac{P(Z\mid X,Y)}{s_{\rm V}}\Bigr|\,dZ
&=\int_{\mathcal{B}_\varepsilon(E)}
P(Z\mid X,Y)\,\Bigl|\tfrac{1}{\mathcal{Z}_\lambda}-\tfrac{1}{s_{\rm V}}\Bigr|\,dZ\\
&=\Bigl|\tfrac{1}{\mathcal{Z}_\lambda}-\tfrac{1}{s_{\rm V}}\Bigr|
\int_{\mathcal{B}_\varepsilon(E)} P(Z\mid X,Y)\,dZ\\
&=\Bigl|\tfrac{1}{\mathcal{Z}_\lambda}-\tfrac{1}{s_{\rm V}}\Bigr|\ s_{\mathcal{B}}\\
&=\frac{|s_{\rm V}-\mathcal{Z}_\lambda|}{\mathcal{Z}_\lambda\,s_{\rm V}}\ s_{\mathcal{B}}.
\end{aligned}
\label{equ:Thm.9}
\tag{A4-9}
\]
Using $s_{\mathcal{B}}=q\,s_{\rm V}$ gives
\[
\int_{\mathcal{B}_\varepsilon(E)} |P_\lambda(Z\mid X,Y,E)-P_{\rm V}(Z\mid X,Y)|
=\frac{|s_{\rm V}-\mathcal{Z}_\lambda|}{\mathcal{Z}_\lambda}\ q.
\]

For $Z\in\mathcal{S}_{\rm V}\setminus\mathcal{B}_\varepsilon(E)$, substituting~\Cref{equ:Thm.5,equ:Thm.6} into the second term in~\Cref{equ:Thm.8}, we have
\[
\begin{aligned}
\int_{\mathcal{S}_{\rm V}\setminus\mathcal{B}_\varepsilon(E)}
\Bigl|\tfrac{e^{-\lambda}P(Z\mid X,Y)}{\mathcal{Z}_\lambda}-\tfrac{P(Z\mid X,Y)}{s_{\rm V}}\Bigr|\,dZ
&=\int_{\mathcal{S}_{\rm V}\setminus\mathcal{B}_\varepsilon(E)}
P(Z\mid X,Y)\,\Bigl|\tfrac{e^{-\lambda}}{\mathcal{Z}_\lambda}-\tfrac{1}{s_{\rm V}}\Bigr|\,dZ\\
&=\Bigl|\tfrac{e^{-\lambda}}{\mathcal{Z}_\lambda}-\tfrac{1}{s_{\rm V}}\Bigr|
\int_{\mathcal{S}_{\rm V}\setminus\mathcal{B}_\varepsilon(E)} P(Z\mid X,Y)\,dZ\\
&=\Bigl|\tfrac{e^{-\lambda}}{\mathcal{Z}_\lambda}-\tfrac{1}{s_{\rm V}}\Bigr|
\bigl(s_{\rm V}-s_{\mathcal{B}}\bigr)\\
&=\frac{|e^{-\lambda}s_{\rm V}-\mathcal{Z}_\lambda|}{\mathcal{Z}_\lambda\,s_{\rm V}}\ (s_{\rm V}-s_{\mathcal{B}}).
\end{aligned}
\label{equ:Thm.10}
\tag{A4-10}
\]
Using  $s_{\mathcal{B}}=q\,s_{\rm V}$, \ie, $s_{\rm V}-s_{\mathcal{B}}=s_{\rm V}(1-q)$ yields
\[
\int_{\mathcal{S}_{\rm V}\setminus\mathcal{B}_\varepsilon(E)} |P_\lambda(Z\mid X,Y,E)-P_{\rm V}(Z\mid X,Y)|
=\frac{|e^{-\lambda}s_{\rm V}-\mathcal{Z}_\lambda|}{\mathcal{Z}_\lambda}\ (1-q).
\]

For $Z\in\mathcal{S}_{\rm V}^c$, we have $P_{\rm V}(Z\mid X,Y)=0$.
Also $Z\in\mathcal{S}_{\rm V}^c$ implies $Z\notin\mathcal{B}_\varepsilon(E)$ because $\mathcal{B}_\varepsilon(E)\subseteq\mathcal{S}_{\rm V}$.
Hence on $\mathcal{S}_{\rm V}^c$,
\(
P_\lambda(Z\mid X,Y,E)=\frac{e^{-\lambda}P(Z\mid X,Y)}{\mathcal{Z}_\lambda}.
\)
Therefore, the second term in~\Cref{equ:Thm.8} becomes
\[
\begin{aligned}
\int_{\mathcal{S}_{\rm V}^c} \bigl|P_\lambda(Z\mid X,Y,E)-P_{\rm V}(Z\mid X,Y)\bigr|\,dZ
&=\int_{\mathcal{S}_{\rm V}^c} P_\lambda(Z\mid X,Y,E)\,dZ\\
&=\int_{\mathcal{S}_{\rm V}^c}\frac{e^{-\lambda}P(Z\mid X,Y)}{\mathcal{Z}_\lambda}\,dZ\\
&=\frac{e^{-\lambda}}{\mathcal{Z}_\lambda}\int_{\mathcal{S}_{\rm V}^c}P(Z\mid X,Y)\,dZ\\
&=\frac{e^{-\lambda}}{\mathcal{Z}_\lambda}\,P(\mathcal{S}_{\rm V}^c\mid X,Y)\\
&=\frac{e^{-\lambda}}{\mathcal{Z}_\lambda}\,(1-s_{\rm V}).
\end{aligned}
\label{equ:Thm.11}
\tag{A4-11}
\]

Combining~\Cref{equ:Thm.8,equ:Thm.9,equ:Thm.10,equ:Thm.11}, the bound is given by
\[
\begin{aligned}
D_{\rm TV}(P_\lambda(Z\mid X,Y,E),P_{\rm V}(Z\mid X,Y))
&=\frac12\Bigg(
\frac{|s_{\rm V}-\mathcal{Z}_\lambda|}{\mathcal{Z}_\lambda}\,q
+\frac{|e^{-\lambda}s_{\rm V}-\mathcal{Z}_\lambda|}{\mathcal{Z}_\lambda}\,(1-q)
+\frac{e^{-\lambda}}{\mathcal{Z}_\lambda}\,(1-s_{\rm V})
\Bigg)\\
&=\frac{1}{2\,\mathcal{Z}_\lambda}\Big(
q\,|s_{\rm V}-\mathcal{Z}_\lambda|
+(1-q)\,|e^{-\lambda}s_{\rm V}-\mathcal{Z}_\lambda|
+e^{-\lambda}(1-s_{\rm V})
\Big).
\end{aligned}
\]
Finally, substituting $p_{\theta^\star}(Z\mid X)=P_\lambda(Z\mid X,Y,E)$ yields the claimed result:
\[
\begin{aligned}
D_{\rm TV}\big(p_{\theta^\star}(\cdot\mid X),P_{\rm V}(\cdot\mid X,Y)\big)
&\le D_{\rm TV}(P_\lambda(Z\mid X,Y,E),P_{\rm V}(Z\mid X,Y)), \\
&\le \frac{1}{2\,\mathcal{Z}_\lambda}\Big(
q\,|s_{\rm V}-\mathcal{Z}_\lambda|
+(1-q)\,|e^{-\lambda}s_{\rm V}-\mathcal{Z}_\lambda|
+e^{-\lambda}(1-s_{\rm V})
\Big),
\end{aligned}
\]
where $\mathcal{Z}_\lambda, s_{\rm V}$ are separately given by~\Cref{equ:Z1,equ:S1}. 

This completes the proof.
\end{proof}

\textbf{Restatement of~\cref{thm:2}} [Guaranteed Improvement over Baselines]\textbf{.}
\emph{Let $D_{\rm TV}(\lambda,\varepsilon)$ be the TV bound in~\cref{thm:1}. For any $\varepsilon$ satisfying $\mathcal{B}_\varepsilon\subseteq \mathcal{S}_{\rm V}$, there exist an intensity $\lambda$ such that}
\[
D_{\rm TV}(\lambda^\star,\varepsilon)
\;\le\;
\min\{\,1-s_{\rm V},\,1-q\,\}.
\]   
\emph{For fixed $\lambda=\lambda^\star$, the bound is strictly decreasing in $q$ ($\varepsilon\uparrow$).}
\begin{remark}
This confirms that with proper calibration of intensity $\lambda$ and radius $\varepsilon$, PFlowNet strictly tightens the \emph{idealized} TV bound of standard MLE and expert-guided RLVR.
\end{remark}

\begin{proof}
Recall
\[
D_{\rm TV}(\lambda,\varepsilon)
\;=\;
\frac{1}{2\,\mathcal{Z}_\lambda}\Big(
q\,|s_{\rm V}-\mathcal{Z}_\lambda|
+(1-q)\,|e^{-\lambda}s_{\rm V}-\mathcal{Z}_\lambda|
+e^{-\lambda}(1-s_{\rm V})
\Big),
\qquad
\mathcal{Z}_\lambda=s_{\mathcal B}+e^{-\lambda}(1-s_{\mathcal B}),
\tag{A4-12}
\label{equ:4-12}
\]
where $q\coloneqq s_{\mathcal B}/s_{\rm V}$ and we assume the valid regime $\mathcal{B}_\varepsilon\subseteq\mathcal{S}_{\rm V}$, hence
$0\le s_{\mathcal B}\le s_{\rm V}\le 1$ and $q\in[0,1]$.

\paragraph{Limiting baselines.}
Let $\alpha\coloneqq e^{-\lambda}\in(0,1]$, so $\mathcal{Z}_\lambda=s_{\mathcal B}+\alpha(1-s_{\mathcal B})$.
As $\lambda\to 0$ we have $\alpha\to 1$ and thus $\mathcal{Z}_\lambda\to 1$. Substituting into $D_{\rm TV}$ yields
\[
\lim_{\lambda\to 0}D_{\rm TV}(\lambda,\varepsilon)
=
\frac12\Big(q|s_{\rm V}-1|+(1-q)|s_{\rm V}-1|+(1-s_{\rm V})\Big)
=
1-s_{\rm V}.
\]
As $\lambda\to\infty$ we have $\alpha\to 0$ and thus $\mathcal{Z}_\lambda\to s_{\mathcal B}$.
Since $s_{\mathcal B}\le s_{\rm V}$, we have $|s_{\rm V}-\mathcal{Z}_\lambda|\to s_{\rm V}-s_{\mathcal B}$ and
$|\alpha s_{\rm V}-\mathcal{Z}_\lambda|\to|0-s_{\mathcal B}|=s_{\mathcal B}$, while the last term vanishes.
Therefore,
\[
\lim_{\lambda\to\infty}D_{\rm TV}(\lambda,\varepsilon)
=
\frac{1}{2s_{\mathcal B}}
\Big(q(s_{\rm V}-s_{\mathcal B})+(1-q)s_{\mathcal B}\Big)
=
\frac{1}{2s_{\mathcal B}}\cdot 2s_{\mathcal B}(1-q)
=
1-q,
\]
where we used $s_{\mathcal B}=q\,s_{\rm V}$ in the simplification. 

\paragraph{Existence of a calibrated $\lambda^\star$ and its closed form.}
Observe that $\mathcal{Z}_\lambda$ is continuous in $\lambda$ and decreases from $1$ (at $\lambda=0$) to $s_{\mathcal B}$ (as $\lambda\to\infty$).
Since $s_{\mathcal B}\le s_{\rm V}\le 1$, by the intermediate value theorem there exists $\lambda^\star\in[0,\infty]$ such that
\[
\mathcal{Z}_{\lambda^\star}=s_{\rm V}.
\]
Equivalently, with $\alpha^\star\coloneqq e^{-\lambda^\star}$,
\[
s_{\rm V}=s_{\mathcal B}+\alpha^\star(1-s_{\mathcal B})
\quad\Longrightarrow\quad
\alpha^\star=\frac{s_{\rm V}-s_{\mathcal B}}{1-s_{\mathcal B}}
=\frac{s_{\rm V}(1-q)}{1-q s_{\rm V}}.
\]
\emph{Remark.} When $s_{\rm V}>s_{\mathcal B}$ this gives a finite $\lambda^\star=\log\!\frac{1-s_{\mathcal B}}{s_{\rm V}-s_{\mathcal B}}$; if $s_{\rm V}=s_{\mathcal B}$ then $\alpha^\star=0$ corresponds to the limiting choice $\lambda^\star=+\infty$, which is consistent with the RLVR limit.

Under this calibration, the first absolute-value term vanishes:
$|s_{\rm V}-\mathcal{Z}_{\lambda^\star}|=0$.
To remove the remaining absolute value, note that for any $\alpha\in(0,1]$,
\[
\mathcal{Z}_\lambda-\alpha s_{\rm V}
=
s_{\mathcal B}+\alpha(1-s_{\mathcal B})-\alpha s_{\rm V}
=
\alpha(1-s_{\rm V})+(1-\alpha)s_{\mathcal B}\;\ge\;0,
\]
so $|\alpha s_{\rm V}-\mathcal{Z}_\lambda|=\mathcal{Z}_\lambda-\alpha s_{\rm V}$.
Applying this at $\lambda^\star$ gives
\[
|e^{-\lambda^\star}s_{\rm V}-\mathcal{Z}_{\lambda^\star}|
=
|\,\alpha^\star s_{\rm V}-s_{\rm V}\,|
=
s_{\rm V}(1-\alpha^\star).
\]
Substituting these identities into $D_{\rm TV}$ (\Cref{equ:4-12}) and using $\mathcal{Z}_{\lambda^\star}=s_{\rm V}$ yields
\[
D_{\rm TV}(\lambda^\star,\varepsilon)
=
\frac{1}{2s_{\rm V}}
\Big((1-q)\,s_{\rm V}(1-\alpha^\star)+\alpha^\star(1-s_{\rm V})\Big).
\]
Finally, plug in
$\alpha^\star=\frac{s_{\rm V}(1-q)}{1-q s_{\rm V}}$ and
$1-\alpha^\star=\frac{1-s_{\rm V}}{1-q s_{\rm V}}$
to obtain the closed form
\[
D_{\rm TV}(\lambda^\star,\varepsilon)
=
\frac{(1-q)(1-s_{\rm V})}{1-q s_{\rm V}}.
\]

\paragraph{Strict improvement over the two limiting baselines.}
From the closed form above,
\[
\frac{D_{\rm TV}(\lambda^\star,\varepsilon)}{1-s_{\rm V}}
=
\frac{1-q}{1-q s_{\rm V}}
\;\le\;1
\quad\Longrightarrow\quad
D_{\rm TV}(\lambda^\star,\varepsilon)\le 1-s_{\rm V},
\]
and the inequality is strict whenever $q\in(0,1)$ and $s_{\rm V}\in(0,1)$ (since then $1-q s_{\rm V}>1-q$).
Similarly,
\[
\frac{D_{\rm TV}(\lambda^\star,\varepsilon)}{1-q}
=
\frac{1-s_{\rm V}}{1-q s_{\rm V}}
\;\le\;1
\quad\Longrightarrow\quad
D_{\rm TV}(\lambda^\star,\varepsilon)\le 1-q,
\]
and it is strict for $q\in(0,1)$ and $s_{\rm V}\in(0,1)$ (since then $1-q s_{\rm V}>1-s_{\rm V}$).
Therefore,
\[
\boxed{
D_{\rm TV}(\lambda^\star,\varepsilon)
\;\le\;
\min\{\,1-s_{\rm V},\,1-q\,\},
}
\]
with strict inequality in the non-degenerate interior regime. %This proves the second bullet.

\paragraph{Monotone tightening \emph{w.r.t.} $q$ under calibration.}
Keeping $\lambda=\lambda^\star$ and treating $s_{\rm V}$ as fixed, differentiate
\[
D_{\rm TV}(\lambda^\star,\varepsilon)
=
\frac{(1-q)(1-s_{\rm V})}{1-q s_{\rm V}}
\]
with respect to $q$, we have:
\[
\boxed{
\frac{\partial}{\partial q}D_{\rm TV}(\lambda^\star,\varepsilon)
=
(1-s_{\rm V})\cdot
\frac{-(1-q s_{\rm V})+s_{\rm V}(1-q)}{(1-q s_{\rm V})^2}
=
-\frac{(1-s_{\rm V})^2}{(1-q s_{\rm V})^2}
\;<\;0
\qquad (s_{\rm V}\in(0,1)).
}
\]
Hence the bound is strictly decreasing in $q$.
Within the valid regime $\mathcal{B}_\varepsilon\subseteq\mathcal{S}_{\rm V}$, enlarging $\varepsilon$ increases $s_{\mathcal B}$ and thus increases $q=s_{\mathcal B}/s_{\rm V}$, which strictly tightens the bound. Under the regularity condition in \cref{thm:1}, the inequality
$q\ge \kappa(\varepsilon/\sigma)^{d_{\rm eff}}$ further quantifies this monotone tightening as $\varepsilon$ increases while maintaining $\varepsilon\le\sigma$ so that $\mathcal{B}_\varepsilon\subseteq\mathcal{S}_{\rm V}$ remains valid.

This completes the proof.
\end{proof}

\clearpage
\section{Implementation Details}
\label{app:B}
\vspace{-1mm}
\subsection{Dataset}
\vspace{-1mm}
\label{app:B1}

To optimize PFlowNet, we curated a comprehensive training corpus by aggregating samples from large-scale open-domain multimodal VQA datasets, including the LLaVA~\cite{liu2024llavanext} official training set, VGR~\cite{wang2025vgr}, ArxivQA~\cite{li2024multimodal}, VLM-R$^3$~\cite{jiang2025vlm}, and ThinkLite-VL~\cite{wang2025sota}. We first filtered the raw data based on task difficulty, typology, and evidence distribution, resulting in 95k visual-centric question-answer pairs. Specifically, a subset of 53k samples was processed via the pipeline described in~\Cref{sec:data} to generate perceptual flows; following multi-stage quality control via rejection sampling, 45k high-quality samples were retained for cold-start initialization. The remaining 42k samples were reserved for the subsequent variational reinforcement fine-tuning stage. Notably, to ensure the effectiveness of evaluation, we rigorously cross-checked this corpus against the 15 adopted benchmarks to confirm \emph{zero data overlap}, thereby minimizing the risk of data leakage.

\vspace{-1mm}
\subsection{Training Recipe}
\vspace{-1mm}
\label{app:B2}

\textbf{Cold Start.} We initialize PFlowNet with Qwen3-VL-8B-Instruct~\cite{bai2511qwen3} and fine-tune it using the LLaMA-Factory framework~\cite{zheng2024llamafactory} on $16\times$ NVIDIA H200 GPUs. The model is trained on the 45k SFT samples for 3 epochs. We employ the AdamW optimizer~\cite{loshchilov2017decoupled} with a global batch size of 256 and a peak learning rate of $1 \times 10^{-5}$, employing a cosine decay schedule with a warm-up ratio of 0.1.

\textbf{RFT.} Initialized with the SFT checkpoint, PFlowNet is trained using a custom framework built upon the vLLM~\cite{kwon2023efficient} and TRL~\cite{vonwerra2022trl} on $16\times$ NVIDIA H200 GPUs. We adopt a hybrid parallelism strategy to maximize throughput: data parallelism is applied across two nodes, DeepSpeed ZeRO-3 shards the policy parameters across GPUs within each node, and the reward model is fully replicated on each device to reduce communication overhead. 
Training is performed on 42k samples for 5 epochs, with detailed hyperparameters reported in~\Cref{tab:param}.

\vspace{-1mm}
\subsection{Exploration \& Exploitation}
\vspace{-1mm}
We alternate between vLLM-based \emph{rollout generation} and TRL-based \emph{reward computation} and \emph{policy optimization} in a serial manner.
At each iteration, the current policy is first loaded into the vLLM engine to generate a rollout buffer, which is then consumed by the TRL-based trainer for reward computation and policy updates.
% Each rollout buffer is used for two policy updates.
Afterwards, the updated policy weights are synchronized back to the vLLM engine before generating the next rollout buffer.
Notably, we employ the same system prompt, provided in~\Cref{app:B4}, for both training and self-conditioned reasoning. 
The special token $\langle\texttt{/localize}\rangle$ is used to apart the perceptual behaviors from the flow-conditioned reasoning. 
Specifically, during rollout, we treat $\langle\texttt{/localize}\rangle$ as a \emph{custom stop token}: once this token is detected, the exploration process is terminated.

During the self-conditioned reasoning stage, we organize the input using the same system prompt, the original multimodal input, the generated perceptual flow, and the zoomed-in visual evidence targeted by the flow. 
The resulting conversation template is structured as follows:
\[
\begin{aligned}
\texttt{system:} \quad & \text{system prompt}, \\
\texttt{user:} \quad & \text{multimodal input \& zoomed-in visual evidence}, \\
\texttt{assistant:} \quad & \text{generated perceptual flow}.
\end{aligned}
\]
The model then continues generation conditioned on this structured context for final reasoning response.

\begin{table}[t]
  \centering
  \small
  \caption{Hyperparameters for variational reinforcement fine-tuning.}
  \label{tab:param}
  \begin{tabular}{@{}l r @{\hskip 0.5in} l r@{}}
    \toprule
    \textbf{Hyperparameter} & \textbf{Value} & \textbf{Hyperparameter} & \textbf{Value} \\
    \midrule
    Vicinal shaping intensity ($\lambda$) & 4.5 & Optimizer & AdamW \\
    Vicinal radius ($\varepsilon$) & 0.5 & Peak learning rate & $5\times10^{-6}$ \\
    Reward temperature & 1.0 & Weight decay & 0.05 \\
    Exploration samples ($L$) & 8 & Warmup ratio & 0.02 \\
    Sampling temperature (max) & 1.0 & Batch size per device & 2 \\
    Sampling temperature (min) & 0.7 & Gradient accum. steps & 32 \\
    Rollout Batch Size (sample-level) & 256 & Global Batch Size (response-level) & 1024 \\    
    Max flow length & 4,096 & Gradient clipping & 1.0 \\
    Min flow length & 128 & Max input tokens & 16,384 \\
    Image resolution (min pixels) & 3,670 & Image resolution (max pixels) & 12,845,056 \\
    \bottomrule
  \end{tabular}
\end{table}

\vspace{-1mm}
\subsection{Reward Calculation}
\vspace{-1mm}
\label{app:B3}

We employ \emph{teacher forcing} to obtain outputs of the reward model $p_\phi$, and utilize the resulting logits to efficiently compute the RFT optimization objective defined in~\Cref{equ:8}. Specifically, treating each state as a token sequence, we calculate the transition probability $\log p_\theta(z_k \mid z_{0:k-1})$ by summing the autoregressive log-probabilities of the tokens within $z_k$. Given a data sample $(X, Y, E) \sim \mathcal{P}_{\text{data}}$ and a sampled relation $Z \sim p_\theta(Z \mid X)$, the computation involves three primary components: 

(1) Transition probabilities: $\log p_\theta(z_k \mid z_{0:k-1})$ and $\log p_\theta(\top \mid z_{0:j})$;

(2) Efficacy reward: $\log p_\phi(Y \mid z_{0:k}, \top, X)$;

(3) Quality reward: the ratio $\log p_\phi^+(z_i) - \log p_\phi^-(z_i)$.

To eliminate redundant computations arising from shared prefixes in the first two components, we designed an efficient parallelization strategy. Specifically, we concatenate the shared flow with multiple terminal states or ground-truth labels. By leveraging customized position indices and attention masks (as illustrated in~\Cref{fig:app1,fig:app2}), we compute all terms corresponding to the sub-flows within a single forward pass.  Regarding the third component (quality reward), while it is intuitive to infer vision token indices from RoI coordinates to enable similar parallelization via dynamic masking, we identify two critical challenges. First, the resulting attention masks and position indices are often non-contiguous, leading to implementation complexity. Second, due to the native resolution property, the visual encoder in Qwen3-VL potentially resizes cropped inputs to enhance information density. Simply masking the original image tokens fails to replicate this process, thereby degrading the reward model's perceptual fidelity and compromising the accuracy of the reward calculation. Consequently, we explicitly crop the regions $I^+$ and $I^-$ and regard them as two separate inputs to the reward model for $\log p_\phi^+(z_i)$ and $\log p_\phi^-(z_i)$, thereby computing the ratio $\log p_\phi^+(z_i)-\log p_\phi^-(z_i)$.

\vspace{-1mm}
\subsection{Prompt}
\vspace{-1mm}
\label{app:B4}

\begin{systempromptbox}
\small
\setlength{\parskip}{0.6em}
\setlength{\parindent}{0pt}

You are a helpful visual reasoning assistant. The user asks a question about an image, and you must provide a visually grounded answer by following a four-stage reasoning process in a fixed format. For every question, you must output the following four blocks in this exact order: 

(1) Question analysis: analyze and interpret the user’s question, clarify what needs to be recognized, counted, compared, or inferred from the image, and wrap this entire step in \texttt{<analyze>}\texttt{</analyze>} tags; 

(2) Evidence localization (interleaved): identify the image regions that are most helpful for answering the question, wrap the entire localization step in \texttt{<localize>}\texttt{</localize>} tags, and inside \texttt{<localize>}...\texttt{</localize>} follow an interleaved pattern where for each region you first output the bounding box coordinates wrapped in \texttt{<box>}\texttt{</box>} tags in the format \texttt{<box>}\texttt{[x1, y1, x2, y2]}\texttt{</box>} and then immediately explain how this region helps answer the question before moving on to the next region and repeating the same pattern; 

(3) Evidence verification: review the previously localized regions, their corresponding explanations and supplied visual evidence (if available) to perform step-by-step reasoning, explicitly connect these visual evidence to the final conclusion, and wrap the entire reasoning process in \texttt{<thinking>}\texttt{</thinking>} tags; 

(4) Final answer: provide a clear, concise answer to the user’s question without introducing new reasoning, and wrap the answer in \texttt{<answer>}\texttt{</answer>} tags. 

You must always include all four stages \texttt{<analyze>}, \texttt{<localize>}, \texttt{<thinking>}, and \texttt{<answer>}, keep the tag names and their order exactly as specified, ensure that the \texttt{<localize>} stage follows the interleaved pattern where each \texttt{<box>}...\texttt{</box>} is immediately followed by an explanation, and never output any text outside these four tagged blocks.
\end{systempromptbox}

\clearpage
\begin{figure}[t]
    \centering
    \includegraphics[width=\linewidth]{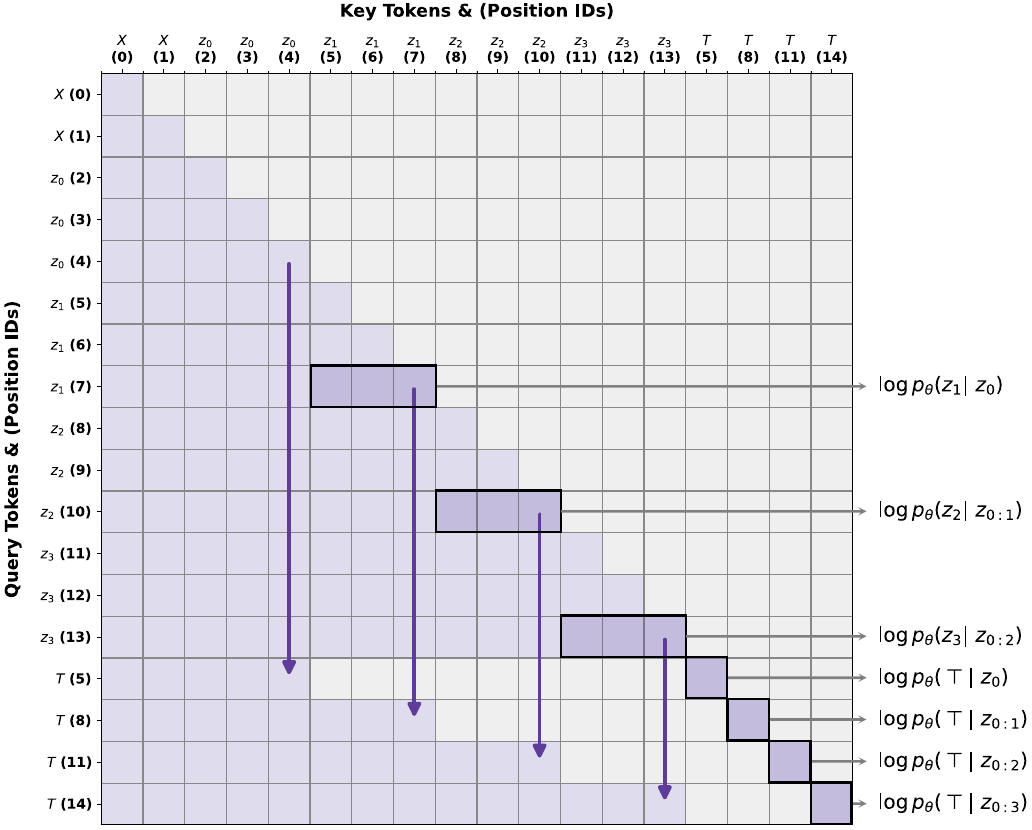}
    \captionof{figure}{Parallel computation strategy of terminal probability, \ie, $\log p_\theta(\top\mid z_{0:i})$, with explicit position IDs \& attention mask.}
    \label{fig:app1}
\end{figure}

\clearpage
\begin{figure}[t]
    \centering
    \includegraphics[width=\linewidth]{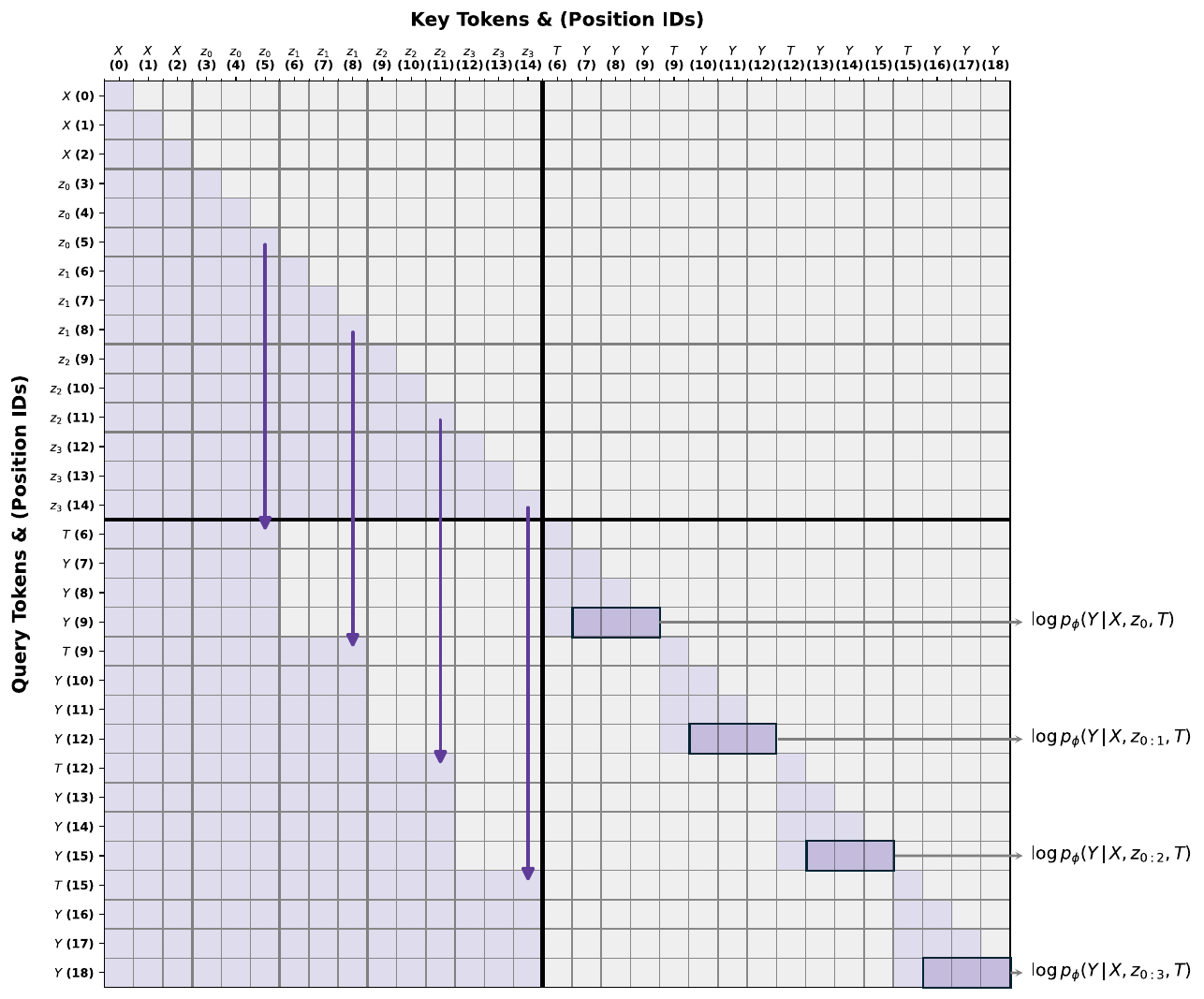}
    \captionof{figure}{Parallel computation strategy of efficacy reward, \ie, $\log p_\phi(Y\mid z_{0:i},X)$, with explicit attention mask.}
    \label{fig:app2}
\end{figure}
\clearpage
\section{Experimental Setup}
\label{app:C}
\subsection{Benchmarks and metrics}
\label{app:C1}

Our evaluation targets visually grounded reasoning from two complementary angles: (i) \emph{general-purpose} VQA that measures broad perception, knowledge, and robustness without requiring explicit evidence localization; and (ii) \emph{fine-grained} VQA/grounding benchmarks that stress high-resolution inputs, small targets, and explicit region-level evidence---precisely the regime where modeling \emph{Perceptual Flow} is expected to be most beneficial. In total, we report results on $15$ widely used benchmarks, following the default protocols from their original evaluations. Unless otherwise specified, we use accuracy for answer correctness; for grounded benchmarks we additionally report localization metrics (\eg, mIoU) when annotated evidence is available.

\emph{General-purpose VQA Benchmarks.}
\textbf{MMBench$_{\text{dev}}^{\text{en}}$}~\cite{liu2024mmbench} provides a comprehensive multi-choice evaluation of multimodal capabilities, spanning fundamental perception, compositional understanding, and higher-level reasoning. Its structured design enables a fine-grained diagnosis of whether gains stem from improved perception or language-side inference.
\textbf{MME-RealWorld-Lite}~\cite{zhang2024mme} is a real-world variant of the MME-style evaluation, designed to reduce dataset bias and emphasize practical visual understanding. It covers diverse perception- and reasoning-centric skills (\eg, OCR, document/scene perception, multi-object understanding), offering a robust stress test for real-world visual grounding.
\textbf{POPE}~\cite{li2023evaluating} focuses on \emph{object hallucination} by asking binary questions about object presence. It directly quantifies the tendency of LVLMs to fabricate visual entities, making it a targeted benchmark for evaluating hallucination mitigation.
\textbf{HallusionBench}~\cite{guan2024hallusionbench} evaluates \emph{detailed visual hallucination} via carefully constructed image-question pairs that probe object attributes, relations, and fine-grained semantics. Compared to coarse hallucination tests, it emphasizes subtle visual distinctions and consistency with the image.
\textbf{AI2D$_{\text{test}}$}~\cite{kembhavi2016diagram} measures diagram understanding and elementary scientific reasoning over educational figures. It tests whether models can correctly interpret schematic structures, labels, and spatial relations rather than relying on natural-image priors.
\textbf{ChartQA$_{\text{test}}$}~\cite{masry2022chartqa} evaluates chart understanding, requiring models to extract numerical values, read legends/axes, and perform lightweight quantitative reasoning grounded in visual plots.
\textbf{MathVision}~\cite{wang2024measuring} targets mathematical visual reasoning over figures (\eg, geometry diagrams and math-centric illustrations). It assesses whether models can ground symbolic reasoning in precise visual cues, which is often brittle under language bias.
\textbf{CV-Bench-2D / CV-Bench-3D}~\cite{tong2024cambrian} is a vision-centric VQA suite repurposed from classic vision tasks to probe fundamental \emph{2D} understanding (\eg, spatial relations, counting) and \emph{3D} understanding (\eg, depth order) within a multimodal QA interface.

\emph{Fine-grained VQA and Grounded Benchmarks.}
\textbf{V* Bench}~\cite{wu2024vstar} is a fine-grained visual search benchmark emphasizing small targets and localization-sensitive queries. It includes subsets such as \emph{Attribute} and \emph{Spatial} that require resolving subtle attributes or spatial configurations, where correct answers typically depend on identifying the right evidence region.
\textbf{HR-Bench (4K/8K)}~\cite{wang2025hrbench} evaluates high-resolution VQA under long-context visual inputs. It contains both \emph{Single} (single high-resolution image) and \emph{Cross} (cross-image / cross-region) settings, stressing the ability to preserve fine details, track small objects, and aggregate evidence across large visual fields.
\textbf{TreeBench}~\cite{wang2025traceable} is a grounded reasoning benchmark that jointly evaluates \emph{answer correctness} and \emph{evidence localization quality} (mIoU). Its taxonomy separates \emph{Perception} (\eg, attributes, OCR, object retrieval) from \emph{Reasoning} (\eg, perspective transforms, ordering, comparisons), enabling a targeted analysis of whether a method improves perception behaviors, reasoning behaviors, or both.
\textbf{ScreenSpot (v2 / Pro)}~\cite{cheng2024seeclick,wuatlas,li2025screenspot} evaluates GUI grounding from screenshots: given an instruction, the model must localize the corresponding UI element (typically via point or box prediction). ScreenSpot-Pro further stresses professional software scenarios with high-resolution screens and smaller targets, making it a representative benchmark for visually grounded interaction and GUI understanding.

\subsection{Baselines}
\label{app:C2}

To evaluate the PFlowNet, we compare against (i) strong \emph{general-purpose LVLMs} that provide competitive zero-/few-shot performance, and (ii) representative \emph{visually grounded reasoning} approaches that explicitly model perception actions, which we categorize into \emph{agentic frameworks} and \emph{grounded RLVR} baselines (\Cref{sec:2}).

\emph{General-purpose LVLMs.}
We include widely adopted instruction-tuned LVLMs (\eg, InternVL3~\cite{zhu2025internvl3}, Qwen2.5-VL~\cite{bai2025qwen25vl}, Qwen3-VL~\cite{bai2511qwen3}) across multiple scales to control for backbone strength. We further report results from leading proprietary or frontier models (\eg, GPT-4o/o3~\cite{gpt4o,o3} and Gemini3 variants~\cite{gemini-3-flash,gemini-3-pro}) when available in the corresponding benchmark protocols, providing an upper-bound reference for general VQA and robustness.

\emph{Agentic Frameworks.}
Agentic frameworks enhance LVLMs with explicit interaction loops and external tools, typically coupling multi-turn planning with image operations (\eg, Zoom-In), code execution, or sandboxed tool calls.
\textbf{Thyme}~\cite{zhang2025thyme} represents ``thinking with images'' by allowing the model to write and execute code for visual processing, improving perception-heavy tasks at the cost of increased latency and tool dependency.
\textbf{DeepEyes / DeepEyesV2}~\cite{zheng2025deepeyes,hong2025deepeyesv2} are tool-augmented grounded reasoning systems that interleave language reasoning with explicit perceptual actions (\eg, zoom/crop/inspect), often relying on external executors to stabilize evidence acquisition.
\textbf{VACoT}~\cite{xu2025vacot} uses visual tools to mitigate performance degradation under challenging inputs (\eg, low quality or ambiguous evidence), emphasizing tool-based intermediate visual steps.
For GUI-centric evaluation, \textbf{Claude Computer Use}~\cite{hu2024dawn} and \textbf{OpenAI CUA}~\cite{openai2025operator} serve as strong agentic baselines that integrate perception with action policies for computer-use settings, reflecting the state of practice for tool-using GUI agents.

\emph{Grounded RLVR and Training-free Methods.}
Grounded RLVR methods train policies with \emph{verifiable} grounding-related rewards by representing perception as explicit spatial tokens (boxes/points) and optimizing the policy toward better evidence localization and answer correctness.
\textbf{TreeVGR}~\cite{wang2025traceable} is a representative grounded RLVR baseline on TreeBench-style tasks, coupling answer reward with localization supervision (often via IoU-style verifiers) to reduce language bias.
\textbf{Pixel-Reasoner}~\cite{su2025pixelreasoner} performs multi-step region selection and refinement to acquire evidence for grounded reasoning, emphasizing iterative perception-to-reasoning transitions.
\textbf{ZoomRefine}~\cite{yu2025zoom} adopts progressive zoom-in/refinement strategies to improve fine-grained evidence capture, typically benefiting attribute/OCR-like perception where small regions matter.
\textbf{DyFo}~\cite{li2025dyfo} represents grounded optimization that encourages structured perceptual behaviors (\eg, MCTS) to improve fine-grained understanding under constrained perception budgets.

\emph{GUI Grounding Methods.} We explicitly include GUI grounding models that predict clickable targets from screenshots:
\textbf{SeeClick}~\cite{cheng2024seeclick} is a screenshot-based GUI agent emphasizing GUI grounding pretraining and realistic element localization.
\textbf{OS-Atlas}~\cite{wuatlas} is a foundation GUI action/grounding model trained on large-scale cross-platform GUI element corpora, outputting normalized coordinates for interaction targets.
\textbf{UGround}~\cite{gou2024navigating} advocates a human-like, fully visual embodiment for GUI agents that perceive GUIs directly from pixels and act via pixel-level operations.
\textbf{UI-TARS}~\cite{qin2025ui} is an end-to-end native GUI agent model that operates directly on screenshots and produces human-like interaction outputs, serving as a strong modern baseline. % for GUI grounding.

\subsection{Evaluation Protocol}
\label{app:C3}

\emph{Evaluation Framework.} 
To ensure a fair comparison, we reproduce all baseline results using their official evaluation pipelines with default configurations. Specifically, for the performance-efficiency and test-time scaling analyses, we migrated the Transformers-based implementations of \emph{TreeVGR} and \emph{Thyme} to \emph{VLMEvalKit} (v0.1.0) utilizing the vLLM backend. For \emph{DeepEyes}, we adopted its official pipeline, which is natively built on VLMEvalKit and vLLM. This standardization ensures strictly consistent experimental conditions, eliminating system-level discrepancies in latency and memory usage caused by different infrastructure frameworks. All evaluations were performed on an NVIDIA H200 GPU.

\emph{Decoding Strategy.} 
For fairness, we employ greedy decoding for all models in standard evaluations. Conversely, for test-time scaling experiments, we utilize stochastic decoding to generate $k$ independent responses per sample. Specifically, pass@$k$ sampling is configured with $\texttt{temperature=1.0}$ and nucleus sampling with $\texttt{top-p=0.95}$ (no explicit \texttt{top-k} truncation is applied unless required by backend defaults).

\emph{Prompting and Inference.} 
For all baseline methods, we adopt their official system prompts and templates (if available) to ensure optimal performance. PFlowNet utilizes the system prompt detailed in~\ref{app:B4}. During inference, PFlowNet's generation is truncated immediately upon detecting the perceptual flow end-of-sequence token ($\texttt{</localize>}$). We then parse the RoIs from the perceptual flow, extract the corresponding fine-grained visual features, and concatenate them with the initial perceptual flow to prompt the model for continued generation, thereby achieving self-conditioned autoregressive generation. Notably, we enforce the identical system prompt across both stages to ensure consistency between perceptual and reasoning behaviors.
\clearpage
\section{Additional Qualitative Analysis}
\label{app:E}
\subsection{Analysis of Test-Time Scaling Behaviors}
\label{app:E1}
\vspace{5mm}

\begin{figure}[h!]
    \centering
    \includegraphics[width=0.75\linewidth]{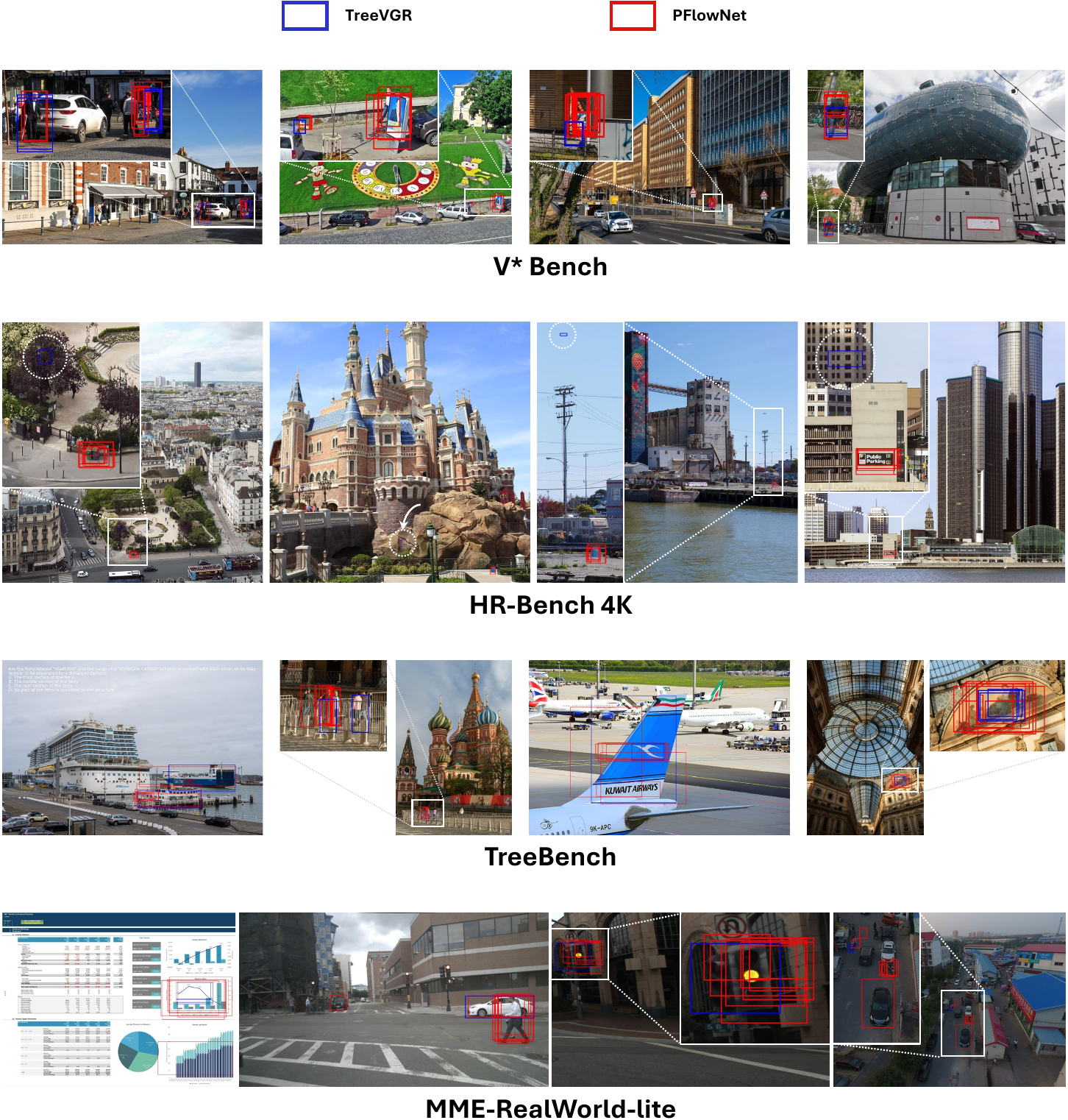}
    \captionof{figure}{Qualitative comparison of grounding results under test-time scaling, highlighting the severe mode collapse in TreeVGR versus the diverse yet reliable perceptual exploration in PFlowNet. This visualization provides an intrinsic explanation for the results in~\Cref{fig:7}: as the computational budget increases, TreeVGR fails to sample diverse latent variables, thereby limiting effective likelihood gains.}
    \label{fig:app5}
    \vspace{1mm}
\end{figure}

To intuitively demonstrate the mode collapse near expert trajectories often exhibited by Grounded RLVR methods, we visualize grounding results selected from four benchmarks under the test-time scaling setting. As presented in~\Cref{fig:app5}, for TreeVGR, \ie, a representative Grounded RLVR method, the bounding boxes generated across multiple reasoning paths overlap almost entirely as the computational budget increases. This indicates a severe lack of perceptual diversity, preventing the model from attending to alternative visual regions even when such exploration is beneficial for reasoning.

In contrast, PFlowNet produces significantly more diverse Regions of Interest (RoIs) across multiple samples. This validates that approximating the target posterior via a variational objective is more effective than rigidly aligning with expert priors in mitigating collapse. Notably, TreeVGR exhibits severe hallucinations in the sample selected from the HR-Bench 4K and attends to featureless background regions. Crucially, due to its collapsed policy, the model lacks the capability to self-correct, persistently focusing on the same erroneous areas despite repeated computation.

\clearpage
\subsection{Analysis of Failure Case}
\label{app:E2}
\vspace{5mm}
    
\begin{figure}[h!]
    \centering
    \includegraphics[width=0.64\linewidth]{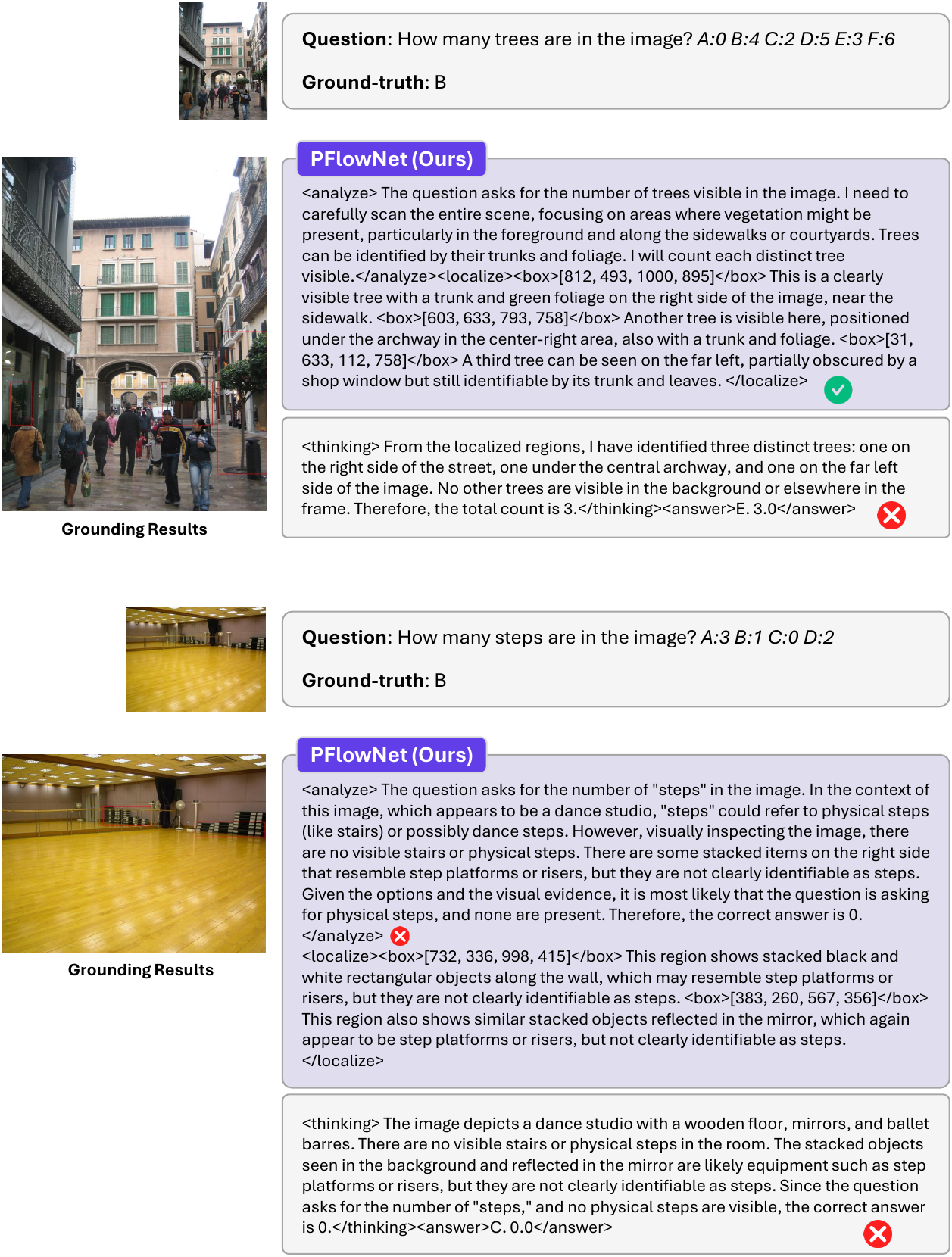}
    \captionof{figure}{Additional qualitative results of visual reasoning. We highlight the important reasoning steps.}
    \label{fig:app6}
    \vspace{4mm}
\end{figure}

We also conduct an in-depth analysis of the failure cases in PFlowNet and identify two primary limitations.

First, a trade-off exists between geometric reliability and fine-grained counting. Since PFlowNet is incentivized to output diverse and reliable bounding boxes, it potentially merges spatially adjacent regions to preserve inter-object context. While it may be beneficial for general visual tasks, this behavior can lead to errors in counting tasks, where the model can be biased by the number of boxes in the perceptual flow. Crucially, as discussed in our ablation study (\Cref{tab:2}), the perceptual flow exerts a strong priming effect on the subsequent reasoning process; consequently, this issue cannot be fully mitigated by simply supplementing fine-grained visual features.

Second, the \emph{planning state} lacks explicit supervision in our current framework, relying solely on passive optimization via the sub-flow level \emph{Efficacy} term in the reward~\Cref{equ:9}. As a result, in challenging scenarios, \eg, OOD scenarios, the model may fail to correctly decompose the necessary evidence. This decomposition failure propagates downstream, inevitably resulting in confusing perceptual behaviors and incorrect reasoning processes.
Addressing these challenges remains a primary focus for our future work.

\clearpage
\subsection{More Examples}
\label{app:E3}
\vspace{10mm}

\begin{figure}[ht]
    \centering
    \includegraphics[width=\linewidth]{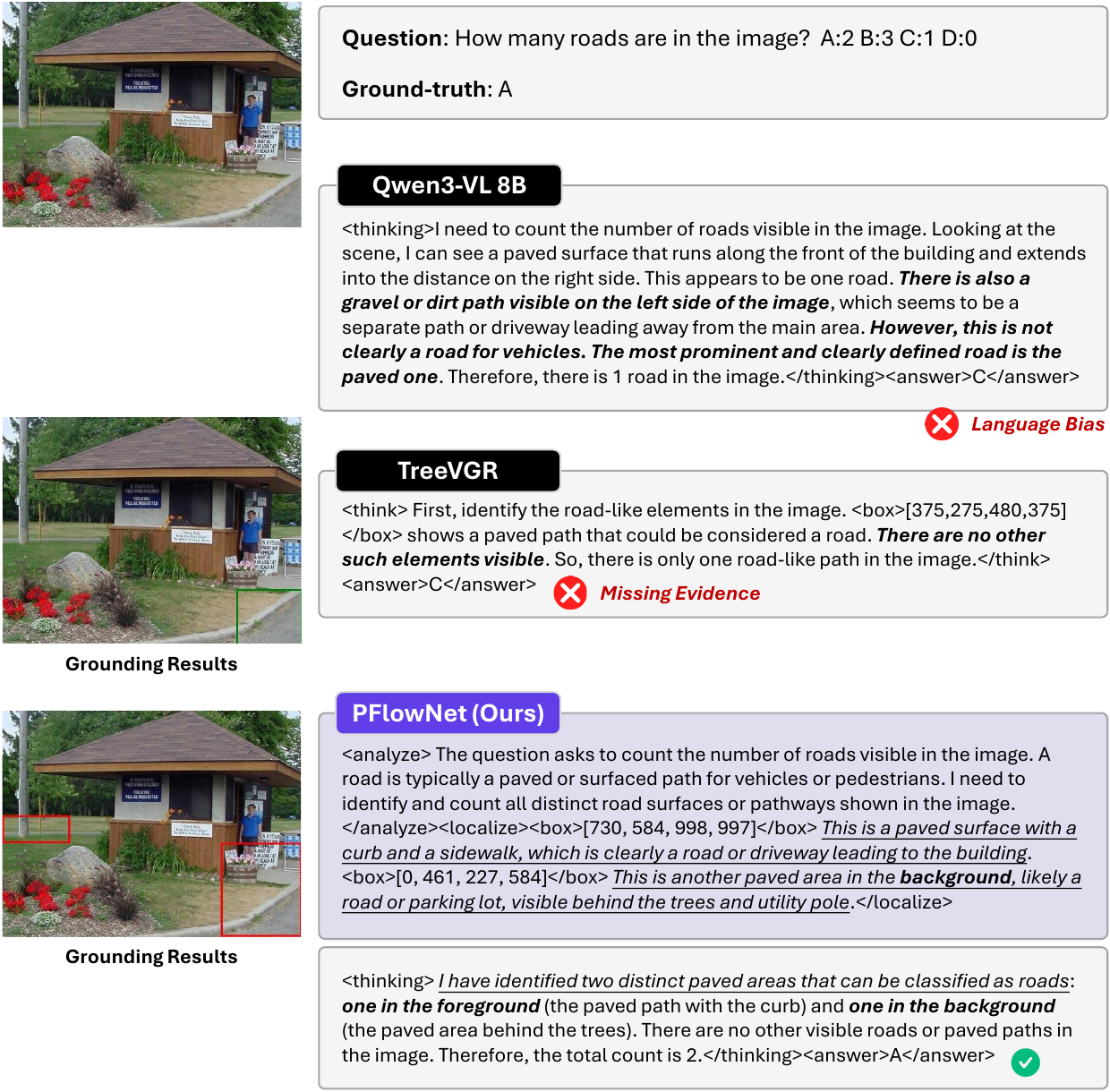}
    \captionof{figure}{Additional qualitative results of visual reasoning. We highlight the important reasoning steps.}
    \label{fig:app3}
\end{figure}

\begin{figure}[h]
    \centering
    \includegraphics[width=\linewidth]{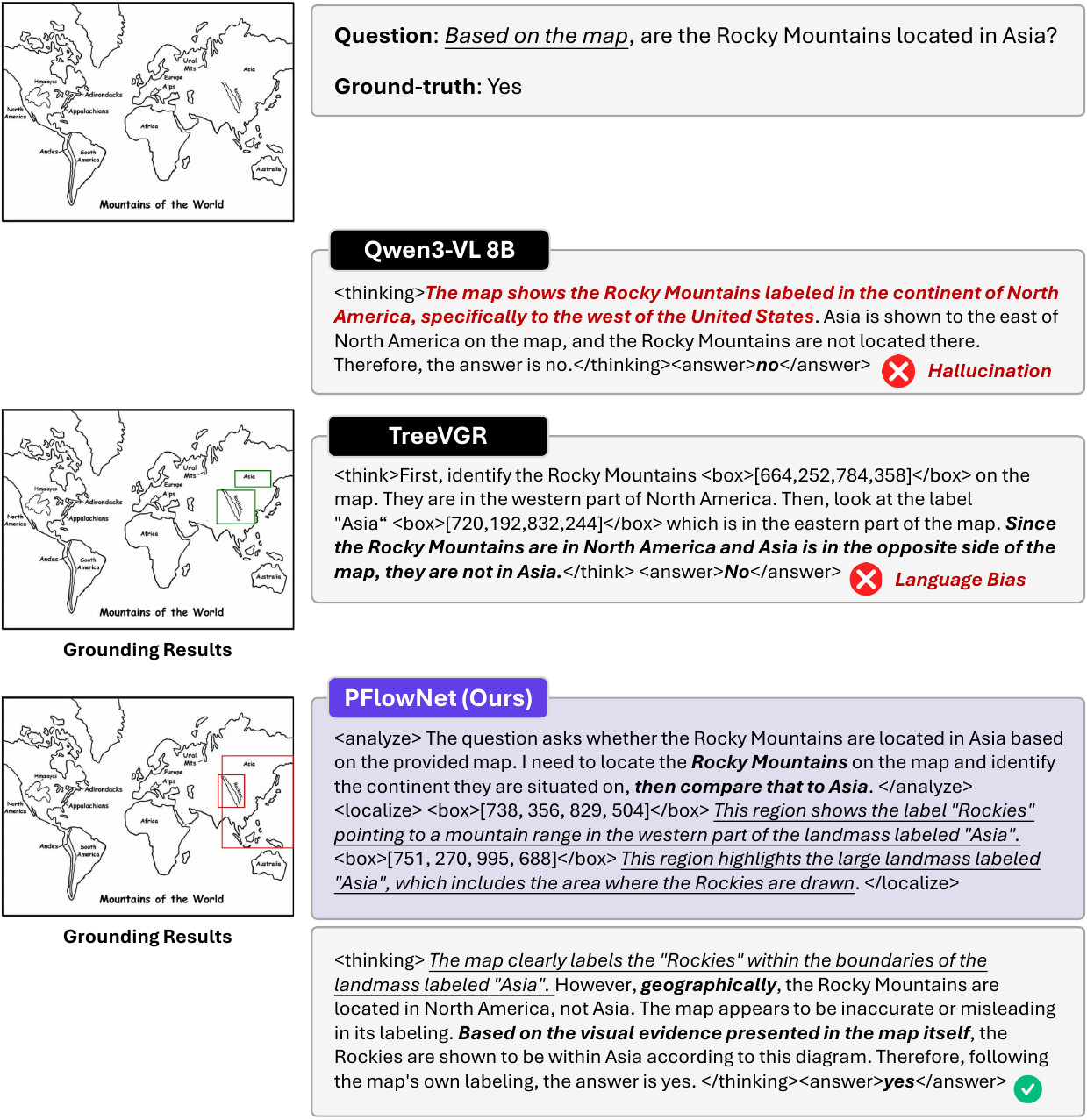}
    \captionof{figure}{Additional qualitative results of visual reasoning. We highlight the important reasoning steps.}
    \label{fig:app4}
\end{figure}

\end{document}